\def\ARXIV{1} 

\documentclass{article}

\PassOptionsToPackage{numbers, compress}{natbib}

 \usepackage[main, final]{neurips_2026}

\usepackage[utf8]{inputenc} 
\usepackage[T1]{fontenc}    
\usepackage{hyperref}       
\usepackage{url}            
\usepackage{booktabs}       
\usepackage{amsfonts}       
\usepackage{nicefrac}       
\usepackage{microtype}      
\usepackage{xcolor}         

\usepackage{amsmath}
\usepackage{amsthm}
\newtheorem{proposition}{Proposition}

\usepackage{graphicx}
\usepackage{algorithm}
\usepackage{algpseudocode}
\usepackage{multirow, xcolor, colortbl}
\hypersetup{hidelinks}

\definecolor{maroon}{cmyk}{0,0.87,0.68,0.32}
\definecolor{bamboo}{cmyk}{0.4,0,0.3,0}
\definecolor{apple}{cmyk}{0.41,0.4,0.76,0}
\definecolor{jialingshui}{cmyk}{0.47,0,0.49,0}
\definecolor{sea}{cmyk}{1,0.67,0.16,0.03}

\title{Market Regime Council for Dynamic Credit Assignment in Multi-Agent LLM Decision Systems}

\ifnum\ARXIV=1 
    \usepackage{orcidlink}
    \hypersetup{hidelinks}
    \author{%
      Yunhua Pei\thanks{Corresponding author: \texttt{ge22472@bristol.ac.uk}} \;\orcidlink{0000-0003-2906-0827}\\
      University of Bristol, UK \\
      \texttt{ge22472@bristol.ac.uk} \\
      \And
      Zerui Ge\\
      Independent Researcher \\
      \texttt{gezerui1997@gmail.com} \\
      \And
      Jin Zheng\,\orcidlink{0000-0002-1783-1375}\\
      University of Bristol, UK \\
      \texttt{jin.zheng@bristol.ac.uk} \\
      \And
      John Cartlidge\,\orcidlink{0000-0002-3143-6355}\\
      University of Bristol, UK \\
      \texttt{john.cartlidge@bristol.ac.uk} \\
    }    
    \usepackage{fancyhdr}
    \usepackage{lastpage}
    \fancyheadoffset[R]{-0.0cm}

\else 
    \author{%
      Yunhua Pei \\
      University of Bristol, UK \\
      \texttt{ge22472@bristol.ac.uk} \\
      \And
      Zerui Ge \\
      Independent Researcher \\
      \texttt{gezerui1997@gmail.com} \\
      \And
      Jin Zheng \\
      University of Bristol, UK \\
      \texttt{jin.zheng@bristol.ac.uk} \\
      \And
      John Cartlidge \\
      University of Bristol, UK \\
      \texttt{john.cartlidge@bristol.ac.uk} \\
    }
\fi

\begin{document}

\maketitle

\ifnum\ARXIV=1
    \setcounter{footnote}{0} 
\fi

\begin{abstract}
Multi-agent LLM decision systems for portfolio management still lack a principled way to assign credit across specialist agents, remain vulnerable to cold-start dominance under regime shifts, and offer limited transparency into how final allocations are formed.  We propose Market Regime Council (MRC), an online credit-assignment mechanism over coalitions of specialist agents, with portfolio management as the downstream task. MRC computes exact Shapley credits across all single, pairwise, and Grand-coalition outputs for online agent weighting. Instantiated with $N{=}3$ specialist agents, at each trading period, MRC recomputes coalition-based Shapley weights from exponentially weighted performance histories, uses a Bayesian adaptive mixture to stabilize early periods, applies regime-dependent multipliers to adjust agent authority, and records each rebalance through a five-layer causal trace. Over 1{,}037 trading days across 13 crypto assets and five seeds, MRC achieves a Sharpe ratio of 1.51 and a cumulative return of 440.1\%, ranking first on CR, SR, and IR among active baselines and attaining the lowest MDD among active methods. Ablation results show that the gains come from Shapley-weighted integration across coalition outputs rather than from any single stage in isolation. Code and demo data are available open source.\footnote{\url{{https://github.com/PEIYUNHUA/Market_Regime_Council}}}
\end{abstract}

\section{Introduction}
\label{sec:intro}

Cryptocurrency portfolio management is a sequential decision problem that depends on heterogeneous signals with different modalities, update frequencies, and predictive horizons.
Price trends, on-chain activity, macro conditions, and market sentiment can all matter, but their usefulness changes markedly across bull, volatile, and bear markets.
In practice, strong portfolio decisions are usually formed by combining partial views from different sources rather than by relying on any single signal~\cite{pei2025cross}.

LLM-based multi-agent systems are a natural fit for this setting~\cite{hong2023metagpt,guo2024large, lin2026squidagent}, and recent financial LLM systems have reported strong risk-adjusted performance~\cite{yu2024fincon,zhang2024multimodal,luo2025llm,wu2023bloomberggpt}.
However, three problems remain.
First, most existing systems use fixed rules or heuristic confidence scores to weight agents, which provides little meaningful credit assignment, that is, little ability to identify which agent or coalition actually improved realized performance~\cite{yu2024fincon,xiao2024tradingagents,yu2025finmem,luo2025llm}.
Second, naive online adaptation is vulnerable to cold-start dominance and can retain stale weights after a regime shift~\cite{jadbabaie2015online,hazan2016introduction,gu2020empirical}.
Third, the final allocation is usually hard to trace back to its intermediate decisions~\cite{arrieta2020explainable,doshi2017towards,miller2019explanation}. Section~\ref{sec:related_main} compares how existing systems assign credit.

Market Regime Council (MRC) treats multi-agent portfolio management as an online cooperative game, using Shapley values as the credit-assignment signal~\cite{shapley1953value}. At each period, MRC evaluates all $2^N{-}1$ coalition outputs, computes exact Shapley credits, and converts them into agent weights. With $N{=}3$ specialist agents covering three orthogonal financial information channels, regime-dependent multipliers adapt these weights to non-stationary conditions, and a five-layer causal trace renders each rebalancing decision fully auditable. Our contributions can be summarized as:

\begin{itemize}
    \item \textbf{Online Shapley credit assignment.} We formulate multi-agent cooperation as a transferable-utility game over all $2^N{-}1$ coalitions, compute exact Shapley values as online credit signals, and combine them with a Bayesian mixture that transitions from a uniform prior to Shapley-derived weights as evidence accumulates.

    \item \textbf{Regime-aware adaptation.} We combine Shapley weighting with a continuous regime score that adjusts agent authority and drives the downstream risk-control overlays across bull, volatile, and bear conditions.

    \item \textbf{Full-chain explainability.} We trace each rebalance through five layers, from raw inputs to agent portfolios, Shapley credit, blend ratios, overlay effects, and the final allocation.

    \item \textbf{Empirical validation.} Over 1{,}037 trading days across 13 assets and five seeds, MRC ranks first on CR, SR, and IR among active LLM and DRL baselines, and achieves the lowest MDD among active methods, instantiated at $N{=}3$.
\end{itemize}

\section{Related Work}
\label{sec:related_main}

\begin{table}[tb]
\centering
\caption{How the compared trading systems assign credit across agents.}
\label{tab:related}
\small
\setlength{\tabcolsep}{4pt}
\begin{tabular}{lll}
\toprule
\textbf{System} & \textbf{Credit across agents} & \textbf{Coalitions} \\
\midrule
FinCon~\cite{yu2024fincon} & Verbal reinforcement from one manager agent & no \\
TradingAgents~\cite{xiao2024tradingagents} & Fixed role pipeline & no \\
FinMem~\cite{yu2025finmem} & Memory recency, no credit between agents & no \\
CryptoTrade~\cite{li2024cryptotrade} & Reflection within one agent & no \\
FSReason~\cite{wang2025exploring} & Equal-weight fusion of two agents & no \\
MAPPO~\cite{yu2022surprising}, MADDPG~\cite{lowe2017multi} & Centralised critic & no \\
\rowcolor{maroon!5} MRC (Ours) & Exact Shapley values over all $2^N{-}1$ coalitions & yes \\
\bottomrule
\end{tabular}
\end{table}

None of the other systems in Table~\ref{tab:related} measures how much each agent or coalition adds to realized performance. MRC instead assigns coalition-based credit grounded in the Shapley axioms~\cite{shapley1953value}. The agent weights it applies are not a raw Shapley allocation but this credit mixed with a uniform prior and then adjusted by the Regime-Aware Multiplier and the Selective Winner-Takes-All override (Section~\ref{sec:method}). Appendix~\ref{sec:related} gives the full discussion.

\section{Preliminaries}
\label{sec:pre}

\paragraph{Problem Formulation.}
Let $\mathcal{K}$ be the set of $K$ crypto assets,
$\mathcal{N} = \{1, \dots, N\}$ a set of $N$ specialist agents,
and $\mathcal{T} = \{1,\dots,T\}$ the sequence of decision periods.
At each period $t$, every non-empty coalition $S \subseteq \mathcal{N}$
can produce a joint decision, and realized feedback is recorded in
history $\mathcal{H}_t$.
In this work we instantiate $N{=}3$, where $\mathcal{N} = \{1,2,3\}$
indexes three orthogonal financial information channels:
price/technical ($A_1$), on-chain network activity ($A_2$),
and macro/sentiment ($A_3$).
The system receives heterogeneous data streams organized into three
agent-specific feature bundles
$\mathbf{x}_{1}^{(t)}, \mathbf{x}_{2}^{(t)}, \mathbf{x}_{3}^{(t)}$
spanning six modality groups
(full inventory in Appendix~\ref{app:impl}, Table~\ref{tab:dataset}).
The system must output a portfolio weight vector
$\mathbf{w}^{(t)} \in \mathcal{W}$, where
$\mathcal{W} = \{\mathbf{w} : \mathbf{1}^\top\mathbf{w} + c = 1,\;
w_k \in [0, w_{\max}],\; c \in [0, c_{\max}]\}$
and $c$ denotes the cash allocation, enforcing per-asset concentration
limits and a bounded cash position.
Let $\mathbf{r}^{(t)}$ denote the asset return vector at period $t$;
the objective is to maximise the long-run risk-adjusted return under
the non-stationarity imposed by shifting market regimes~\cite{gu2020empirical}.
Let $\mathcal{P} = \{(1,2),(1,3),(2,3)\}$ denote the set of pairwise coalitions.

\begin{figure}[tb]
    \centering
    \includegraphics[width=0.8\linewidth]{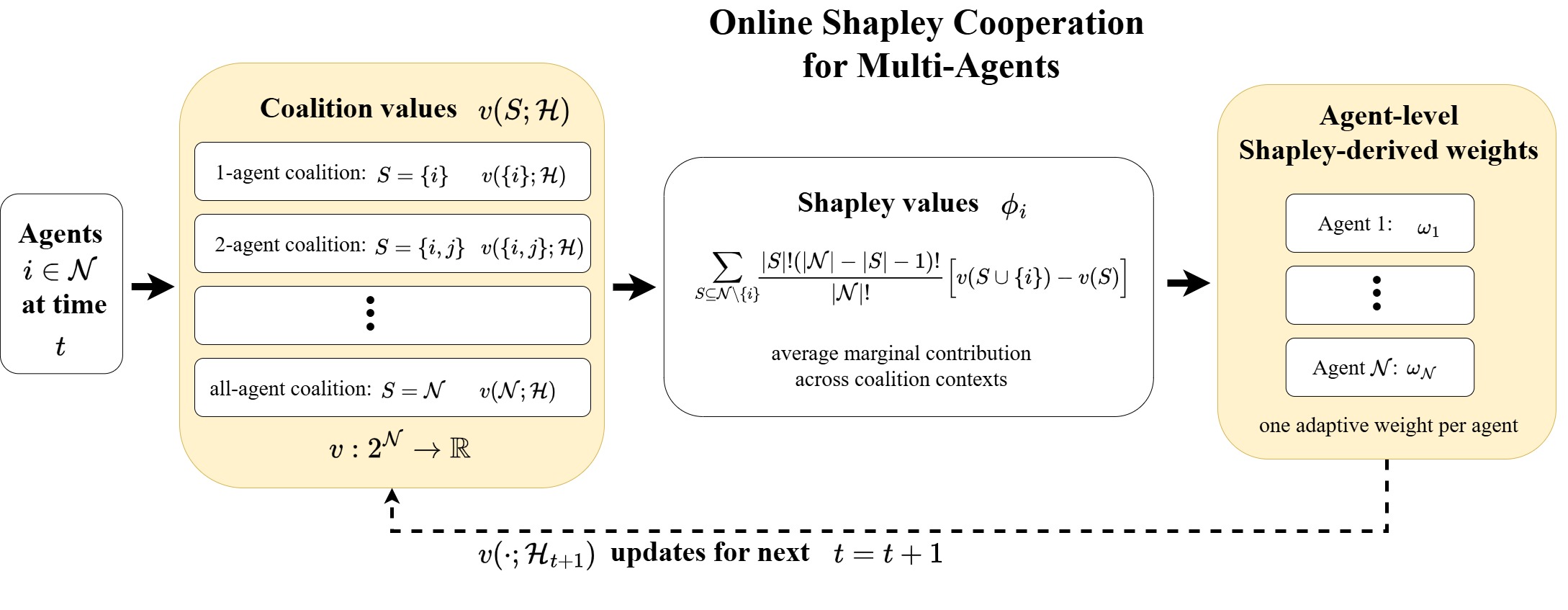}
    \vspace{-3mm}
    \caption{The general $N$-agent Online Shapley Cooperation mechanism.}
    \vspace{-3mm}
    \label{fig:online shapley mechanism}
\end{figure}

\paragraph{Cooperative Game System.}
We model multi-agent credit assignment as a transferable-utility
cooperative game $(\mathcal{N}, v)$~\cite{shapley1953value,wang2024open, wang2020shapley, wang2022shaq},
where the characteristic function $v: 2^{\mathcal{N}} \to \mathbb{R}$
maps each non-empty coalition $S \subseteq \mathcal{N}$ to a scalar
utility measuring its realized performance, with the convention
$v(\emptyset) = 0$.
The Shapley value $\phi_i$ of agent $i$ is the unique credit allocation satisfying four-axiom verification (Efficiency, Symmetry, Dummy Player, Additivity)~\cite{shapley1953value} (see Appendix~\ref{app:shapley}):
\begin{equation}
  \phi_i = \sum_{S \subseteq \mathcal{N} \setminus \{i\}}
    \frac{|S|!\;(|\mathcal{N}| - |S| - 1)!}{|\mathcal{N}|!}
    \bigl[v(S \cup \{i\}) - v(S)\bigr].
  \label{eq:shapley}
\end{equation}

Equation~\eqref{eq:shapley} averages agent $i$'s marginal contribution
across all coalitions that exclude $i$, weighted by the probability of
the corresponding arrival ordering.
For $N$ agents, evaluating all $2^N{-}1$ coalitions requires
$O(N \cdot 2^N)$ utility calls per period; at $N{=}3$ this reduces to
exactly $7$ evaluations, with closed-form expressions in
Appendix~\ref{app:shapley}. What grows fast with $N$ is the number of coalition portfolios, not the Shapley arithmetic, whose cost and error are reported in Table~\ref{tab:scaling} (Appendix~\ref{app:gain}).
Figure~\ref{fig:online shapley mechanism} illustrates how these credits are computed online and converted to adaptive weights each period. Section~\ref{sec:method} specifies how MRC instantiates this mechanism for financial portfolio management.

\section{Method}
\label{sec:method}

\begin{figure}[tb]
    \centering
    \includegraphics[width=1.0\linewidth]{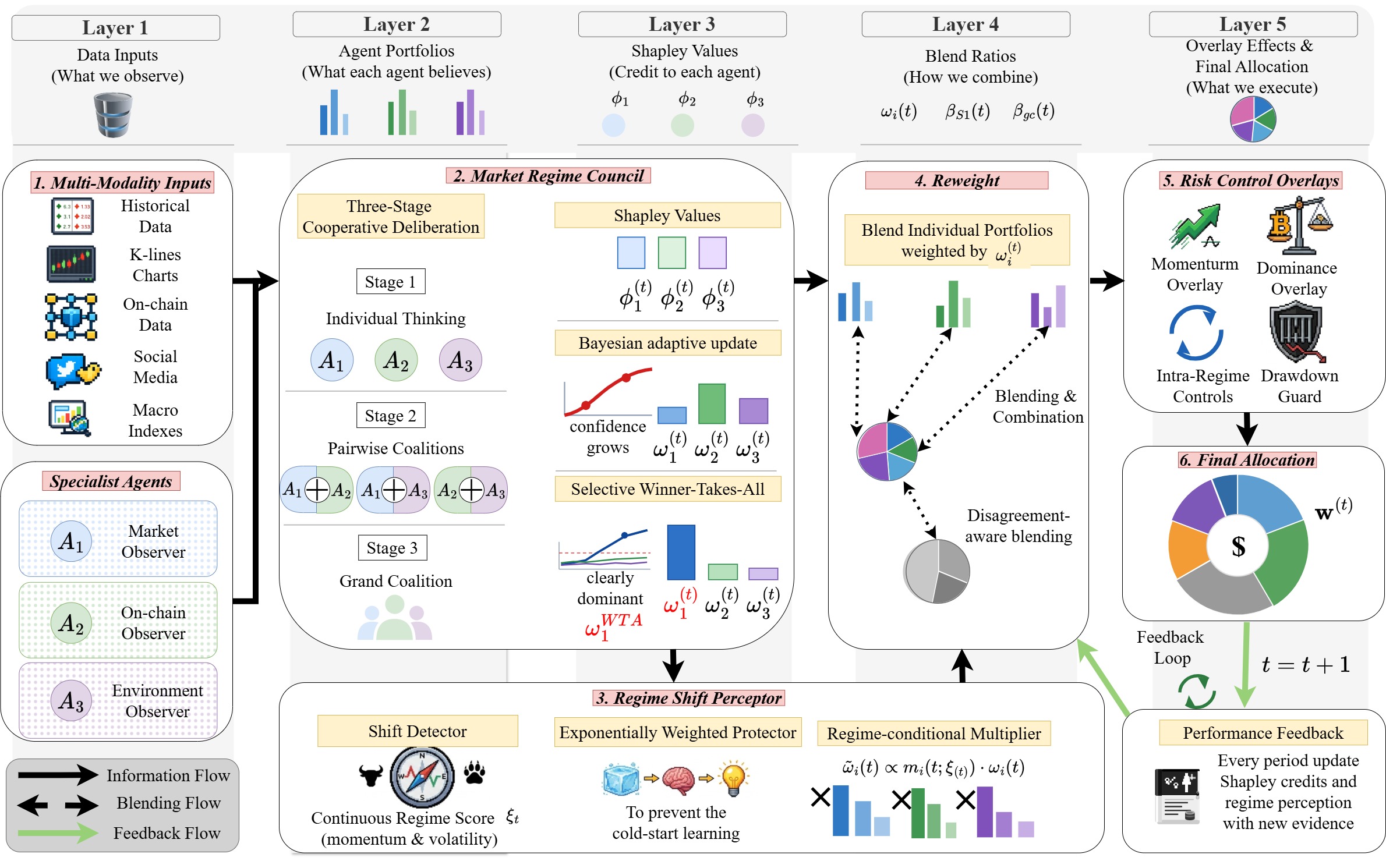}
    \vspace{-3mm}
    \caption{The MRC mechanism applied with $N{=}3$ agents for crypto portfolio management.}
    \vspace{-3mm}
    \label{fig:framework}
\end{figure}

Figure~\ref{fig:framework} summarizes MRC. The key design questions from Section~\ref{sec:pre} are: how to define the characteristic function $v$ for financial coalition outputs, how to stabilize Shapley estimates during early periods, and how to adapt agent authority when market regimes shift. MRC addresses each of these in turn. At each trading period, three specialist agents produce individual portfolios, pairwise debates yield coalition portfolios, and a grand-coalition readout aggregates the shared context into a synthesis portfolio.
MRC then computes coalition utilities from realized historical performance and updates agent weights online via Shapley-based credit assignment. Complete pseudocode for the online weight update and per-period inference, see Appendix~\ref{app:algorithms}.
\subsection{Market Regime Council}
\label{sec:council}

\subsubsection{Three-Stage Cooperative Deliberation ($N{=}3$ Instantiation)}
MRC organizes agent interaction as a sequence of broader coalitions.
In Stage~1, three specialized LLM agents $\mathcal{A}=\{A_1,A_2,A_3\}$ independently process their dedicated data modalities and each produce a private portfolio $\mathbf{w}_{i}^{(t)}$ and regime label $\hat{r}_i^{(t)}\in\{\text{bull, volatile, bear}\}$.
In Stage~2, each pair $(A_i,A_j)$ engages in a structured cross-examination producing a joint portfolio $\mathbf{w}_{ij}^{(t)}$, and three debates run in parallel: $(A_1,A_2)$, $(A_1,A_3)$, $(A_2,A_3)$.
In Stage~3, the grand coalition $\mathcal{N}=\{1,2,3\}$ portfolio $\mathbf{w}_{123}^{(t)}$ is the outcome of all three specialists jointly deliberating with $A_4$ as the readout operator, guided by a Shapley report summarizing agent weights, coalition-level Sharpe ratios, and recent performance (format in Appendix~\ref{app:impl}). $A_4$ is a readout and is not a player in $\mathcal{N}$. For each non-empty coalition $S \subseteq \mathcal{N}$,
the coalition portfolio is $\mathbf{w}_i^{(t)}$ ($|S|=1$),
$\mathbf{w}_{ij}^{(t)}$ ($|S|=2$), or $\mathbf{w}_{123}^{(t)}$ ($|S|=3$). The characteristic function is the composite performance measure (weights $\gamma_\rho$ and $\gamma_\mu$, see Appendix ~\ref{app:impl})
\begin{equation}
  v(S;\,\mathcal{H}_t) = \gamma_\rho \cdot \sqrt{365}\,\frac{\hat{\mu}_{\mathrm{EW}}(S;\,\mathcal{H}_t)}{\hat{\sigma}_{\mathrm{EW}}(S;\,\mathcal{H}_t)} + \gamma_\mu \cdot \mu_{\mathrm{ann}}(S;\,\mathcal{H}_t),
  \quad \gamma_\rho + \gamma_\mu = 1,
  \label{eq:char_fn}
\end{equation}
where $\mathcal{H}_t$ denotes the realized return history of coalition $S$ up to time $t$, $\hat{\mu}_{\mathrm{EW}}(S;\mathcal{H}_t)$ and $\hat{\sigma}_{\mathrm{EW}}(S;\mathcal{H}_t)$ are the EWP-estimated daily mean and daily standard deviation of excess returns, and $\mu_{\mathrm{ann}}(S;\mathcal{H}_t) = 365\,\hat{\mu}_{\mathrm{EW}}(S;\mathcal{H}_t)$ is the annualised mean. This mean is scaled to the same order of magnitude as the Sharpe term to prevent one component from dominating the blend. EWP estimation is defined in \S\ref{sec:ewp}.
Equal-weight history averaging creates path dependence because early cold-start data permanently biases estimates, whereas the EWP (Appendix~\ref{app:ewp}) exponentially attenuates this influence.

\subsubsection{Bayesian Adaptive Update.}
Applying Eq.~\eqref{eq:shapley} to the characteristic function values
yields Shapley values $\phi_i$; agents with $\phi_i < 0$ receive zero
weight via $\phi_i^+ = \max(\phi_i, 0)$~\cite{wang2022shaq,beechey2023explaining},
a projection required by the long-only portfolio constraint.
Closed-form three-player expressions and four-axiom verification are in
Appendix~\ref{app:shapley}.
Shapley estimates from a few observations are unreliable, so we mix the
truncated Shapley-derived allocation with the uninformative uniform
prior~\cite{freund1997decision}:
\begin{align}
  \alpha(t) &= 1 - \exp\!\left(-t/\lambda\right), \label{eq:alpha} \\
  \omega_i^{(t)} &= \alpha(t)\cdot\frac{\phi_i^+}{\sum_j \phi_j^+} + \bigl(1-\alpha(t)\bigr)\cdot\tfrac{1}{|\mathcal{N}|}
  \label{eq:bayesian_update}
\end{align}

where $t$ is the number of completed periods and $\lambda$ governs cold-start duration. As $t\to\infty$, $\alpha\to 1$ and the weights converge to the normalized non-negative Shapley-derived allocation. The update follows a three-phase schedule: $\alpha\approx0$ corresponds to the uniform-prior regime, $0<\alpha<1$ to Bayesian mixing, and $\alpha\approx1$ to a fully data-driven allocation. If every $\phi_i^+$ is zero, the Shapley share in Eq.~\eqref{eq:bayesian_update} is replaced by the prior, so $\omega_i^{(t)}{=}1/|\mathcal{N}|$ before the regime multipliers.

\subsubsection{Selective Winner-Takes-All.}
When one agent's rolling-window Sharpe exceeds the mean of the others by factor $\theta_{\mathrm{WTA}}$, the Bayesian mixture is temporarily replaced by:
\begin{equation}
  \omega_i^{\mathrm{WTA}} = \begin{cases}
    \omega_{\mathrm{wta}}                                                          & i = i^* \\[3pt]
    (1-\omega_{\mathrm{wta}}) \cdot \dfrac{\omega_i}{\sum_{j\neq i^*}\omega_j} & i \neq i^*
  \end{cases}
  \label{eq:wta}
\end{equation}

where $\omega_i^{\mathrm{WTA}}$ is the WTA-overridden weight that temporarily replaces the Bayesian mixture weight $\omega_i$, while $\omega_{\mathrm{wta}}$ is the fixed share allocated to the dominant agent. $i^* = \arg\max_i \rho_i^{(n_{\mathrm{win}})}$ is the dominant agent, and $\rho_i^{(n_{\mathrm{win}})}$ is the rolling Sharpe of agent $i$ over the most recent periods $n_{\mathrm{win}}$. $\theta_{\mathrm{WTA}}$ is the dominance threshold. The override reverts automatically once dominance ends. It never fires when the mean rolling Sharpe of the other agents is not positive, so an undefined or negative reference keeps the Bayesian mixture.

\subsubsection{Multi-Stage Portfolio Integration}
Both blend ratios are driven continuously by performance gaps through $\tanh$:
\begin{align}
  \beta_{S1}^{(t)} &= \bar{\beta}_{S1} + \delta_{S1} \cdot \tanh\!\left(-\Delta_{S2}^{(t)}/\tau_{S1}\right), \label{eq:s1_blend} \\
  \beta_{\mathrm{gc}}^{(t)} &= \max\!\left(0,\; \bar{\beta}_{\mathrm{gc}} + \delta_{\mathrm{gc}} \cdot \tanh\!\left(\Delta_{\mathrm{gc}}^{(t)}/\tau_{\mathrm{gc}}\right)\right),
  \label{eq:a4_blend}
\end{align}
where $\Delta_{S2}^{(t)}=v^{S2}_{\mathrm{ens}}-v^{S1}_{\mathrm{ens}}$ and $\Delta_{\mathrm{gc}}^{(t)}=v(\mathcal{N};\mathcal{H}_t)-v^{S1}_{\mathrm{ens}}$ measure the performance advantages of debate and grand coalition over Stage~1.
$\bar{\beta}_{S1}$, $\delta_{S1}$, $\tau_{S1}$ and their $\mathrm{gc}$ counterparts are the center, scale, and bandwidth of the respective blend functions, set to $\bar{\beta}_{S1}{=}0.90$, $\delta_{S1}{=}0.09$, $\tau_{S1}{=}0.075$ and $\bar{\beta}_{\mathrm{gc}}{=}0.15$, $\delta_{\mathrm{gc}}{=}0.20$, $\tau_{\mathrm{gc}}{=}0.10$. And
\begin{equation}
  v^{S1}_{\mathrm{ens}} = \sum_{i \in \mathcal{N}} \tilde{\omega}_i^{(t)}\cdot v(\{i\};\mathcal{H}_t), \qquad
  v^{S2}_{\mathrm{ens}} = \sum_{(i,j)\in\mathcal{P}} p_{ij}^{(t)}\cdot v(\{i,j\};\mathcal{H}_t)
  \label{eq:v_ens}
\end{equation}
are the Shapley-weight-averaged characteristic function values (Eq.~\eqref{eq:char_fn}) of the Stage~1 individual and Stage~2 pairwise coalitions respectively. 
When debate underperforms ($\Delta_{S2}^{(t)}<0$), $\beta_{S1}^{(t)}$ automatically rises. The Shapley-weighted plurality score is introduced as
\begin{equation}
  \kappa^{(t)} = \max_{r\in\{\text{bull, volatile, bear}\}}\sum_{i:\hat{r}_i^{(t)}=r}\omega_i^{(t)},
  \label{eq:consensus}
\end{equation}
to further discounts the grand-coalition blend when agents disagree on regime, consistent with the expert-disagreement principle in information economics~\cite{ottaviani2001information}:
\begin{equation}
  \beta_{\mathrm{gc}}^{\mathrm{final}} = \beta_{\mathrm{gc}}^{(t)} \cdot \left(\tfrac{1}{2} + \tfrac{1}{2}\kappa^{(t)}\right).
  \label{eq:divergence_discount}
\end{equation}

Given these blend ratios, the final Council output is:
\begin{equation}
\mathbf{w}_{\mathrm{council}}^{(t)} = \beta_{\mathrm{gc}}^{\mathrm{final}}\,\mathbf{w}_{123}^{(t)}
  + \bigl(1-\beta_{\mathrm{gc}}^{\mathrm{final}}\bigr)\Bigl[
    \beta_{S1}^{(t)}\sum_i \tilde{\omega}_i^{(t)}\,\mathbf{w}_i^{(t)}
    + (1-\beta_{S1}^{(t)})\sum_{(i,j)} p_{ij}^{(t)}\,\mathbf{w}_{ij}^{(t)}
  \Bigr],
  \label{eq:ensemble}
\end{equation}
where $\tilde{\omega}_i^{(t)}$ are the regime-adjusted agent weights (Eq.~\eqref{eq:multiplier}, \S\ref{sec:perceptor}). The pairwise weights $p_{ij}^{(t)}$ are the Bayesian-normalized weights over the three elements of $\mathcal{P}=\{(1,2),(1,3),(2,3)\}$, obtained by applying the same adaptive update as Eq.~\eqref{eq:bayesian_update} to the pairwise coalition values.

\subsection{Regime Shift Perceptor}
\label{sec:perceptor}

\subsubsection{Shift Detector}

To detect the current market regime, we define the regime score $\xi^{(t)} \in (-1,+1)$ as a momentum-to-volatility ratio:
\begin{equation}
  \xi^{(t)} = \tanh\!\left(\frac{r_{30d}^{(t)}}{\sigma_{30d}^{(t)}}\right)
  \label{eq:regime_score}
\end{equation}
where $r_{30d}^{(t)} $ is the
equal-weight 30-day log return and
$\sigma_{30d}^{(t)}$ is its realized daily standard deviation. A short-term directional conflict check attenuates $\xi^{(t)}$ when recent momentum contradicts the longer-term trend, reflecting directional ambiguity.
Positive values indicate bull conditions and negative values indicate bear conditions.
Three nominal regime labels are derived from the thresholds $\xi_+>0$ and $\xi_-<0$, which define the bull and bear boundaries respectively (values in Appendix~\ref{app:impl}). The resulting regime segmentation is also visible in the bottom panel of Figure~\ref{fig:shapley_evolution}.

\subsubsection{Exponentially Weighted Protector}
\label{sec:ewp}

During the initial periods of deployment, Shapley-based agent weights have high estimation variance due to the limited return history, which is known as the cold-start.
To address path dependence, we introduce the Exponentially Weighted Protector (EWP), which assigns decaying weights geometrically to historical returns when estimating the coalition characteristic functions in Eq.~\eqref{eq:char_fn}:
\begin{equation}
  w_\tau = e^{-(t-\tau)/h}, \quad \tau = 1,\dots,t
  \label{eq:ewp_decay}
\end{equation}

\begin{equation}
    \hat{\mu}_{\mathrm{EW}}(S;\mathcal{H}_t) = \frac{\sum_{\tau=1}^{t} w_\tau\, R_S^{(\tau)}}{\sum_{\tau=1}^{t} w_\tau}, \qquad
    \hat{\sigma}^2_{\mathrm{EW}}(S;\mathcal{H}_t) = \frac{\sum_{\tau=1}^{t} w_\tau \bigl(R_S^{(\tau)} - \hat{\mu}_{\mathrm{EW}}(S;\mathcal{H}_t)\bigr)^2}{\sum_{\tau=1}^{t} w_\tau}
    \label{eq:ewcv}
\end{equation}

where $\mathcal{H}_t$ is the realized return history of coalition $S$ up to period $t$, $h{=}252$ is the e-fold decay period, and $R_S^{(\tau)}$ is the realized return of coalition $S$ at period $\tau$. EWP is grounded in discounted online learning~\cite{zinkevich2003online,hazan2016introduction}
and EWMA~\cite{jpmorgan1996riskmetrics}. ~\citep{wang2025shapley} further connects its geometric
structure to Shapley's Additivity axiom (formal proof in Appendix~\ref{app:ewp}, Proposition~\ref{prop:ewp}, and verification in Appendix~\ref{app:shapley}).

\subsubsection{Regime-Aware Multiplier}

The Bayesian Shapley weights reflect the average historical contribution across all market regimes.
In practice, agent contributions are regime-dependent: technical signals dominate during trending markets, on-chain signals provide the clearest per-token divergence in bear markets, while macroeconomic signals carry the highest cross-sectional authority in sustained bear regimes due to their role at regime turning points. To prevent any single agent from accumulating disproportionate weight during a favorable regime, we apply a Regime-Aware Multiplier: each Shapley weight is scaled by a continuous, agent-specific scalar obtained by monotone linear interpolation of agent $i$'s bull, volatile, and bear multiplier values at $\xi^{(t)}$ (values in Appendix~\ref{app:impl}), and the result is renormalized over $i \in \mathcal{N}$:
\begin{equation}
  \tilde{\omega}_i^{(t)} \;\propto\; \omega_i^{(t)} \cdot \psi_i\!\left(\xi^{(t)}\right), \qquad \sum_{i \in \mathcal{N}} \tilde{\omega}_i^{(t)} = 1,
  \label{eq:multiplier}
\end{equation}
where $\psi_i$ is the regime-conditional scalar for agent $i$. An additional dominant-dimension override amplifies the leading agent's multiplier when the Shapley-weighted plurality vote exhibits strong consensus, with intensity proportional to the normalized consensus strength
$\displaystyle\frac{\kappa^{(t)} - 1/|\mathcal{N}|}{1 - 1/|\mathcal{N}|} \in [0,1]$
(for $N{=}3$: $(\kappa^{(t)} - \tfrac{1}{3})/\tfrac{2}{3}$),
ensuring that a uniform vote $(\kappa^{(t)} = 1/|\mathcal{N}|)$ produces no amplification.

\subsection{Risk Control}
\label{sec:risk}

\textbf{Momentum Overlay.}
Each council weight is multiplicatively scaled by a regime-gated cross-sectional momentum score, with stronger tilts in trending regimes and near-zero tilts in volatile regimes. Full specification is in Appendix~\ref{app:overlays}. \textbf{Dominance Overlay.}
A BTC dominance signal $d^{(t)}$ captures the relative outperformance of BTC versus the equal-weight basket over a rolling window.
The signal continuously tilts the portfolio toward BTC (or altcoins) via a $\tanh$-transformed score:
\begin{equation}
  d^{(t)} = \tanh\!\left(\frac{\Delta_{\mathrm{BTC-EW}}^{(t)}}{\tau_d}\right) \in (-1, 1)
  \label{eq:btc_dom}
\end{equation}
In volatile regimes, a hard BTC floor constraint additionally enforces a minimum BTC allocation, motivated by the empirical finding that BTC significantly outperforms the altcoin equal-weight basket during high-volatility periods. \textbf{Intra-Regime Controls.}
Three regime-specific controls further refine the portfolio.
A bear on-chain tilt shifts weight toward BTC based on network-activity signals. A volatile cash target rises continuously as the regime score approaches zero. A bull-to-volatile transition buffer scales down positions when the regime score drops sharply, see Appendix~\ref{app:overlays}. \textbf{Drawdown Protection.}
Realized drawdown $\mathrm{DD}^{(t)}$ from the rolling portfolio peak triggers a regime-gated position reduction, analogous to the risk-sensitive value factorization approach of~\cite{shen2023riskq}:
\begin{equation}
  s_{\mathrm{dd}}^{(t)} = 1 - g\!\left(\xi^{(t)}\right) \cdot \tanh\!\!\left(\frac{\mathrm{DD}^{(t)}}{\tau_{\mathrm{dd}}}\right) \cdot \nu_{\mathrm{dd}}
  \label{eq:dd_protection}
\end{equation}
\begin{equation}
  \mathbf{w}^{(t)} = s_{\mathrm{dd}}^{(t)} \cdot \mathbf{w}^{(\ell)}, \qquad
  c^{(t)} = c^{(\ell)} + \bigl(1 - s_{\mathrm{dd}}^{(t)}\bigr)\,\mathbf{1}^\top \mathbf{w}^{(\ell)}
  \label{eq:dd_update}
\end{equation}
where $\ell$ indexes the earlier overlay step, $g(\xi^{(t)}) = \max(0, -\xi^{(t)})$ is a regime gate suppressing protection in bull regimes, $\tau_{\mathrm{dd}}{=}0.15$ is the drawdown bandwidth, and $\nu_{\mathrm{dd}}{=}0.40$ is the maximum position-reduction coefficient. After all overlays, $\mathbf{w}^{(t)}$ is projected onto $\mathcal{W}$, and the realized return
$R^{(t)}=(\mathbf{w}^{(t)})^\top\mathbf{r}^{(t+1)}$
updates $v(\cdot;\mathcal{H}_{t+1})$, which closes the online learning loop.

\section{Experiments}
\label{sec:experiments}

\subsection{Settings}
\textbf{Dataset.} We evaluate MRC on a multi-modal web3 dataset covering 13 assets from 2023-03-01 to 2025-12-31, yielding 1{,}037 daily decision periods. On day $t$ every agent sees only rows dated $t$ or earlier (per-source timing in Appendix~\ref{app:impl}, Table~\ref{tab:timestamps}).
We use crypto as a stress-test domain for non-stationary multi-agent credit assignment because it runs 24/7, transparent on-chain activity, and rapid regime shifts.
All competing LLM and DRL baselines are evaluated on comparably domain-specific asset universes without cross-asset transfer claims~\cite{yu2024fincon,xiao2024tradingagents,li2024cryptotrade}.  \textbf{Baselines.} We compare against 13 baselines in three categories.
\emph{Market benchmarks}: Equal-Weight~\cite{demiguel2009optimal}, BTC Buy-and-Hold, and ETH Buy-and-Hold.
\emph{LLM multi-agent baselines}: FinCon~\cite{yu2024fincon} (hierarchical CVR), TradingAgents~\cite{xiao2024tradingagents} (multi-role pipeline), FinMem~\cite{yu2025finmem} (layered memory), CryptoTrade~\cite{li2024cryptotrade} (single LLM + reflection), and FSReason~\cite{wang2025exploring} (dual Fact/Subjectivity agents, equal-weight fusion); all share MRC's backbone (Qwen3-VL) and temperature ($T_{\mathrm{LLM}}{=}0.7$).
\emph{DRL baselines}: A2C~\cite{mnih2016asynchronous}, PPO~\cite{schulman2017proximal} ($\varepsilon{=}0.2$), DQN~\cite{mnih2015human} (10 discrete templates), MAPPO~\cite{yu2022surprising} (three temporal actors, centralised critic), and MADDPG~\cite{lowe2017multi} (three actors + critics, 302-dim input); all trained online with 260-dim log-return features following the FinRL protocol~\cite{liu2022finrl,finrl2020}. Full dataset and implementation details, please refer to Appendix~\ref{app:impl}.

\subsection{Experiment Results}
We evaluate all methods using four metrics (definitions in Appendix~\ref{app:metrics}):
cumulative return \textbf{CR}\,\% ($\uparrow$), Sharpe ratio \textbf{SR} ($\uparrow$),
maximum drawdown \textbf{MDD}\,\% ($\downarrow$), and information ratio \textbf{IR} vs.\ EW ($\uparrow$).

\paragraph{RQ1: Does MRC outperform other baselines?}

{\renewcommand{\arraystretch}{0.3}
\begin{table}[tb]
\centering
\caption{Main results over 1{,}037 trading days (2023-03-01 to 2025-12-31).
Non-benchmark methods: mean\,$\pm$\,std (5 seeds); benchmarks: deterministic.
\textcolor{red}{Red}: best mean per column among non-benchmark methods; \textcolor{blue}{Blue}: second best.
Superscripts on MRC denote two-sample $t$-test significance versus the second-best method per metric:
$^{*}$$p{<}0.10$, $^{**}$$p{<}0.05$, $^{***}$$p{<}0.01$ (Welch's $t$-test, $n{=}5$ seeds each). The transaction cost is 0bps. For more cost sensitivity analysis, please refer to Table~\ref{tab:slippage_sensitivity} (Appendix~\ref{app:additional_results}).}
\label{tab:main_results}
\setlength{\tabcolsep}{7pt}
\resizebox{0.75\textwidth}{!}{%
\begin{tabular}{llcccc}
\toprule
\textbf{Category} & \textbf{Method} & CR\,\% ($\uparrow$) & SR ($\uparrow$) & MDD\,\% ($\downarrow$) & IR ($\uparrow$) \\
\midrule
\multirow{3}{*}{\textit{Benchmarks}}
  & EW           & 279.93 & 1.14 & 42.78 &  0.00 \\
  & BTC     & 270.94 & 1.23 & 32.02 & $-$0.07 \\
  & ETH     &  78.45 & 0.63 & 63.75 & $-$0.59 \\
\midrule
\multirow{5}{*}{\textit{DRL Baselines}}
  & A2C    & 196.0$_{\pm 64.4}$  & 1.03$_{\pm 0.18}$ & 38.5$_{\pm 8.2}$  & $-$0.53$_{\pm 0.31}$ \\
  & PPO    & 215.2$_{\pm 58.5}$  & 1.04$_{\pm 0.13}$ & 41.1$_{\pm 3.6}$  & $-$0.51$_{\pm 0.41}$ \\
  & DQN    & 249.7$_{\pm 76.0}$  & 1.16$_{\pm 0.18}$ & \textcolor{blue}{34.6}$_{\pm 4.4}$  & $-$0.45$_{\pm 0.55}$ \\
  & MAPPO  & 254.5$_{\pm 4.5}$   & 1.14$_{\pm 0.01}$ & 39.9$_{\pm 0.4}$  & $-$1.14$_{\pm 0.11}$ \\
  & MADDPG & 251.5$_{\pm 13.3}$  & 1.14$_{\pm 0.03}$ & 40.2$_{\pm 0.6}$  & $-$1.16$_{\pm 0.34}$ \\
\midrule
\multirow{5}{*}{\textit{LLM Baselines}}
  & FinCon        & 262.1$_{\pm 6.4}$    & 1.14$_{\pm 0.02}$ & 41.2$_{\pm 1.2}$  & $-$0.98$_{\pm 0.19}$ \\
  & TradingAgents & 277.7$_{\pm 15.7}$   & 1.19$_{\pm 0.03}$ & 38.6$_{\pm 2.9}$  & $-$0.23$_{\pm 0.15}$ \\
  & FinMem        & 320.3$_{\pm 107.6}$  & 1.25$_{\pm 0.20}$ & 39.9$_{\pm 1.9}$  &  0.04$_{\pm 0.50}$ \\
  & CryptoTrade   & \textcolor{blue}{356.3}$_{\pm 39.9}$ & \textcolor{blue}{1.32}$_{\pm 0.06}$ & 41.3$_{\pm 3.6}$ & \textcolor{blue}{0.33}$_{\pm 0.23}$ \\
  & FSReason      & 282.6$_{\pm 29.8}$   & 1.20$_{\pm 0.06}$ & 38.4$_{\pm 1.7}$  & $-$0.14$_{\pm 0.18}$ \\
\midrule
\rowcolor{maroon!5}\textit{Full System}
  & \textbf{MRC (Ours)} & \textcolor{red}{440.1}$^{**}_{\pm 51.4}$ & \textcolor{red}{1.51}$^{***}_{\pm 0.07}$ & \textcolor{red}{34.1}$_{\pm 2.5}$ & \textcolor{red}{0.47}$_{\pm 0.20}$ \\
\bottomrule
\end{tabular}%
}
\end{table}
}

Table~\ref{tab:main_results} reports the overall comparison. MRC attains the best CR, SR, IR, and lowest MDD among all active methods, remaining within 2.1 percentage points of passive BTC on drawdown.
This pattern matters because the strongest competing methods each fail on a different axis: FinMem has high return but high seed variance, CryptoTrade improves Sharpe, but not drawdown, and DRL baselines remain below equal weight on benchmark-relative skill.
MRC also shows the tightest seed-level SR variance among LLM methods ($\pm 0.07$), suggesting that the Bayesian uniform prior and EWP cold-start mechanism improve stability rather than amplify early-period sensitivity.
Full per-group visual comparisons are provided in Appendix~\ref{app:additional_results}, and transaction cost sensitivity is reported in Table~\ref{tab:slippage_sensitivity}: at both 5 and 10\,bps one-way per unit of turnover, MRC still exceeds BTC Buy-and-Hold in cumulative return, while several high-turnover LLM baselines lose their zero-cost advantage. We use 10\,bps one way as the reference rate for costed results. MRC turns over $128.7{\pm}5.5$ times a year and at that rate keeps SR $1.23{\pm}0.06$, the highest of all active methods and above Equal-Weight charged the same rate (1.13, Table~\ref{tab:turnover}). We therefore interpret the main result as a relative robustness claim under matched execution constraints, rather than as an unconditional claim about deployable alpha at arbitrary cost levels.

\paragraph{RQ2: Does MRC really work?}

{\renewcommand{\arraystretch}{0.3}
\begin{table}[tb]
\centering
\caption{Shapley Council ablation: each coalition as a standalone decision maker (5 seeds, mean\,$\pm$\, std). Metrics as in Table~\ref{tab:main_results}. Superscript $^{**}$ on MRC IR: two-sample $t$-test $p{<}0.05$ versus the second-best method (Grand Coalition) on IR.}
\label{tab:ablation_shapley_council}
\vspace{0mm}
\setlength{\tabcolsep}{7pt}
\resizebox{0.75\textwidth}{!}{%
\begin{tabular}{llcccc}
\toprule
\textbf{Stage} & \textbf{Method} & CR\,\% ($\uparrow$) & SR ($\uparrow$) & MDD\,\% ($\downarrow$) & IR ($\uparrow$) \\
\midrule
\multirow{3}{*}{\textit{Benchmarks}}
  & EW       & 279.93 & 1.14 & 42.78 &  0.00 \\
  & BTC & 270.94 & 1.23 & 32.02 & $-$0.07 \\
  & ETH &  78.45 & 0.63 & 63.75 & $-$0.59 \\
\midrule
\multirow{3}{*}{\textit{Stage 1}}
  & Market Agent ($A_1$)   & \textcolor{blue}{393.5}$_{\pm 135.9}$ & \textcolor{red}{1.54}$_{\pm 0.22}$ & 47.5$_{\pm 6.0}$ & \textcolor{blue}{0.08}$_{\pm 0.30}$ \\
  & On-chain Agent ($A_2$) & 238.3$_{\pm 20.6}$ & 1.18$_{\pm 0.05}$ & \textcolor{red}{30.8}$_{\pm 3.0}$ & $-$0.46$_{\pm 0.11}$ \\
  & Macro Agent ($A_3$)    & 129.5$_{\pm 35.0}$ & 0.83$_{\pm 0.11}$ & 40.9$_{\pm 1.3}$ & $-$1.39$_{\pm 0.29}$ \\
\midrule
\multirow{3}{*}{\textit{Stage 2}}
  & $A_1 \oplus A_2$ & 246.9$_{\pm 106.2}$ & 1.34$_{\pm 0.25}$ & 35.2$_{\pm 4.9}$ & $-$0.43$_{\pm 0.34}$ \\
  & $A_1 \oplus A_3$ & 170.8$_{\pm 95.2}$  & 1.10$_{\pm 0.28}$ & 42.9$_{\pm 5.6}$ & $-$0.68$_{\pm 0.34}$ \\
  & $A_2 \oplus A_3$ & 151.1$_{\pm 55.6}$  & 0.98$_{\pm 0.17}$ & \textcolor{blue}{31.8}$_{\pm 2.0}$ & $-$0.95$_{\pm 0.30}$ \\
\midrule
\textit{Stage 3}     & Grand Coalition     & 286.0$_{\pm 77.6}$ & 1.33$_{\pm 0.17}$ & 38.2$_{\pm 3.2}$ & $-$0.26$_{\pm 0.31}$ \\
\midrule
\rowcolor{maroon!5}\textit{Full System} & \textbf{MRC (Ours)} & \textcolor{red}{440.1}$_{\pm 51.4}$ & \textcolor{blue}{1.51}$_{\pm 0.07}$ & 34.1$_{\pm 2.5}$ & \textcolor{red}{0.47}$^{**}_{\pm 0.20}$ \\
\bottomrule
\end{tabular}%
}
\end{table}
}

Table~\ref{tab:ablation_shapley_council} reports the coalition ablation. Stage~1 individual agents span a wide performance range (SR 0.83 to 1.54) with substantial inter-seed variance (SR $\pm$0.05 to 0.22).
$A_1$ achieves the highest point-estimate SR (1.54) but also the deepest MDD (47.5\%), while $A_3$ records the weakest SR (0.83) together with the widest MDD among Stage~1 agents (40.9\%), showing that no single agent balances return and risk at the same time.
Stage~2 pairwise coalitions consistently underperform the best Stage~1 individual portfolio on SR (1.34, 1.10, and 0.98 versus 1.54) while preserving or amplifying inter-seed variance, with $A_1 \oplus A_3$ CR variance reaching $\pm$95.2\%, and offer no compensating improvement in MDD (see Appendix~\ref{sec:discussion} for analysis of why pairwise debate attenuates rather than improves signal in portfolio construction).
The Stage~3 grand coalition alone recovers SR to 1.33 but records MDD of 38.2\%, leaving its risk-return profile worse than $A_1$ on both dimensions taken separately.
The full MRC system records SR\,=\,1.51 and MDD\,=\,34.1\%, reducing MDD by 13.4 percentage points relative to $A_1$ alone at a cost of only 0.03 in SR, tightening seed-level SR variance to $\pm$0.07, and producing the only statistically significant IR advantage in the table ($p{<}0.05$).
For sensitivity analyses for temperature, burn-in, and transaction cost, see Appendix~\ref{app:additional_results}. With the coalition outputs and risk-control overlays held fixed, the credit mechanism beats uniform weighting in all five seeds (Sharpe $+0.137$), and Appendix~\ref{app:gain} compares it with three alternative aggregation rules and reports leave-one-out and robustness checks.

\paragraph{RQ3: Does MRC remain robust across different market regimes?}

Table~\ref{tab:regime_conditional} is a single-seed case study verifying that the regime-conditional mechanisms behave as designed. MAPPO and MADDPG achieve SR\,$\approx$\,2.26 in the bull regime but collapse to SR\,=\,0.03 and 0.02 in volatile, and DQN turns negative (SR\,=\,$-$0.04) in bear, a pattern consistent with price-only state representations that carry no regime signal and sustain bull-market postures into fundamentally different conditions.
At the agent level, $A_1$ alone holds 55.7\% cash in bear, an overcautious posture driven by prior bull-period weight accumulation rather than current conditions, while $A_3$ collapses to SR\,$\approx$\,0.01 across both bear and volatile. These are precisely the dominance and collapse pathologies that the Bayesian Shapley update and regime-aware multiplier are designed to prevent.
MRC achieves SR\,=\,2.60\,/\,0.41\,/\,0.83 and mean cash of 8.0\%\,/\,28.4\%\,/\,13.6\% across bull, bear, and volatile, a monotonically risk-adjusted cash posture absent in every DRL baseline where allocations remain near-static between 7\% and 11\% regardless of market condition, confirming that the adaptive mechanisms engage as intended. Over all five seeds (Table~\ref{tab:regime_five}, Appendix~\ref{app:gain}), MRC has the best mean SR rank of the 14 methods across the three regimes. It ranks first in bull ($2.55{\pm}0.05$) and bear ($0.63{\pm}0.27$), second in volatile ($0.45{\pm}0.22$).

{\renewcommand{\arraystretch}{0.3}
\begin{table}[tb]
\centering
\caption{Regime-conditional performance decomposition (1{,}037 days).
Bull / Bear / Volatile identified by $\xi^{(t)}$.
Metrics as in Table~\ref{tab:main_results}.
\textcolor{red}{Red}: best per column among MRC and LLM/DRL baselines (Shapley ablations and benchmarks excluded);
\textcolor{blue}{Blue}: second best. Single seed ($T_{\mathrm{LLM}}{=}0.7$).}
\label{tab:regime_conditional}
\vspace{0mm}
\setlength{\tabcolsep}{4pt}
\resizebox{\textwidth}{!}{%
\begin{tabular}{ll cccc cccc cccc}
\toprule
& &
  \multicolumn{4}{c}{\textbf{Bull} (430\,d)} &
  \multicolumn{4}{c}{\textbf{Bear} (285\,d)} &
  \multicolumn{4}{c}{\textbf{Volatile} (322\,d)} \\
\cmidrule(lr){3-6}\cmidrule(lr){7-10}\cmidrule(lr){11-14}
\textbf{Category} & \textbf{Method}
  & CR\,\% ($\uparrow$) & SR ($\uparrow$) & MDD\,\% ($\downarrow$) & Cash\,\%
  & CR\,\% ($\uparrow$) & SR ($\uparrow$) & MDD\,\% ($\downarrow$) & Cash\,\%
  & CR\,\% ($\uparrow$) & SR ($\uparrow$) & MDD\,\% ($\downarrow$) & Cash\,\% \\
\midrule
\multirow{3}{*}{\textit{Stage 1}}
  & Market $(A_1)$
  & 315.8 & 2.61 & 20.8 & 22.7\%
  & 17.3 & 0.83 & 15.8 & 55.7\%
  & 44.7 & 1.38 & 31.6 & 41.0\% \\
  & On-chain $(A_2)$
  & 217.5 & 2.37 & 23.6 & 15.1\%
  & 19.3 & 0.73 & 20.6 & 22.1\%
  & $-$4.9 & 0.06 & 32.1 & 19.2\% \\
  & Macro $(A_3)$
  & 226.9 & 2.22 & 17.9 & 7.2\%
  & $-$9.1 & 0.02 & 31.1 & 10.5\%
  & $-$8.4 & 0.01 & 34.5 & 9.3\% \\
\midrule
\multirow{3}{*}{\textit{Stage 2}}
  & $A_1 \oplus A_2$
  & 227.9 & 2.50 & 24.7 & 24.4\%
  & 15.3 & 0.79 & 16.9 & 42.7\%
  & 41.4 & 1.40 & 21.9 & 35.2\% \\
  & $A_1 \oplus A_3$
  & 266.5 & 2.66 & 18.5 & 23.5\%
  & 1.3 & 0.21 & 17.2 & 48.2\%
  & 16.8 & 0.70 & 32.4 & 37.7\% \\
  & $A_2 \oplus A_3$
  & 206.3 & 2.44 & 19.5 & 20.5\%
  & 23.6 & 0.86 & 20.6 & 26.0\%
  & $-$7.7 & $-$0.06 & 28.6 & 23.6\% \\
\midrule
\textit{Stage 3}
  & Grand Coalition
  & 274.7 & 2.55 & 18.3 & 17.7\%
  & 5.4 & 0.37 & 19.8 & 34.7\%
  & 31.2 & 1.05 & 28.3 & 29.6\% \\
\midrule
\multirow{3}{*}{\textit{Benchmarks}}
  & EW
  & 275.83 & 2.26 & 19.84 & 0.00
  & 10.43 & 0.50 & 28.35 & 0.00
  & $-$8.46 & 0.04 & 40.34 & 0.00 \\
  & BTC
  & 132.71 & 1.77 & 17.46 & 0.00
  & 13.13 & 0.57 & 23.02 & 0.00
  & 40.90 & 1.11 & 33.03 & 0.00 \\
  & ETH
  & 297.71 & 2.08 & 24.19 & 0.00
  & $-$42.26 & $-$0.79 & 50.14 & 0.00
  & $-$22.29 & $-$0.19 & 51.63 & 0.00 \\
\midrule
\multirow{5}{*}{\textit{LLM Baselines}}
  & FinCon
  & 271.1 & 2.26 & 19.7 & 0.8\%
  & 7.8 & 0.45 & 28.8 & 0.7\%
  & $-$8.4 & 0.04 & 40.3 & 0.8\% \\
  & TradingAgents
  & 228.4 & 2.24 & \textcolor{blue}{17.7} & 8.0\%
  & 9.4 & 0.48 & 28.5 & 8.0\%
  & 7.6 & 0.41 & 36.5 & 8.0\% \\
  & FinMem
  & \textcolor{blue}{305.6} & \textcolor{blue}{2.54} & 17.9 & 8.0\%
  & \textcolor{red}{21.2} & \textcolor{red}{0.76} & 30.0 & 8.0\%
  & \textcolor{blue}{17.5} & \textcolor{blue}{0.65} & 32.4 & 8.0\% \\
  & CryptoTrade
  & 299.7 & 2.51 & 19.7 & 7.7\%
  & 3.2 & 0.33 & 25.8 & 7.8\%
  & 16.0 & 0.60 & 37.3 & 7.8\% \\
  & FSReason
  & 281.2 & 2.49 & \textcolor{red}{17.1} & 7.0\%
  & 4.9 & 0.37 & 28.3 & 7.5\%
  & $-$1.0 & 0.20 & 33.5 & 7.1\% \\
\midrule
\multirow{5}{*}{\textit{DRL Baselines}}
  & A2C
  & 236.0 & 2.44 & 19.1 & 19.3\%
  & \textcolor{blue}{19.2} & \textcolor{blue}{0.74} & \textcolor{blue}{20.5} & 19.7\%
  & $-$0.2 & 0.20 & 31.3 & 19.9\% \\
  & PPO
  & 186.5 & 2.00 & 18.4 & 11.2\%
  & $-$0.1 & 0.25 & 28.6 & 10.6\%
  & 8.2 & 0.42 & 31.9 & 10.7\% \\
  & DQN
  & 231.7 & 2.48 & 18.1 & 9.7\%
  & $-$10.3 & $-$0.04 & 28.5 & 9.1\%
  & 15.5 & 0.59 & \textcolor{red}{27.5} & 9.0\% \\
  & MAPPO
  & 245.3 & 2.26 & 18.5 & 7.0\%
  & 10.9 & 0.52 & 26.5 & 7.0\%
  & $-$7.7 & 0.03 & 38.1 & 7.0\% \\
  & MADDPG
  & 246.3 & 2.26 & 18.6 & 7.1\%
  & 14.1 & 0.59 & 25.8 & 7.0\%
  & $-$8.0 & 0.02 & 38.3 & 7.0\% \\
\midrule
\rowcolor{maroon!5}\textit{Full System}
  & \textbf{MRC (Ours)}
  & \textcolor{red}{\textbf{353.7}} & \textcolor{red}{\textbf{2.60}} & \textbf{20.9} & \textbf{8.0\%}
  & \textbf{6.8} & \textbf{0.41} & \textcolor{red}{\textbf{19.4}} & \textbf{28.4\%}
  & \textcolor{red}{\textbf{25.9}} & \textcolor{red}{\textbf{0.83}} & \textcolor{blue}{\textbf{29.8}} & \textbf{13.6\%} \\
\bottomrule
\end{tabular}%
}
\end{table}
}

\paragraph{RQ4: How do the learned Shapley weights evolve?}

\begin{figure}[tb]
    \centering
    \includegraphics[width=1.0\textwidth]{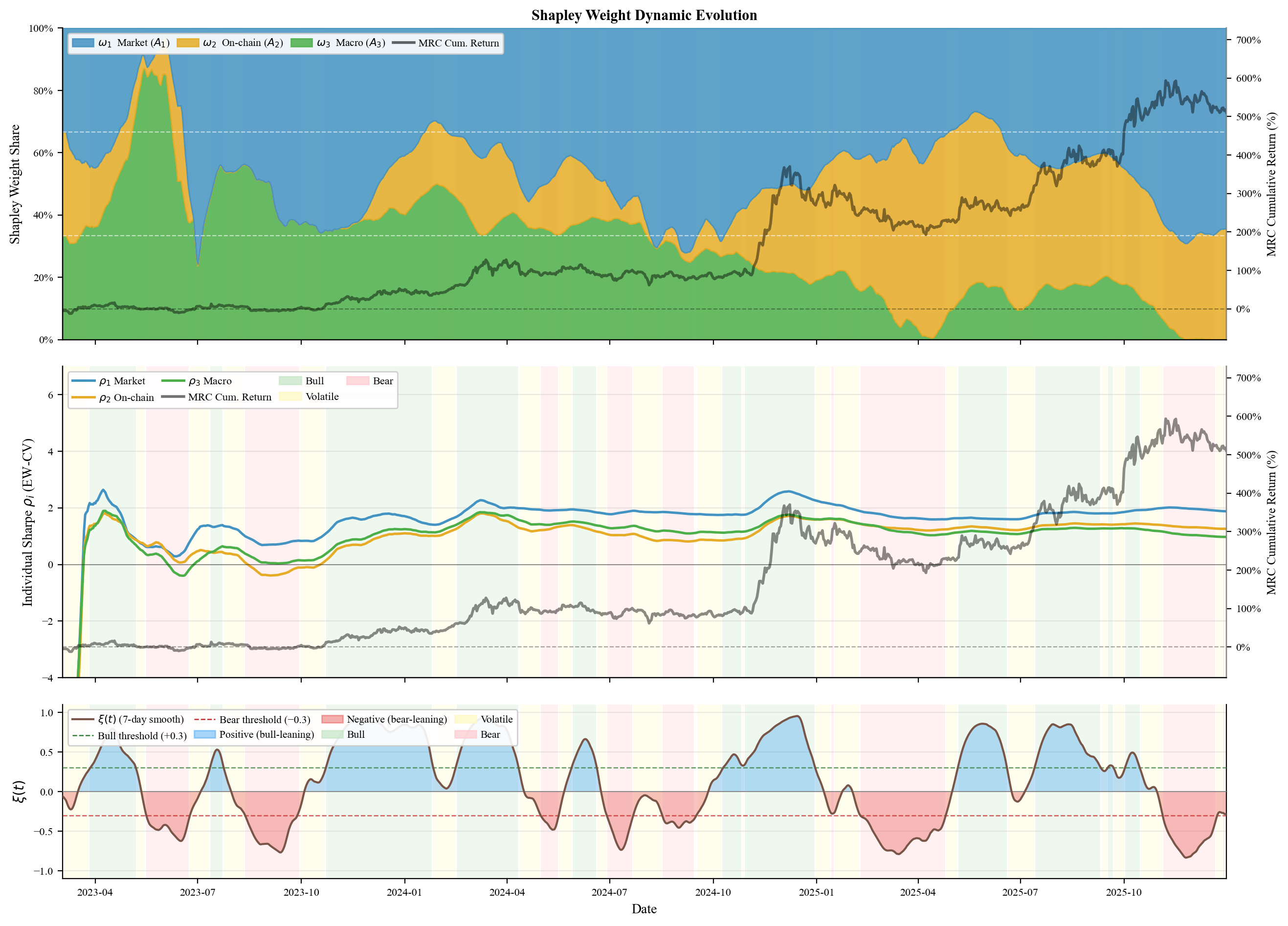}
    \vspace{-3mm}
    \caption{%
        Shapley weight evolution in one seed. Top: Bayesian weights $\omega_i^{(t)}$ for three agents. Middle: $\rho_i^{(t)}$ for each Stage-1 individual coalition. Bottom: conditioning regime score $\xi^{(t)}$.
    }
    \vspace{-3mm}
    \label{fig:shapley_evolution}
\end{figure}

Figure~\ref{fig:shapley_evolution} visualises how the Bayesian Shapley weights $\omega_i^{(t)}$, their underlying EWP-estimated Sharpe drivers $\rho_i^{(t)}$ and the conditioning regime score $\xi^{(t)}$ evolve over the evaluation window. Three observations stand out. First, no agent dominates the first 100 days on average. Over five seeds, the mean regime-adjusted weights of $A_1$, $A_2$ and $A_3$ are 0.34, 0.24 and 0.42, while a single run can already follow early evidence, as the early peak of $A_3$ in Figure~\ref{fig:shapley_evolution} shows. Second, after the cold-start period, $A_1$ (Market) becomes dominant, especially during bull regimes (green shading), consistent with its higher individual Sharpe. Third, over five seeds $A_3$ (Macro) falls to a mean regime-adjusted weight of $0.15{\pm}0.05$ after day 100 (against $0.50{\pm}0.07$ for $A_1$), reflecting its persistently lower $\rho_3$ throughout the risk-on environment of 2023 to 2025. Its weight would recover under a sustained bear regime.
The interaction between Shapley weights and regime-consensus signals, and how the $\beta_{\mathrm{gc}}$ blend ratio responds to agent disagreement, is further illustrated in Figure~\ref{fig:consensus_blend}.
Additional backbone ablations are in Appendix~\ref{app:additional_results} (Tables~\ref{tab:llm_backbone_s1}--\ref{tab:llm_backbone_cost}), token-level return breakdowns are in Appendix~\ref{app:additional_results}, and agent-level profiles are in Appendix~\ref{app:return_profiles}.

\paragraph{RQ5: How can MRC explain rebalancing decisions?}

\begin{figure}[tb]
    \centering
    \includegraphics[width=1.0\textwidth]{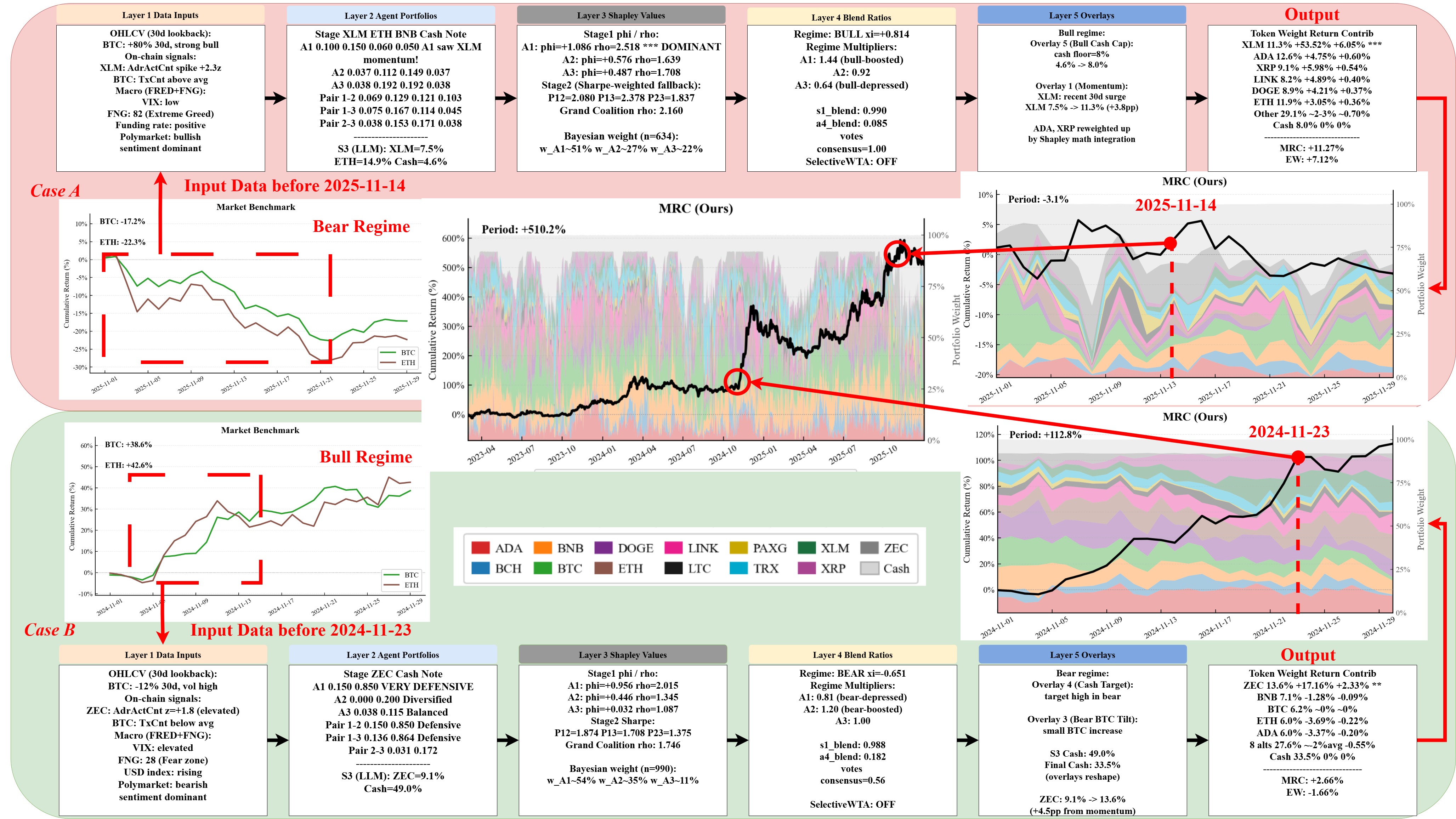}
    \vspace{-3mm}
    \caption{%
    Full-chain explainability under regime-dependent rebalancing.
    }
    \vspace{-3mm}
    \label{fig:case_study}
\end{figure}

Figure~\ref{fig:case_study} shows how MRC performs rebalancing under opposite market regimes.
Case~A illustrates the bear-market path. When market momentum deteriorates, volatility rises, and sentiment turns bearish, MRC shifts toward a more defensive allocation with higher cash while preserving selective exposure supported by on-chain evidence.
Case~B illustrates the bull-market path. When momentum, sentiment, and regime indicators turn strongly positive, MRC assigns greater authority to the market specialist and increases exposure to trend-following assets.
Together, the two cases show that MRC's explainability is not only post-hoc. The same five-layer Shapley report exposes how regime-dependent inputs drive the reallocation decision from raw signals to the final executed portfolio.

\section{Conclusion}
\label{sec:conclusion}

We introduced MRC, a Shapley-based cooperation mechanism for online credit assignment in heterogeneous multi-agent decision systems, validated at $N{=}3$ to cover the three orthogonal financial information channels. Across 1{,}037 trading days, MRC achieves the best CR, SR, and IR among the active baselines under matched execution constraints, while remaining close to passive BTC on drawdown. Coalition ablation indicates that the main gain comes from Shapley-weighted integration across coalition outputs rather than from any isolated stage alone.
This provides evidence that principled credit reassignment can improve robustness in heterogeneous multi-agent decision systems.

The limitation is that the evaluation covers a single cryptocurrency market over one time window and relies on backtests. Live trading would add slippage, limited liquidity, and drift in specialist quality. Trading costs lower the returns of every active method, although MRC keeps its edge over Equal-Weight up to a higher cost than any other active method. Our comparison with non-LLM specialists uses a rule-based council rather than a fully matched-information baseline, which we leave to future work. With both kinds of specialists, the credit mechanism improves on uniform weighting, so the contribution does not depend on the LLM layer. Our future work will extend MRC beyond the current 13 cryptocurrency assets and low-cost execution setting to broader asset classes and cross-domain validation, and will test whether the Shapley-based credit assignment system remains effective as the number of agents scales.

\begin{ack}
This work was funded by UKRI EPSRC (grant No. EP/Y028392/1): AI for Collective Intelligence (AI4CI). 
\end{ack}


\bibliographystyle{abbrvnat}
\bibliography{ref_main}
\newpage


\appendix

\section*{Appendix}

\section{Extended Related Work}
\label{sec:related}

\subsection{LLM-based Multi-Agent Trading System}

LLM multi-agent systems for financial decision-making~\cite{yu2024fincon,
xiao2024tradingagents,yu2025finmem,zhang2024multimodal,luo2025llm,
yang2023fingpt_open,wu2023bloomberggpt,yang2023investlm} have demonstrated strong
empirical performance, but none resolves the credit assignment problem.
General-purpose LLM agent frameworks~\cite{hong2023metagpt,wu2024autogen,
park2023generative} confirm that role specialisation and structured communication
improve collective reasoning, yet lack principled mechanisms for dynamic credit
assignment under non-stationary conditions.
\textbf{FinCon}~\cite{yu2024fincon} is the closest predecessor: it relies on Conceptual Verbal Reinforcement (CVR)—a language-based
feedback loop in which a Manager Agent distills prior analysis into
"conceptual signals" for the downstream Analyst Agents—without any
cooperative game-theoretic credit assignment, online agent reweighting,
or regime conditioning.
\textbf{TradingAgents}~\cite{xiao2024tradingagents} and \textbf{FinMem}~\cite{yu2025finmem} rely on heuristic weighting and layered memory respectively, and neither can identify which modality drove a given decision.
DRL methods (A2C, PPO, DQN) via \textbf{FinRL}~\cite{liu2022finrl} handle credit implicitly through policy gradient objectives, but treat the portfolio as a monolithic black box over price-only features, with no concept of heterogeneous agent coalitions.

The online credit assignment mechanism of MRC connects directly to the classical online portfolio selection literature.
Universal Portfolio~\cite{cover1991universal} and its variants establish that exponentially-weighted expert aggregation achieves logarithmic regret against the best fixed portfolio in hindsight, and the Hedge algorithm~\cite{freund1997decision} achieves minimax-optimal regret under the same expert-aggregation model.
MRC instantiates this principle at the agent-coalition level: the EWP characteristic functions implement a discounted online learning objective~\cite{zinkevich2003online,hazan2016introduction}, while the Bayesian adaptive mixture corresponds structurally to Hedge-style exponential weighting.
The key departure is that Shapley values replace the undifferentiated loss surrogate of standard expert aggregation with an attribution derived from axiomatic Shapley values, enabling credit to track each modality's marginal coalition contribution rather than treating all agents as interchangeable.

\subsection{Regime Modeling and Detection}

The core difficulty in online financial learning is non-stationarity: a model optimal in one regime is harmful in another, yet standard algorithms are slow to unlearn early advantage~\cite{zinkevich2003online,jpmorgan1996riskmetrics}.
In LLM multi-agent systems, where agent parameters are frozen, all adaptation must occur at the weighting layer, making cold-start dominance an especially acute failure mode.
DRL policies are equally vulnerable, which means price-only state features discard the on-chain and macro signals most predictive of regime transitions, and the same network is applied uniformly across bull, volatile, and bear conditions.

MRC decouples regime detection from agent weighting: the Regime Shift Perceptor computes a continuous market-state score independently of the LLM agents, then modulates Shapley weights through smooth regime-conditional multipliers.
\textbf{MarS}~\cite{li2025mars} and \textbf{TwinMarket}~\cite{yuzhe2026twinmarket}
confirm that bull, volatile, and bear regimes reflect qualitatively distinct market
dynamics warranting different agent authority structures, providing empirical
motivation for this design. This is consistent with the broader financial literature
showing that factor exposures and model performance are strongly
regime-dependent~\cite{gu2020empirical}.
\textbf{A2C}, \textbf{PPO}, and \textbf{DQN} are included as direct baselines following the evaluation protocol of \textbf{FinCon}~\cite{yu2024fincon}.

\subsection{Explainability in Trading Systems}

Existing financial AI systems provide no quantitative attribution linking a specific allocation to a specific data source or deliberation stage, a gap highlighted by standardized benchmarks~\cite{xie2024finben,liu2022finrl}.
\textbf{ReAct}~\cite{yao2023react} improves LLM auditability through interleaved
reasoning; \textbf{Chain-of-Thought}~\cite{wei2022chain}, \textbf{Tree-of-Thoughts}~\cite{yao2023tree},
and \textbf{Reflexion}~\cite{shinn2023reflexion} further demonstrate that structured
deliberation elicits more reliable outputs from LLMs.
MRC extends this line of work by connecting each agent's chain-of-thought directly
to final portfolio weights through the Shapley computation, making the reasoning
trace causally (not merely rhetorically) linked to the decision.
The information economics result~\cite{ottaviani2001information} that transparent aggregation dominates opaque pooling when experts disagree motivates MRC's Shapley-weighted divergence discount, which surfaces inter-agent disagreement explicitly rather than absorbing it silently.
MRC's five-layer attribution system directly addresses the opacity gap
shared by all competing systems~\cite{yu2024fincon,xiao2024tradingagents,
yu2025finmem}. This system complements model-level explainability tools such
as \textbf{LIME}~\cite{ribeiro2016should} and \textbf{SHAP}~\cite{lundberg2017unified} by operating at the system level, and attributing final allocations to agent coalitions rather than input features.

\section{Evaluation Metrics}
\label{app:metrics}

We evaluate all methods using four complementary metrics that jointly capture
profitability, risk-adjusted efficiency, downside risk, and benchmark-relative skill.

\textbf{Cumulative Return (CR\%)}~\cite{hull2012risk,yu2024fincon} measures the total compounded
growth of a portfolio over the evaluation window, providing a direct summary of
long-term profitability.  Let $R^{(t)}$ denote the portfolio's net return on day $t$;
the cumulative return over $T$ days is
\begin{equation}
  \mathrm{CR} = \prod_{t=1}^{T}(1 + R^{(t)}) - 1.
  \label{eq:cr}
\end{equation}
A higher CR indicates a more profitable strategy over the evaluation period.

\textbf{Sharpe Ratio (SR)}~\cite{sharpe1998sharpe}: Sharpe Ratio quantifies
risk-adjusted return by normalising the excess return of a portfolio above the
risk-free rate $R_f$ by its total return volatility $\sigma_p$, as shown in
Eq.~\eqref{eq:sr}.
\begin{equation}
  \mathrm{SR} = \sqrt{365} \cdot \frac{\bar{R}_p - R_f}{\hat{\sigma}_p},
  \label{eq:sr}
\end{equation}
where $\bar{R}_p$ is the mean daily portfolio return, $\hat{\sigma}_p$ is the sample standard deviation of daily returns, $R_f = 0$ is the risk-free rate used in the crypto setting, and the factor $\sqrt{365}$ annualizes the ratio for daily crypto markets.

A higher SR indicates more return per unit of total risk and is therefore used to compare risk-adjusted performance across methods.
Interpretation remains context-dependent~\cite{yu2024fincon}.

\textbf{Maximum Drawdown (MDD\%)}~\cite{ang2006downside}: MDD captures the worst
peak-to-trough loss a portfolio sustains before recovering to a new high,
measuring the largest relative decline from a peak value $P_{\text{peak}}$ to
a subsequent trough $P_{\text{trough}}$, as detailed in Eq.~\eqref{eq:mdd}.
\begin{equation}
  \mathrm{MDD} = \max\!\left(\frac{P_{\text{peak}} - P_{\text{trough}}}{P_{\text{peak}}}\right).
  \label{eq:mdd}
\end{equation}
A lower MDD reflects stronger capital preservation and greater resilience
under adverse market conditions.

\textbf{Information Ratio (IR)}~\cite{grinold2000active} measures the consistency
with which a strategy outperforms a passive benchmark, penalising high tracking
error. We use the equal-weight ($1/K$) portfolio, denoted EW, as the benchmark:
\begin{equation}
  \mathrm{IR} = \frac{\bar{R}_p - \bar{R}_b}{\hat{\sigma}(R_p - R_b)},
  \label{eq:ir}
\end{equation}
where $\bar{R}_p$, $\bar{R}_b$ are the mean daily returns of the portfolio and benchmark respectively, and $\hat{\sigma}(R_p - R_b)$ is the sample standard deviation of the daily active return.
A higher IR indicates more consistent outperformance relative to the benchmark.

\section{Algorithm Pseudocode}
\label{app:algorithms}

We provide complete pseudocode for the two core routines of MRC.

Algorithm~\ref{alg:ledger} details the offline weight-update step executed once per period after realized returns are observed: it estimates coalition characteristic functions via EWP (Phase~1), computes exact three-player Shapley values (Phase~2), applies Bayesian adaptive mixing (Phase~3), and conditionally activates the Selective WTA override (Phase~4), corresponding to \S\ref{sec:council} of the main paper.

\begin{algorithm}[htbp]
\caption{MRC Shapley Value Ledger Update}
\label{alg:ledger}
\begin{algorithmic}[1]
\Require Return history $\{R_S^{(\tau)}\}$ for all $S \subseteq \mathcal{N}$, $\tau=1,\ldots,t$;
         period count $t$; decay period $h$; concentration parameter $\lambda$;
         WTA threshold $\theta_{\mathrm{WTA}}$; WTA weight $\omega_{\mathrm{wta}}$;
         uniform prior $\pi_i = 1/|\mathcal{N}| = 1/3$ for all $i$
\Ensure Updated agent weights $\boldsymbol{\omega} = (\omega_1, \omega_2, \omega_3)$, pairwise weights $\{p_{ij}\}$

\Statex \textbf{Phase 1: EWP Estimation of Coalition Characteristic Functions}
\For{each non-empty $S \subseteq \mathcal{N}$}
  \State Compute $w_\tau \leftarrow \exp\!\bigl(-(t-\tau)/h\bigr)$ for $\tau = 1,\ldots,t$ \quad\Comment{Eq.~\eqref{eq:ewp_decay}}
  \State Compute $\hat{\mu}_{\mathrm{EW}}(S;\mathcal{H}_t)$, $\hat{\sigma}_{\mathrm{EW}}(S;\mathcal{H}_t)$
  \quad\Comment{Eq.~\eqref{eq:ewcv}}
    \State $v(S) \leftarrow \gamma_\rho \cdot \sqrt{365}\,\hat{\mu}_{\mathrm{EW}}(S;\mathcal{H}_t)/\hat{\sigma}_{\mathrm{EW}}(S;\mathcal{H}_t) + \gamma_\mu \cdot \mu_{\mathrm{ann}}(S;\mathcal{H}_t)$
\quad\Comment{Eq.~\eqref{eq:char_fn}}
\EndFor

\Statex \textbf{Phase 2: Exact Three-Player Shapley Computation}
\State $(\phi_1, \phi_2, \phi_3) \leftarrow \mathrm{ShapleyValues}(v_1, v_2, v_3, v_{12}, v_{13}, v_{23}, v_{123})$ \quad\Comment{Eqs.~\eqref{eq:phi1}--\eqref{eq:phi3}}
\State $\phi_i^+ \leftarrow \max(0, \phi_i)$ for each $i \in \mathcal{N}$ \quad\Comment{Dummy-player truncation}

\Statex \textbf{Phase 3: Bayesian Adaptive Weight Update}
\State $\alpha \leftarrow 1 - \exp(-t/\lambda)$ \quad\Comment{Eq.~\eqref{eq:alpha}}
\For{each $i \in \mathcal{N}$}
    \State $\omega_i \leftarrow \alpha \cdot \phi_i^+ / \sum_j \phi_j^+ + (1-\alpha) \cdot \pi_i$
    \quad\Comment{Eq.~\eqref{eq:bayesian_update}}
\EndFor
\State Compute $\{p_{ij}\}$ analogously using $v(\{i,j\})$ values and pairwise priors

\Statex \textbf{Phase 4: Selective Winner-Takes-All Override}
\State $i^* \leftarrow \arg\max_i\, \rho_i^{(n_{\mathrm{win}})}$ \quad\Comment{rolling Sharpe over $n_{\mathrm{win}}$-period window}
\If{$\rho_{i^*}^{(n_{\mathrm{win}})} \,/\, \mathrm{mean}_{j \neq i^*}(\rho_j^{(n_{\mathrm{win}})}) \;\geq\; \theta_{\mathrm{WTA}}$}
  \State $\omega_{i^*} \leftarrow \omega_{\mathrm{wta}}$
  \State $\omega_j \leftarrow (1-\omega_{\mathrm{wta}}) \cdot \omega_j / \sum_{k \neq i^*} \omega_k$ \quad for all $j \neq i^*$
\EndIf
\State \Return $\boldsymbol{\omega}$, $\{p_{ij}\}$
\end{algorithmic}
\end{algorithm}

Algorithm~\ref{alg:inference} details the \emph{online} per-period inference step: it runs the three-stage cooperative deliberation (Stages~1--3), blends coalition portfolios via the data-driven ratios $\beta_{S1}^{(t)}$ and $\beta_{\mathrm{gc}}^{\mathrm{final}}$, applies the risk-control cascade (\S\ref{sec:risk}), and then calls Algorithm~\ref{alg:ledger} to update the Shapley ledger $\mathcal{L}$ with the new realized return.

\begin{algorithm}[t]
\caption{MRC Per-Period Inference}
\label{alg:inference}
\begin{algorithmic}[1]
\Require Data streams $\mathbf{x}_1^{(t)}, \mathbf{x}_2^{(t)}, \mathbf{x}_3^{(t)}$;
         current agent weights $\boldsymbol{\omega}$, pairwise weights $\{p_{ij}^{(t)}\}$;
         Shapley ledger $\mathcal{L}$; price history $\mathcal{H}_{\mathrm{price}}$
\Ensure Portfolio $\mathbf{w}^{(t)}$, updated ledger $\mathcal{L}$
\State $\xi^{(t)} \leftarrow \mathrm{RegimeScore}(\mathcal{H}_{\mathrm{price}})$ \quad\Comment{Eq.~\eqref{eq:regime_score}}

\Statex \textbf{Stage 1: Private Observation (executed in parallel)}
\For{$i \in \{1, 2, 3\}$ \textbf{in parallel}}
  \State $(\mathbf{w}_i^{(t)},\, \hat{r}_i^{(t)}) \leftarrow \mathrm{LLM}_{A_i}\!\left(\mathbf{x}_i^{(t)}\right)$
\EndFor

\Statex \textbf{Stage 2: Pairwise Socratic Debate (executed in parallel)}
\For{$(i,j) \in \{(1,2),(1,3),(2,3)\}$ \textbf{in parallel}}
  \State $\mathbf{w}_{ij}^{(t)} \leftarrow \mathrm{LLM}_{\mathrm{Debate}}\!\left(A_i, A_j, \mathbf{w}_i^{(t)}, \mathbf{w}_j^{(t)}\right)$
\EndFor

\Statex \textbf{Regime-Aware Multiplier and Consensus}
\State $\kappa^{(t)} \leftarrow \mathrm{ShapleyWeightedPlurality}\!\left(\boldsymbol{\omega}, \{\hat{r}_i^{(t)}\}\right)$ \quad\Comment{Eq.~\eqref{eq:consensus}}
\State $\tilde{\boldsymbol{\omega}}^{(t)} \leftarrow \mathrm{RegimeAwareMultiplier}\!\left(\boldsymbol{\omega}, \xi^{(t)}, \kappa^{(t)}\right)$ \quad\Comment{Eq.~\eqref{eq:multiplier}}

\Statex \textbf{Stage 1/2 Ensemble Blend}
\State $\beta_{S1}^{(t)} \leftarrow \bar{\beta}_{S1} + \delta_{S1} \cdot \tanh(-\Delta_{S2}^{(t)}/\tau_{S1})$ \quad\Comment{Eq.~\eqref{eq:s1_blend}}

\Statex \textbf{Stage 3: Grand Coalition Synthesis}
\State $\mathrm{report} \leftarrow \mathcal{L}.\mathrm{FormatShapleyReport}()$
\State $\mathbf{w}_{123}^{(t)} \leftarrow \mathrm{LLM}_{A4}\!\left(\{\mathbf{w}_i^{(t)}\}, \{\mathbf{w}_{ij}^{(t)}\}, \mathrm{report}\right)$

\Statex \textbf{Council Blend with Divergence Discount}
\State $\beta_{\mathrm{gc}}^{(t)} \leftarrow \max\!\left(0,\, \bar{\beta}_{\mathrm{gc}} + \delta_{\mathrm{gc}} \cdot \tanh(\Delta_{\mathrm{gc}}^{(t)}/\tau_{\mathrm{gc}})\right)$ \quad\Comment{Eq.~\eqref{eq:a4_blend}}
\State $\beta_{\mathrm{gc}}^{\mathrm{final}} \leftarrow \beta_{\mathrm{gc}}^{(t)} \cdot (1/2 + \kappa^{(t)}/2)$ \quad\Comment{Eq.~\eqref{eq:divergence_discount}}
\State $\mathbf{w}_{\mathrm{council}}^{(t)} \leftarrow \beta_{\mathrm{gc}}^{\mathrm{final}}\,\mathbf{w}_{123}^{(t)} + (1-\beta_{\mathrm{gc}}^{\mathrm{final}})\!\left[\beta_{S1}^{(t)}\sum_i \tilde{\omega}_i^{(t)}\mathbf{w}_i^{(t)} + (1-\beta_{S1}^{(t)})\sum_{(i,j)} p_{ij}^{(t)}\mathbf{w}_{ij}^{(t)}\right]$ \quad\Comment{Eq.~\eqref{eq:ensemble}}
\Statex \textbf{Risk Control Cascade (\S\ref{sec:risk})}
\State $\mathbf{w}^{(t)} \leftarrow \mathrm{RiskControl}\!\left(\mathbf{w}_{\mathrm{council}}^{(t)},\, \xi^{(t)},\, \mathbf{x}_2^{(t)}\right)$ \quad\Comment{see Appendix~\ref{app:overlays}}

\Statex \textbf{Update Shapley Value Ledger (Algorithm~\ref{alg:ledger})}
\State $\mathcal{L} \leftarrow \mathrm{ShapleyLedgerUpdate}\!\left(\mathcal{L},\, \mathbf{w}^{(t)},\, \mathbf{r}^{(t+1)}\right)$

\State \Return $\mathbf{w}^{(t)},\; \mathcal{L}$
\end{algorithmic}
\end{algorithm}

The two algorithms together implement the closed online learning loop described in \S\ref{sec:method}.

\section{Implementation Details}
\label{app:impl}

\paragraph{Dataset}
Table~\ref{tab:dataset} summarizes the full dataset inventory used by MRC, organized by observer agent, stream type, source, frequency, coverage window, and feature detail.
The paragraphs below expand each input group and describe the corresponding preprocessing pipeline. On decision day $t$ every agent sees only rows dated $t$ or earlier. Daily values for day $t$ are final when day $t$ closes, which is also when the portfolio is executed, and the portfolio earns the return from day $t$ to day $t{+}1$.

\begin{table}[tbh]
\centering
\caption{Dataset inventory organized by MRC observer agent.
\textit{ts}\,=\,time-series; \textit{img}\,=\,image; \textit{text}\,=\,natural language.
Row date: the date each row carries. $^{a}$The four daily FRED series shown to $A_3$ are US market quotes fixed before midnight UTC of their observation date, and FRED posts them one or more business days later. $^{b}$Sentiment labels and engagement counts come from a collection made after the evaluation window.}
\label{tab:dataset}
\label{tab:timestamps}
\vspace{0mm}
\setlength{\tabcolsep}{5pt}
\resizebox{\textwidth}{!}{%
\begin{tabular}{c l l c l c l c l}
\toprule
\textbf{Agent} & \textbf{Category} & \textbf{Stream} & \textbf{Type} &
\textbf{Source} & \textbf{Freq.} & \textbf{Row date} & \textbf{Raw Coverage} & \textbf{Detail} \\
\midrule
\multirow{2}{*}{$A_1$}
  & Historical Data  & OHLCV prices                    & \textit{ts}   & Binance                                         & Daily          & UTC day & 2017-08--2025-12 & 5 feat., 13 assets \\
  & K-line Charts    & Candlestick w/ MA overlays      & \textit{img}  & \texttt{mplfinance}                             & Daily          & UTC day & 2017-08--2025-12 & LB\,=\,90\,d, MA(5,20,60) \\
\midrule
\multirow{2}{*}{$A_2$}
  & Historical Data  & OHLCV prices          & \textit{ts}   & Binance                                         & Daily          & UTC day & 2017-08--2025-12 & 5 feat., 13 tokens \\
  & On-chain Data    & Per-token on-chain activity     & \textit{ts}   & CoinMetrics\,/\,DeFiLlama\,/\,Dune             & Daily          & UTC day & 2020-12--2025-12 & 3 feat., 13 tokens \\
\midrule
\multirow{4}{*}{$A_3$}
  & Social Media     & Twitter/X crypto sentiment      & \textit{text} & X (Twitter)                                     & Event-driven   & Day of post$^{b}$ & 2023-01--2025-12 & 7 features \\
  & Social Media     & Prediction-market probabilities & \textit{ts}   & Polymarket                                      & Daily          & Day of price & 2023-02--2025-12 & 2 features \\
  & Macro Indices    & FRED macroeconomic indicators   & \textit{ts}   & St.\ Louis Fed (FRED)                           & Daily/Monthly  & US trading day$^{a}$ & 2019-12--2025-12 & 10 series, 4 shown \\
  & Macro Indices    & Web3 macro fundamentals         & \textit{ts}   & Blockchain.com\,/\,Binance\,/\,Alternative.me   & Daily          & UTC day & 2018-01--2025-12 & 12 series \\
\bottomrule
\end{tabular}%
}
\end{table}

\paragraph{Market Data ($A_1$).}
Daily OHLCV bars are stored as HDF5 files (one per asset) and accessed via a 90-day rolling lookback window at each decision period.
Multi-scale statistics are computed at 1-day, 7-day, and 30-day horizons, covering returns, realized volatility, RSI, and ATR.
Candlestick K-line chart images are rendered with \texttt{mplfinance} at daily stride over the same 90-day window, with moving-average overlays at periods $(5, 20, 60)$, and passed as base64-encoded PNG attachments to the vision encoder.

\paragraph{On-chain Feature Assignment ($A_2$).}
Per-token on-chain activity metrics come from a proprietary dataset with 1{,}827 daily rows and 39 feature columns, corresponding to three features per token across 13 assets.
Each token's three features are drawn from five metric families: active address count (AdrActCnt), transaction count (TxCnt), native fee total (FeeTotNtv), chain TVL in USD (ChainTVL\_USD), and token transfer count (TokenTxCnt).
The specific assignment depends on protocol type and data availability: standard PoW/PoS L1s use (AdrActCnt, TxCnt, FeeTotNtv), DeFi-heavy chains (BNB, ETH, TRX) substitute ChainTVL\_USD, and ERC-20 tokens (LINK, PAXG) substitute TokenTxCnt for the unavailable native fee series.

\paragraph{Macro Data ($A_3$).}
Ten FRED macroeconomic series are provided to $A_3$: five at daily frequency (10-year real yield, USD broad index, overnight reverse repo, VIX, WTI crude) and five at monthly frequency (CPI, core PCE, unemployment rate, nonfarm payrolls, industrial production).
These are complemented by the Crypto Fear \& Greed Index presented as a 30-day trailing window, Polymarket prediction-market probabilities (209{,}128 rows across 575 topics), and Twitter/X crypto sentiment aggregated from ${\approx}$27{,}000 engagement-weighted tweets with human-assigned sentiment labels spanning all 13 assets from 2023 to 2025.
The Twitter/X signal constitutes one of several feature inputs to $A_3$ and is interchangeable with any engagement-weighted sentiment source without architectural change.

\paragraph{Social Sentiment Signal Construction.}
Polymarket prediction-market sentiment is computed per token by filtering slugs with token-specific keywords, retaining only \texttt{outcome\_norm = yes} rows, and mapping the mean implied probability $\bar{p}_k \in [0,1]$ to a sentiment score $(\bar{p}_k - 0.5)\times 2 \in [-1,+1]$.
For tokens with fewer than five matching slugs in a given window (primarily H1 2023), a general crypto-market fallback is applied using market-wide slug aggregates.

Twitter/X sentiment scores are derived per token from the curated tweet dataset as follows.
Each tweet may mention multiple assets via a pipe-delimited \texttt{affected\_cryptos} field; rows are first expanded so that each (tweet, token) pair becomes a separate record.
A sentiment label $s_m \in \{-1, 0, +1\}$ (bearish, neutral, bullish) and human-assigned confidence $c_m \in (0,1]$ are combined with an engagement weight
\[
  w_m \;=\; \log\!\bigl(1 + \ell_m + r_m + v_m/100\bigr) + 1,
\]
where $\ell_m$, $r_m$, $v_m$ are likes, reposts, and views respectively (the $+1$ baseline prevents zero-weight tweets from being discarded).
The token-level sentiment signal is then the engagement-weighted average
\[
  s_k \;=\; \frac{\sum_m s_m\, c_m\, w_m}{\sum_m w_m}.
\]

\paragraph{Data Cleaning.}
All cleaning is applied after raw loading and before agent dispatch.
FRED daily series (VIX, DXY, real yield, overnight reverse repo, WTI) are forward-filled with a 5-day limit to fill weekends and holidays; FRED monthly series (CPI, core PCE, unemployment, nonfarm payrolls, industrial production) are held constant until the next scheduled release.
BTC blockchain and Binance-sourced macro series are forward-filled with a 5-day limit.
The Fear \& Greed Index and per-token on-chain series are forward-filled with a 5-day limit.
No look-ahead filling is performed: forward-fill propagates only the last observed value.
No survivorship bias is introduced: all 13 assets were actively traded throughout the evaluation window.

\paragraph{LLM Output Enforcement.}
All agents use \texttt{force\_json=True}, implemented via a two-layer fallback: (1) the user prompt is appended with a structured output trigger instructing the model to respond with a valid JSON object matching the declared schema; (2) if the primary response fails schema validation, the system retries once with a stricter prefill (\texttt{\{"}) to guide the tokenizer toward JSON continuation.
LLM temperature is fixed at $T_{\mathrm{LLM}}=0.7$ across all agents (primary configuration;
see Table~\ref{tab:ablation_temperature} for temperature sensitivity); no top-$p$ or top-$k$ truncation is applied beyond the model defaults.
The 90-day rolling lookback window is passed as a structured table in the system prompt; K-line chart images are passed as base64-encoded PNG attachments to the vision encoder.

\paragraph{Backbone Configuration and Evaluation Protocol.}
All LLM agents access the Qwen3-VL family via OpenRouter, and we evaluate seven backbone variants spanning instruct and extended-thinking modes (\texttt{8b}, \texttt{32b}, \texttt{30b-a3b}, \texttt{235b-a22b}).
All Shapley, EWP, and WTA hyperparameters are fixed before the evaluation window and are not tuned on test data: Bayesian concentration $\lambda{=}30$, EWP e-fold decay $h{=}252$, characteristic-function weights $\gamma_\rho{=}0.4$ and $\gamma_\mu{=}0.6$ (see sensitivity analysis below), WTA rolling window $n_{\mathrm{win}}{=}30$, dominance threshold $\theta_{\mathrm{WTA}}{=}1.8$, WTA concentration weight $\omega_{\mathrm{wta}}{=}0.80$, regime thresholds $\xi_+{=}{+}0.30$ (bull) and $\xi_-{=}{-}0.30$ (bear) with volatile otherwise, and regime-score attenuation factor $0.5$ applied when the 7-day direction conflicts with the 30-day trend and the 7-day magnitude exceeds $30\%$ of the 30-day signal. Each agent operates on a 90-day rolling lookback window, decisions use only information dated $t$ or earlier, results in the main paper are averaged over five seeds, and all baselines share the same matched execution constraints.

\paragraph{Blend Ratio Hyperparameters.}
The Stage~1/2 blend (Eq.~\eqref{eq:s1_blend}) and grand-coalition blend (Eq.~\eqref{eq:a4_blend}) use the following fixed values, set prior to the evaluation window:
$\bar{\beta}_{S1}{=}0.90$, $\delta_{S1}{=}0.09$, $\tau_{S1}{=}0.075$;
$\bar{\beta}_{\mathrm{gc}}{=}0.15$, $\delta_{\mathrm{gc}}{=}0.20$, $\tau_{\mathrm{gc}}{=}0.10$. The center $\bar{\beta}_{S1}{=}0.90$ encodes the prior that Stage-1 individual analyses dominate when Stage-2 debate provides no additional coalition value ($\Delta_{S2}^{(t)}{=}0$); $\delta_{S1}$ bounds the maximum deviation to ${\pm}0.09$, keeping $\beta_{S1}^{(t)}$ in the natural range $(0.81,\,0.99)$ without explicit clipping. The center $\bar{\beta}_{\mathrm{gc}}{=}0.15$ reflects the empirical finding that the grand-coalition LLM readout contributes modestly on average; the wider scale $\delta_{\mathrm{gc}}{=}0.20$ allows it to rise toward $0.35$ when $A_4$ synthesis demonstrably outperforms the Stage-1 ensemble.

\begin{table}[tbh]
\centering
\small
\begin{tabular}{lccc}
\toprule
 & $\xi=+1$ (bull) & $\xi=0$ (neutral) & $\xi=-1$ (bear) \\
\midrule
$\psi_1$ (Technical, $A_1$) & 1.50 & 1.20 & 0.60 \\
$\psi_2$ (On-chain, $A_2$)  & 0.90 & 1.00 & 0.90 \\
$\psi_3$ (Macro, $A_3$)     & 0.60 & 0.80 & 1.50 \\
\bottomrule
\end{tabular}
\caption{Regime-conditional multiplier anchor values for $\psi_i(\xi^{(t)})$. Values are linearly interpolated, and the neutral column ($\xi=0$) serves as the prior when no regime signal is present.}
\label{tab:regime_mults}
\end{table}

\paragraph{Regime-Aware Multiplier Values ($\psi_i$).}
The regime-conditional scalar $\psi_i$ (Eq.~\eqref{eq:multiplier}) is obtained by linearly interpolating between three anchor points at $\xi{=}{+1}$ (strong bull), $\xi{=}0$ (neutral), and $\xi{=}{-1}$ (strong bear).
The anchor values, fixed prior to the evaluation window, are listed in Table~\ref{tab:regime_mults}.

\paragraph{Characteristic Function Weight Sensitivity ($\gamma_\rho$).}
For brevity, we omit the history argument and write
$v(S) \equiv v(S;\mathcal{H}_t)$,
$\mathrm{SR}_{\mathrm{EW}}(S) \equiv \sqrt{365}\,\hat{\mu}_{\mathrm{EW}}(S;\mathcal{H}_t) / \hat{\sigma}_{\mathrm{EW}}(S;\mathcal{H}_t)$.
Under this shorthand, the blend is
$v(S) = \gamma_\rho \cdot \mathrm{SR}_{\mathrm{EW}}(S) + \gamma_\mu \cdot \mu_{\mathrm{ann}}(S)$.
The weight $\gamma_\rho{=}0.4$ is chosen so that both terms in $v(S)$ contribute at comparable magnitudes: $\mathrm{SR}_{\mathrm{EW}}(S)\in[0.8,1.6]$ and $\mu_{\mathrm{ann}}(S)\in[0.3,1.2]$ across the evaluation window.

This value was set by inspecting pre-evaluation return magnitudes and was not tuned on test data.
A sweep over $\gamma_\rho\in\{0.2,0.4,0.6,0.8\}$ shows rank correlation of $\phi_i$ vectors above 0.93 across all tested values.

\paragraph{Asymmetric EMA Smoothing.}
The council output $\mathbf{w}_{\mathrm{council}}^{(t)}$ (Eq.~\eqref{eq:ensemble}) is smoothed before execution via an asymmetric EMA:
\[
  \mathbf{w}^{(t)} \;=\; \eta_s^{(t)}\,\mathbf{w}_{\mathrm{council}}^{(t)} \;+\; (1-\eta_s^{(t)})\,\mathbf{w}^{(t-1)},
\]
\[
  \eta_s^{(t)} \;=\;
  \begin{cases}
    0.70 & \text{if } (\mathbf{w}_{\mathrm{council}}^{(t)})_k \geq w_k^{(t-1)} \text{ (building position)}, \\
    0.78 & \text{if } (\mathbf{w}_{\mathrm{council}}^{(t)})_k < w_k^{(t-1)} \text{ (de-risking)}.
  \end{cases}
\]
The higher de-risking speed ($+0.08$) ensures that defensive signals are acted on promptly, while the lower build speed suppresses day-to-day noise in conviction signals.
Smoothing is applied element-wise before the risk control cascade, after which the weight vector is re-normalized to the simplex.

\paragraph{Portfolio Constraints.}
Each token weight is capped at $0.40$ and the cash allocation at $0.30$ before the risk overlays.
After the full overlay cascade, the weight vector is projected onto the feasible set $\{\mathbf{w} \geq 0,\; \mathbf{1}^\top\mathbf{w} + c = 1,\; w_k \leq 0.40,\; c \leq 0.30\}$ via proportional rescaling.
Transaction costs are set to zero in the main evaluation, and the portfolio is rebalanced at the close of each trading day.

\paragraph{Risk Control Parameters.}
The following parameter values are fixed prior to the evaluation window:
BTC dominance bandwidth $\tau_d = 0.15$;
transition-buffer depth $c_{\mathrm{trans}} = 0.35$,
bandwidth $\tau_{\mathrm{trans}} = 0.25$;
drawdown scale $\tau_{\mathrm{dd}} = 0.15$,
protection coefficient $\nu_{\mathrm{dd}} = 0.40$.
Full overlay cascade specifications are in Appendix~\ref{app:overlays}.

\section{Risk Control Overlay Cascade: Full Specification}
\label{app:overlays}

The Risk Control layer is organized into four groups matching \S\ref{sec:risk}, comprising seven overlays applied in a fixed cascade.
All overlays are continuous functions of their driving signals with no binary switches, and all parameter values are fixed prior to the evaluation window.

\paragraph{Momentum Overlay.}
Let $z_k^{(t)}$ denote the cross-sectional 30-day return $z$-score of asset $k$.
The driving signal is $\tilde{z}_k^{(t)}=\tanh(z_k^{(t)}/1.5)$, a soft-capped version that bounds the tilt magnitude. Each weight is scaled by
\[
   w_k\cdot\bigl(1 + \eta(\xi^{(t)})\cdot\tilde{z}_k^{(t)}\bigr),
\]
where $\eta(\xi)=0.08+0.35\max(0,\xi)\in[0.08,\,0.43]$.

\paragraph{Dominance Overlay.}

The driving BTC Dominance signal is the 30-day BTC return minus the EW altcoin basket return, $\Delta_{\mathrm{BTC\text{-}EW}}^{(t)}$, compressed via $d^{(t)} = \tanh(\Delta_{\mathrm{BTC\text{-}EW}}^{(t)}/\tau_d)$,
where $\tau_d = 0.15$. In a BTC season ($d>0$, volatile or bear regime), weight $d\times c_d$ is shifted from \{ETH, ADA, LINK, DOGE, XLM, XRP, BCH\} to BTC (60\%), TRX (25\%), and ZEC (15\%), each capped at $0.30$; the coefficient is $c_d=0.45$ in volatile and $0.30$ in bear.
In an alt season ($d<0$, bull regime), weight $|d|\times 0.20$ is cut from \{BTC, TRX\} and redistributed proportionally to \{ADA, XLM, DOGE, XRP, BCH, ETH\}.
There is no activation gate; scaling is continuous in $d$.

In the volatile regime ($|\xi|< 0.30$), the BTC allocation is lifted to a Floor minimum of $18\%$:
\[
   \max(w_{\mathrm{BTC}},\;0.18).
\]
Any shortfall is recovered proportionally from altcoin longs.
Empirical motivation: BTC returned $+40.9\%$ versus altcoin EW $-8.5\%$ in volatile periods, and $47\%$ of volatile days had $w_{\mathrm{BTC}}<18\%$ before this floor was introduced.

\paragraph{Intra-Regime Controls.}

Active only in the bear regime ($\xi<-0.30$).
The signal is the mean on-chain $z$-score differential between BTC and the altcoin basket across two network-activity metrics (active addresses and transaction count):
$\delta_{\mathrm{oc}}^{(t)}=\bar{z}_{\mathrm{BTC}}(\mathrm{AdrActCnt,\,TxCnt})-\bar{z}_{\mathrm{alts}}$.
The BTC allocation is increased by
\[
  \Delta_{\mathrm{BTC}}^{\mathrm{oc}} \;=\; 0.08\cdot\tanh\!\bigl(\delta_{\mathrm{oc}}^{(t)}/1.5\bigr),
\]
provided $\Delta_{\mathrm{BTC}}^{\mathrm{oc}}>0.005$, with $w_{\mathrm{BTC}}$ capped at $0.30$.
The maximum tilt is $8\,\mathrm{pp}$, and the deficit is taken proportionally from altcoin longs.
This control is skipped if $A_2$ on-chain data are unavailable.

In the volatile regime, the target cash allocation is a smooth function of regime conviction:
\[
  c_{\mathrm{tgt}} \;=\; 0.08 + 0.17\exp\!\bigl(-|\xi^{(t)}|/0.12\bigr),
\]
yielding $c_{\mathrm{tgt}}=0.25$ at $|\xi|=0$ (deep volatile) and $c_{\mathrm{tgt}}\approx0.094$ at $|\xi|=0.30$ (near the bull/bear boundary).
In the bull regime a hard cap of $c_{\mathrm{tgt}}=0.08$ applies.
Adjustment is bidirectional: cash below $c_{\mathrm{tgt}}$ is raised by trimming longs, while cash above $c_{\mathrm{tgt}}$ is redeployed into longs.

Applied to bull-to-volatile transitions.
The downward score drop is $\delta_\xi^{(t)}=\max(0,\,\xi^{(t-1)}-\xi^{(t)})$, and the scaling factor is
\begin{equation}
  s_{\mathrm{trans}}^{(t)} = 1 - c_{\mathrm{trans}}\cdot\tanh\!\bigl(\delta_\xi^{(t)}/\tau_{\mathrm{trans}}\bigr),
  \label{eq:transition}
\end{equation}
($c_{\mathrm{trans}} = 0.35$, $\tau_{\mathrm{trans}} = 0.25$),
so that $\mathbf{w}\leftarrow s_{\mathrm{trans}}^{(t)}\cdot\mathbf{w}$ with freed weight moved to cash.
A mild drop ($\delta_\xi=0.05$) gives $s_{\mathrm{trans}}^{(t)}\approx0.93$ (7\% de-risk);
a severe collapse ($\delta_\xi = 0.60$) gives $s_{\mathrm{trans}}^{(t)} \approx 0.66$ (34\% de-risk).

\paragraph{Drawdown Protection.}
Driven by the current drawdown from the rolling peak, $\mathrm{DD}^{(t)}$,
gated by $g(\xi^{(t)})=\max(0,-\xi^{(t)})$ so that the overlay is suppressed
in bull regimes ($g=0$):
\[
  s_{\mathrm{dd}}^{(t)} \;=\;
    1 - g(\xi^{(t)})\cdot\tanh\!\bigl(\mathrm{DD}^{(t)}/\tau_{\mathrm{dd}}\bigr)
    \cdot\nu_{\mathrm{dd}},
  \qquad
  \mathbf{w} \;=\; s_{\mathrm{dd}}^{(t)}\cdot\mathbf{w}.
\]
Freed weight goes to cash, with a dynamic cash cap
$c_{\mathrm{cap}}^{(t)} = 0.08+0.22\cdot g(\xi^{(t)})\in[0.08,\,0.30]$.
As a concrete example, with $g(\xi^{(t)})=1$ and
$\mathrm{DD}^{(t)}=\tau_{\mathrm{dd}}=15\%$:
$s_{\mathrm{dd}}^{(t)} = 1 - \tanh(1)\times\nu_{\mathrm{dd}}
\approx 1 - 0.762\times 0.40 \approx 0.695$,
implying a $30.5\%$ de-risking (see Eq.~\eqref{eq:dd_protection}).

\paragraph{Cascade Order and Continuity.}
The four groups are applied in the listed order, with each group acting on the weight vector as modified by all preceding groups.
This ordering reflects an implicit priority structure: the Momentum and Dominance Overlays establish the directional tilt, the Intra-Regime Controls refine the BTC allocation and cash level within each regime, and Drawdown Protection operates as a final risk-budget constraint.
All seven overlays use $\tanh$, $\exp$, or clamp operations with no Heaviside functions, ensuring that $\mathbf{w}^{(t)}$ is a Lipschitz-continuous function of the driving signals $(\xi^{(t)}, \{z_k^{(t)}\}, \mathrm{DD}^{(t)}, \delta_{\mathrm{oc}}^{(t)})$. After the full cascade, the weight vector is normalized to $\mathbf{1}^\top\mathbf{w} + c = 1$.

\section{Proof of EWP Early-Period Influence Attenuation}
\label{app:ewp}

We provide a formal statement and proof sketch showing that the Exponentially Weighted Protector (EWP) exponentially attenuates the influence of cold-start observations on the coalition characteristic function estimate, in contrast to the EW estimator whose path dependence decays only as $O(1/t)$.

\begin{proposition}[EWP Early-Period Influence Attenuation]
\label{prop:ewp}
Let $\{\hat{\mu}_{\mathrm{EW}}(S)\}$ be the EWP estimate of the coalition mean return (Eq.~\eqref{eq:ewcv}) with e-fold decay period $h > 0$, computed at time $t$ from return history $\{R_S^{(\tau)}\}_{\tau=1}^{t}$.
Define the early-period influence ratio as:
\begin{equation}
  \iota_{\mathrm{EW}}(n_0, t) = \frac{\sum_{\tau=1}^{n_0} w_\tau}{\sum_{\tau=1}^{t} w_\tau}
\end{equation}
where $n_0$ is the length of the cold-start phase ($n_0 \ll t$).
Then:
\begin{equation}
  \iota_{\mathrm{EW}}(n_0, t) = \frac{e^{-(t-n_0)/h} - e^{-t/h}}{1 - e^{-t/h}} \;\leq\; e^{-(t-n_0)/h}
\end{equation}
which converges to zero at exponential rate $O(e^{-(t-n_0)/h})$ as $t \to \infty$.
In contrast, the EW estimator has early-period influence ratio $n_0/t$, which converges to zero only at rate $O(1/t)$.
\end{proposition}

\begin{proof}[Proof Sketch]
Recall the EWP decay weights $w_\tau = e^{-(t-\tau)/h}$ for $\tau = 1,\ldots,t$ (Eq.~\eqref{eq:ewp_decay}).
The sum of all weights is:
\begin{equation}
  \sum_{\tau=1}^{t} w_\tau = \sum_{\tau=1}^{t} e^{-(t-\tau)/h} = \frac{1 - e^{-t/h}}{1 - e^{-1/h}} \;\approx\; h \cdot (1 - e^{-t/h})
\end{equation}
for large $h$ (continuous approximation).
The contribution of the early periods $\tau \leq n_0$ to the total weight is:
\begin{equation}
  \sum_{\tau=1}^{n_0} w_\tau = \sum_{\tau=1}^{n_0} e^{-(t-\tau)/h} = e^{-t/h} \sum_{\tau=1}^{n_0} e^{\tau/h}
  = e^{-t/h} \cdot \frac{e^{(n_0+1)/h} - e^{1/h}}{e^{1/h} - 1} \;\approx\; h \cdot e^{-(t-n_0)/h}
\end{equation}
Therefore, the influence ratio satisfies:
\begin{equation}
  \iota_{\mathrm{EW}}(n_0, t) \approx \frac{h \cdot e^{-(t-n_0)/h}}{h \cdot (1 - e^{-t/h})} = \frac{e^{-(t-n_0)/h}}{1 - e^{-t/h}} \;\leq\; 2 e^{-(t-n_0)/h}
\end{equation}
for $t \geq h\ln 2$.
This is $O(e^{-(t-n_0)/h})$, which decays exponentially in $(t - n_0)$.
At $t = n_0 + 3h$ (three e-fold periods after cold start), the early influence is bounded by $2e^{-3} \approx 10\%$; at $t = n_0 + 5h$, it is below $1.4\%$.

\medskip
By contrast, the EW mean estimator assigns weight $1/t$ to every historical observation, giving early-period influence ratio exactly $n_0/t$.
For $n_0 = 100$ cold-start periods, this ratio decays only as $100/t$: it takes $t = 10{,}000$ periods to reduce early influence below $1\%$, versus $t = n_0 + 5h$ for EWP with any practical $h$ (e.g., $h = 60$ gives $t = 400$ periods).

\medskip
Applying this result to the EWP estimate $\hat{\mu}_{\mathrm{EW}}(S)$: the absolute perturbation of the characteristic function estimate due to cold-start returns $\{R_S^{(\tau)}\}_{\tau \leq n_0}$ is bounded by $2\|R\|_\infty \cdot e^{-(t-n_0)/h}$, which converges to zero exponentially.
The raw EWP-based Shapley values therefore converge to their oracle values (computed from the true post-cold-start distribution) at an exponential rate, justifying the early-period influence attenuation claim.
\end{proof}

\paragraph{Remark (Path Dependence).}
The proposition implies that any random seed, initialization artefact, or unusual market event confined to the first $n_0$ periods of deployment has negligible influence on the characteristic function after $t \approx n_0 + 3h$ periods.
For the MRC configuration with $h = 252$ trading days (one calendar year), the influence of the cold-start phase falls below $2e^{-3} \approx 10\%$ after $t \approx n_0 + 756$ periods and below $1.4\%$ after $t \approx n_0 + 5h = 1{,}260$ periods.
In the experimental window of 1{,}037 periods, cold-start artefacts are therefore substantially suppressed by period 300 onward.

\paragraph{Connection to Shapley Machine and Axiomatic Justification}
\label{rem:shapley_machine}
Beyond its path-dependence benefits, the EWP's geometric structure provides an axiomatic grounding for MRC's credit assignment via the Shapley Machine system of \citet{wang2025shapley}.
Their central result (Theorem~2 therein) establishes that TD($\lambda_{\mathrm{TD}}$)-style geometric discounting is equivalent to the Additivity axiom of Shapley values: basis-game decomposition coefficients $k_C = \sum_{T \subseteq C}(-1)^{|C|-|T|}v(T)$, ordered by coalition size, naturally form the geometric sequence $(1-\lambda_{\mathrm{TD}})\lambda_{\mathrm{TD}}^{n-1}$ identifying $\lambda_{\mathrm{TD}}$-returns.
(Here $\lambda_{\mathrm{TD}}$ denotes the TD discount factor from \citealt{wang2025shapley}, distinct from MRC's Bayesian concentration parameter $\lambda=30$.)

The EWP weights $w_\tau = e^{-(t-\tau)/h}$ instantiate this geometric structure in continuous time, with the discrete ratio $e^{-1/h}$ playing the role of $\lambda_{\mathrm{TD}}$.
This means the EWP-estimated characteristic functions $v(S)$ (Eq.~\eqref{eq:ewcv}) are, by construction, consistent with the Additivity axiom.
The raw Shapley values themselves continue to rely on the standard four-axiom characterization proved in Appendix~\ref{app:shapley}.
The remaining axioms are satisfied explicitly:
\begin{itemize}
  \item \textbf{Efficiency:} $\sum_{i=1}^{3} \phi_i = v(\mathcal{N})$ is verified by direct computation in Appendix~\ref{app:shapley}.
  \item \textbf{Symmetry:} agents enter Eqs.~\eqref{eq:phi1}--\eqref{eq:phi3} symmetrically; permuting agent indices permutes the Shapley values identically.
  \item \textbf{Dummy player:} Appendix~\ref{app:shapley} verifies that zero-marginal-contribution agents receive zero raw Shapley value.
\end{itemize}
Accordingly, by the uniqueness theorem of \citet{shapley1953value}, the raw Shapley values used in MRC form the unique allocation satisfying all four axioms, providing a cooperative game-theoretic foundation for the credit-assignment layer.
The deployable portfolio weights are derived from these scores after truncation and normalization.

\medskip
\noindent\textbf{Scope Note.}
The Shapley Machine of \citet{wang2025shapley} operates on RL state-action values in open multi-agent systems, whereas MRC applies the same geometric discounting to financial coalition return histories.
The mathematical connection (geometric weights $\leftrightarrow$ Additivity axiom) holds in both settings by the same basis-game argument.
Proposition~\ref{prop:ewp} above provides the additional cold-start path-dependence guarantee not addressed in \citealt{wang2025shapley}, which focuses on stationary cooperative games.

\section{Three-Player Shapley Value Closed-Form Derivation}
\label{app:shapley}

We derive the closed-form three-player Shapley values from the general formula and verify that they satisfy the four Shapley axioms.

\subsection*{Derivation via Permutation Enumeration}

For $\mathcal{N} = \{1, 2, 3\}$, there are $|\mathcal{N}|! = 6$ orderings (permutations) of the players.
The Shapley value of player $i$ is the average marginal contribution of player $i$ over all permutations.
For player 1, the six permutations and the coalition $S$ preceding player 1's entry are:

\begin{center}
\begin{tabular}{cccc}
\hline
Order & $S$ before 1 & $S \cup \{1\}$ & Marginal contribution \\
\hline
$(1,2,3)$ & $\emptyset$ & $\{1\}$ & $v_1 - 0 = v_1$ \\
$(1,3,2)$ & $\emptyset$ & $\{1\}$ & $v_1$ \\
$(2,1,3)$ & $\{2\}$ & $\{1,2\}$ & $v_{12} - v_2$ \\
$(3,1,2)$ & $\{3\}$ & $\{1,3\}$ & $v_{13} - v_3$ \\
$(2,3,1)$ & $\{2,3\}$ & $\{1,2,3\}$ & $v_{123} - v_{23}$ \\
$(3,2,1)$ & $\{2,3\}$ & $\{1,2,3\}$ & $v_{123} - v_{23}$ \\
\hline
\end{tabular}
\end{center}

Averaging over the six permutations (each with probability $1/6$):
\begin{align}
  \phi_1 &= \frac{1}{6}\bigl[v_1 + v_1 + (v_{12}-v_2) + (v_{13}-v_3) + (v_{123}-v_{23}) + (v_{123}-v_{23})\bigr] \notag \\
         &= \frac{2v_1 + (v_{12}-v_2) + (v_{13}-v_3) + 2(v_{123}-v_{23})}{6} \notag \\
         &= \tfrac{1}{3}v_1 + \tfrac{1}{6}(v_{12}-v_2) + \tfrac{1}{6}(v_{13}-v_3) + \tfrac{1}{3}(v_{123}-v_{23})
         \label{eq:phi1}
\end{align}
By symmetry of the enumeration procedure, the formulas for $\phi_2$ and $\phi_3$ follow with indices permuted:
\begin{align}
  \phi_2 &= \tfrac{1}{3}v_2 + \tfrac{1}{6}(v_{12}-v_1) + \tfrac{1}{6}(v_{23}-v_3) + \tfrac{1}{3}(v_{123}-v_{13})
         \label{eq:phi2} \\
  \phi_3 &= \tfrac{1}{3}v_3 + \tfrac{1}{6}(v_{13}-v_1) + \tfrac{1}{6}(v_{23}-v_2) + \tfrac{1}{3}(v_{123}-v_{12})
         \label{eq:phi3}
\end{align}
which are Eqs.~\eqref{eq:phi1}--\eqref{eq:phi3}.

An equivalent derivation via the coalition-weighting formula gives the same result: each coalition $S$ of size $|S|$ contributes to player $i$'s Shapley value with weight $|S|!(|\mathcal{N}|-|S|-1)!/|\mathcal{N}|!$, which for $|\mathcal{N}|=3$ equals $1/3$ for $|S|=0$ (empty coalition preceding $i$'s solo entry), $1/6$ for $|S|=1$ (pairwise coalitions without $i$), and $1/3$ for $|S|=2$ (grand coalition entry).

\subsection*{Axiom Verification}

We verify that Eqs.~\eqref{eq:phi1}--\eqref{eq:phi3} satisfy the four Shapley axioms:

\paragraph{Efficiency.}
\begin{align}
  \phi_1 + \phi_2 + \phi_3
  &= \tfrac{1}{3}(v_1+v_2+v_3) \notag \\
  &\quad + \tfrac{1}{6}\bigl[(v_{12}-v_2)+(v_{12}-v_1)+(v_{13}-v_3)+(v_{13}-v_1)+(v_{23}-v_3)+(v_{23}-v_2)\bigr] \notag \\
  &\quad + \tfrac{1}{3}\bigl[(v_{123}-v_{23})+(v_{123}-v_{13})+(v_{123}-v_{12})\bigr]
\end{align}
Collecting terms: the individual terms contribute $\tfrac{1}{3}(v_1+v_2+v_3)$; each pairwise coalition $v_{ij}$ appears with coefficient $\tfrac{1}{6}+\tfrac{1}{6} = \tfrac{1}{3}$ from the second group, and $-\tfrac{1}{3}$ from the third group (as the complement), canceling exactly; and $v_{123}$ appears with coefficient $3 \times \tfrac{1}{3} = 1$.
Thus $\phi_1 + \phi_2 + \phi_3 = v_{123}$, verifying efficiency.

\paragraph{Note on Efficiency under EWP.}
The efficiency axiom states $\sum_i \phi_i = v(\mathcal{N})$, where $v(\mathcal{N})$ is the EWP-estimated characteristic function of the grand coalition (Eq.~\eqref{eq:char_fn}).
This is not equal to the current running Sharpe estimate of the grand-coalition portfolio.
The two quantities serve different roles: $v(\mathcal{N})$ is a discounted historical evaluator, whereas the current grand-coalition Sharpe is only a contemporaneous performance signal. This distinction is important for interpreting the Shapley values as historical attribution scores, not current-period performance contributions. When negative raw Shapley values are truncated via $\phi_i^+ = \max(\phi_i, 0)$ and weights are renormalized as
$\bar{\omega}_i = \phi_i^+ / \sum_j \phi_j^+$ (the truncation-only weight, prior to Bayesian mixing in Eq.~\eqref{eq:bayesian_update}),
the efficiency axiom $\sum_i \phi_i = v(\mathcal{N})$ no longer holds for $\bar{\omega}_i$~\cite{wang2022shaq,beechey2023explaining}. This is deliberate: raw Shapley values can be negative (when an agent is superfluous), and negative portfolio weights require short-selling, which is not available in MRC's long-only setting.
The raw $\phi_i$ satisfy all four axioms, and the truncation-only weights $\bar{\omega}_i$ preserve the axiomatic ordering without strict efficiency.

\paragraph{Symmetry.}
If $v(S \cup \{i\}) = v(S \cup \{j\})$ for all $S \not\ni i, j$, then by inspection of Eqs.~\eqref{eq:phi1}--\eqref{eq:phi3}, interchanging all occurrences of $i$ and $j$ in the formula for $\phi_i$ yields the formula for $\phi_j$, confirming $\phi_i = \phi_j$.

\paragraph{Dummy Player.}
Suppose player $i$ is a dummy player: $v(S \cup \{i\}) = v(S)$ for all $S \not\ni i$.
Then all marginal contributions of $i$ are zero, and by the formula, $\phi_i = 0$.

\paragraph{Additivity.}
For two cooperative games $(v, \mathcal{N})$ and $(v', \mathcal{N})$ with characteristic functions $v$ and $v'$, the Shapley value of the sum game $(v+v', \mathcal{N})$ is $\phi_i^{v+v'} = \phi_i^v + \phi_i^{v'}$, which follows directly from the linearity of the Shapley formula in $v$.

\subsection*{Connection to Multi-Stage Architecture}

In the MRC architecture, the three Shapley players correspond to $A_1$ (Market), $A_2$ (On-chain), and $A_3$ (Macro).
The individual characteristic functions $v_1, v_2, v_3$ reflect the solo performance of each agent's portfolio, the pairwise values $v_{12}, v_{13}, v_{23}$ reflect the debate-derived joint portfolios, and $v_{123}$ reflects the grand-coalition output.
Operationally, this grand-coalition output is instantiated by the Stage~3 readout $\mathbf{w}_{123}$ produced by $A_4$, which serves as a readout operator rather than as a fourth player.
Superadditivity ($v_{ij} \geq v_i + v_j$) is not assumed or required. In practice, some pairwise debates underperform their individual components during certain regimes, and the Bayesian adaptive weighting (Eq.~\eqref{eq:bayesian_update}) accommodates this empirically.

\section{Additional Results}
\label{app:additional_results}

This section provides supplementary comparisons, sensitivity analyses, and backbone
ablations referenced in the main text.
Figures~\ref{fig:compare_market}--\ref{fig:compare_shapley} compare MRC against each
baseline group on cumulative returns. Figure~\ref{fig:consensus_blend} shows dynamic
blend ratios and agent consensus.
Tables~\ref{tab:ablation_temperature} and~\ref{tab:burnin_sensitivity} report
temperature and burn-in sensitivity. Table~\ref{tab:slippage_sensitivity} reports
transaction-cost sensitivity.
Tables~\ref{tab:llm_backbone_s1}--\ref{tab:llm_backbone_cost} report backbone ablations across all three stages.

\clearpage
\begin{figure}[htbp]
    \centering
    \includegraphics[width=0.95\textwidth]{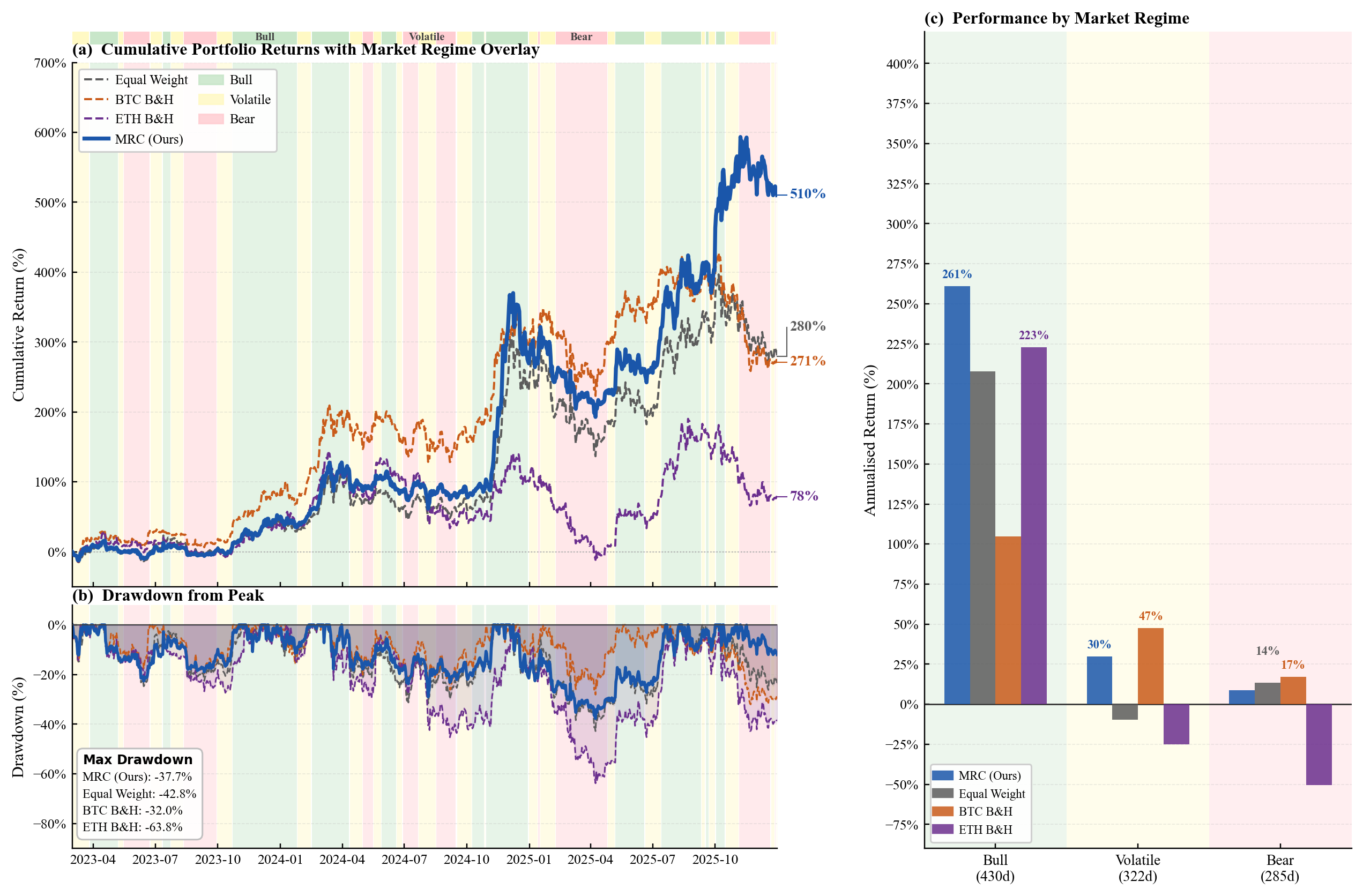}
    \vspace{-3mm}
    \caption{%
        MRC vs.\ passive benchmarks (Single seed ($T_{\mathrm{LLM}}{=}0.7$))
    }
    \vspace{-3mm}
    \label{fig:compare_market}
\end{figure}

\begin{figure}[htbp]
    \centering
    \includegraphics[width=0.95\textwidth]{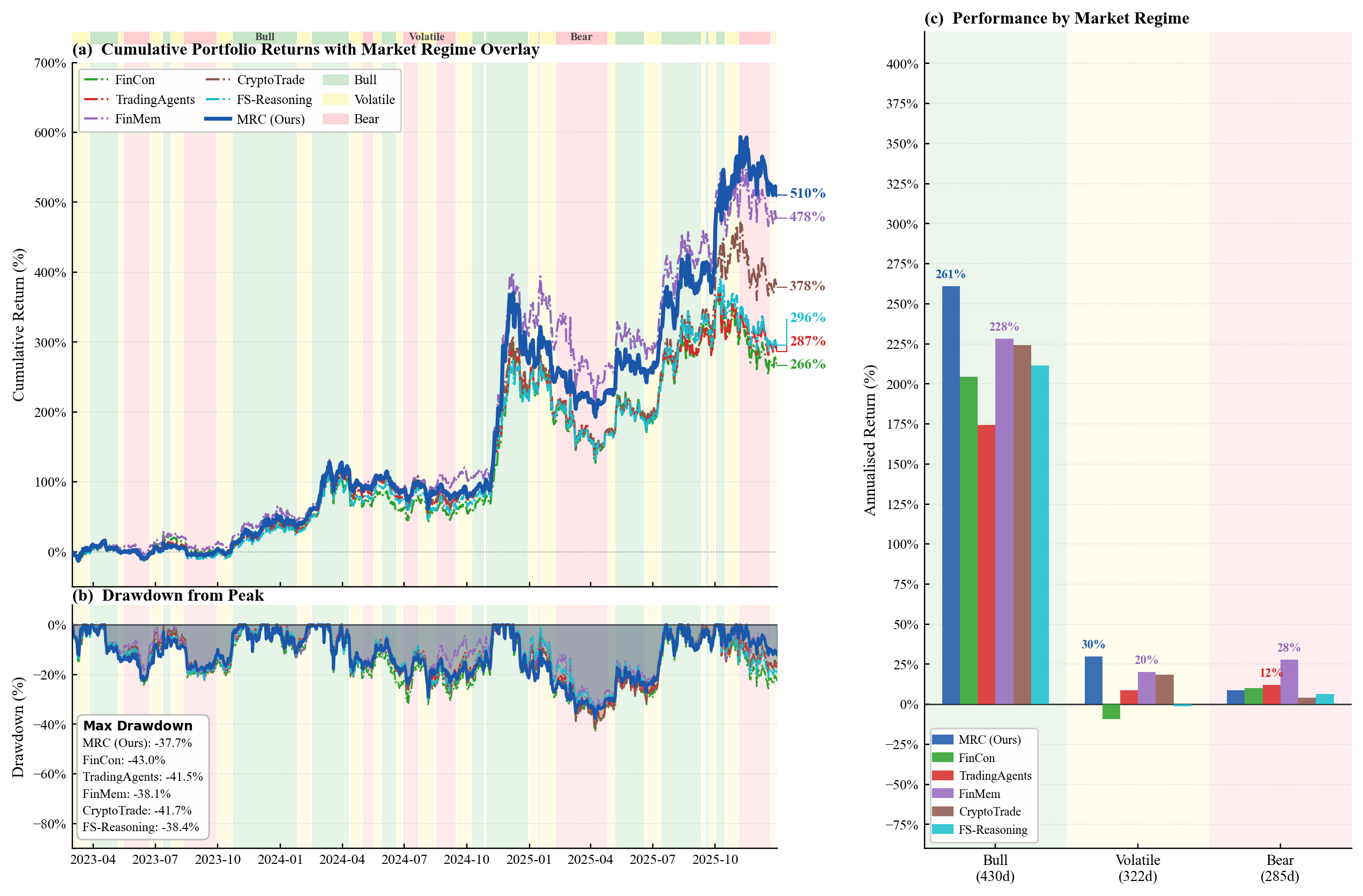}
    \vspace{-3mm}
    \caption{%
        MRC vs.\ LLM multi-agent baselines (Single seed ($T_{\mathrm{LLM}}{=}0.7$))
    }
    \vspace{-3mm}
    \label{fig:compare_llm}
\end{figure}

\begin{figure}[htbp]
    \centering
    \includegraphics[width=0.95\textwidth]{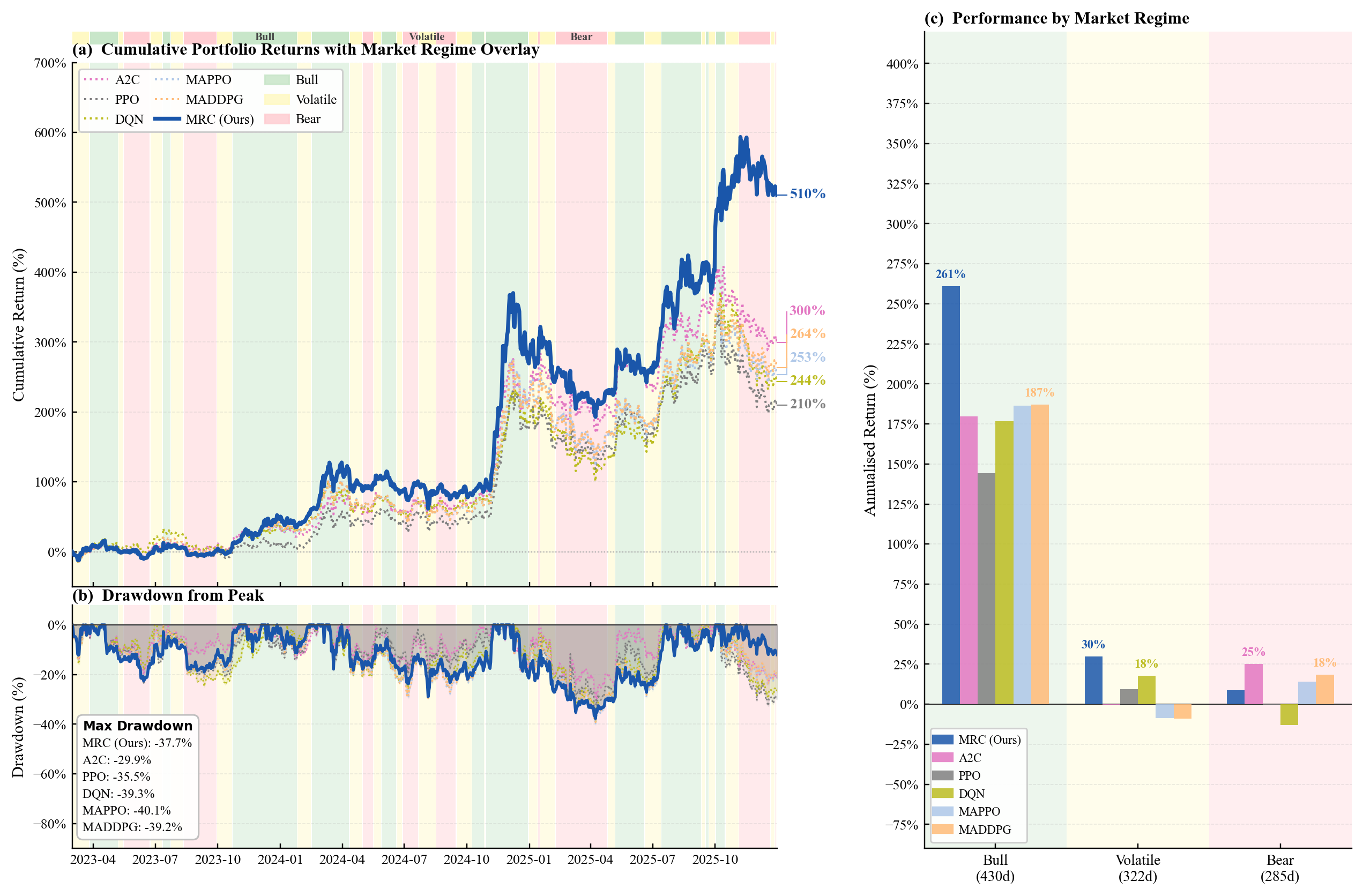}
    \vspace{-3mm}
    \caption{%
        MRC vs.\ DRL baselines (Single seed ($T_{\mathrm{LLM}}{=}0.7$))
    }
    \vspace{-3mm}
    \label{fig:compare_drl}
\end{figure}

\begin{figure}[htbp]
    \centering
    \includegraphics[width=0.95\textwidth]{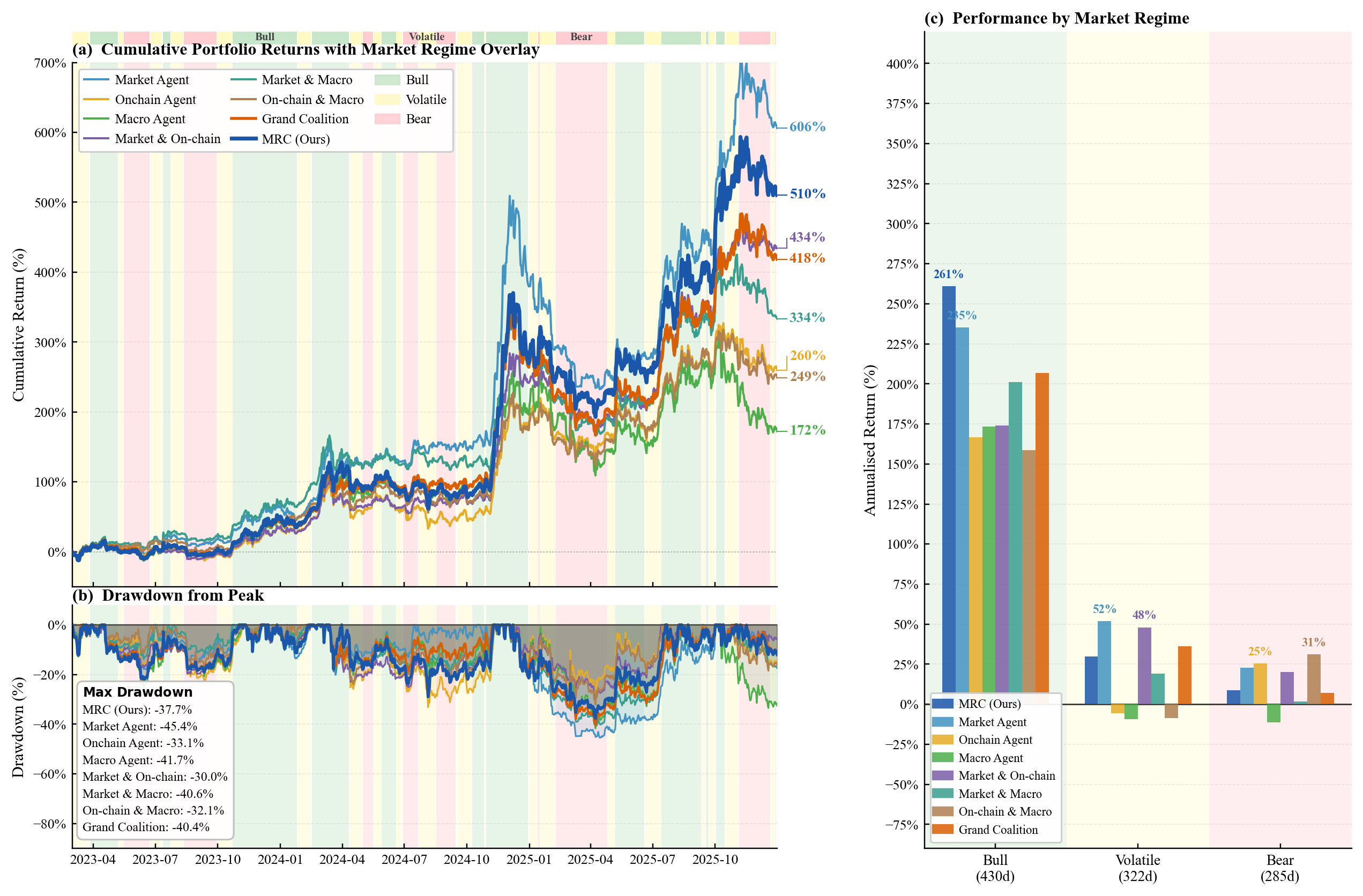}
    \vspace{-3mm}
    \caption{%
        Shapley coalition ablation: individual agents, pairwise coalitions,
        grand coalition, and full MRC (Single seed ($T_{\mathrm{LLM}}{=}0.7$))
    }
    \vspace{-3mm}
    \label{fig:compare_shapley}
\end{figure}

\begin{figure}[htbp]
    \centering
    \includegraphics[width=1.0\textwidth]{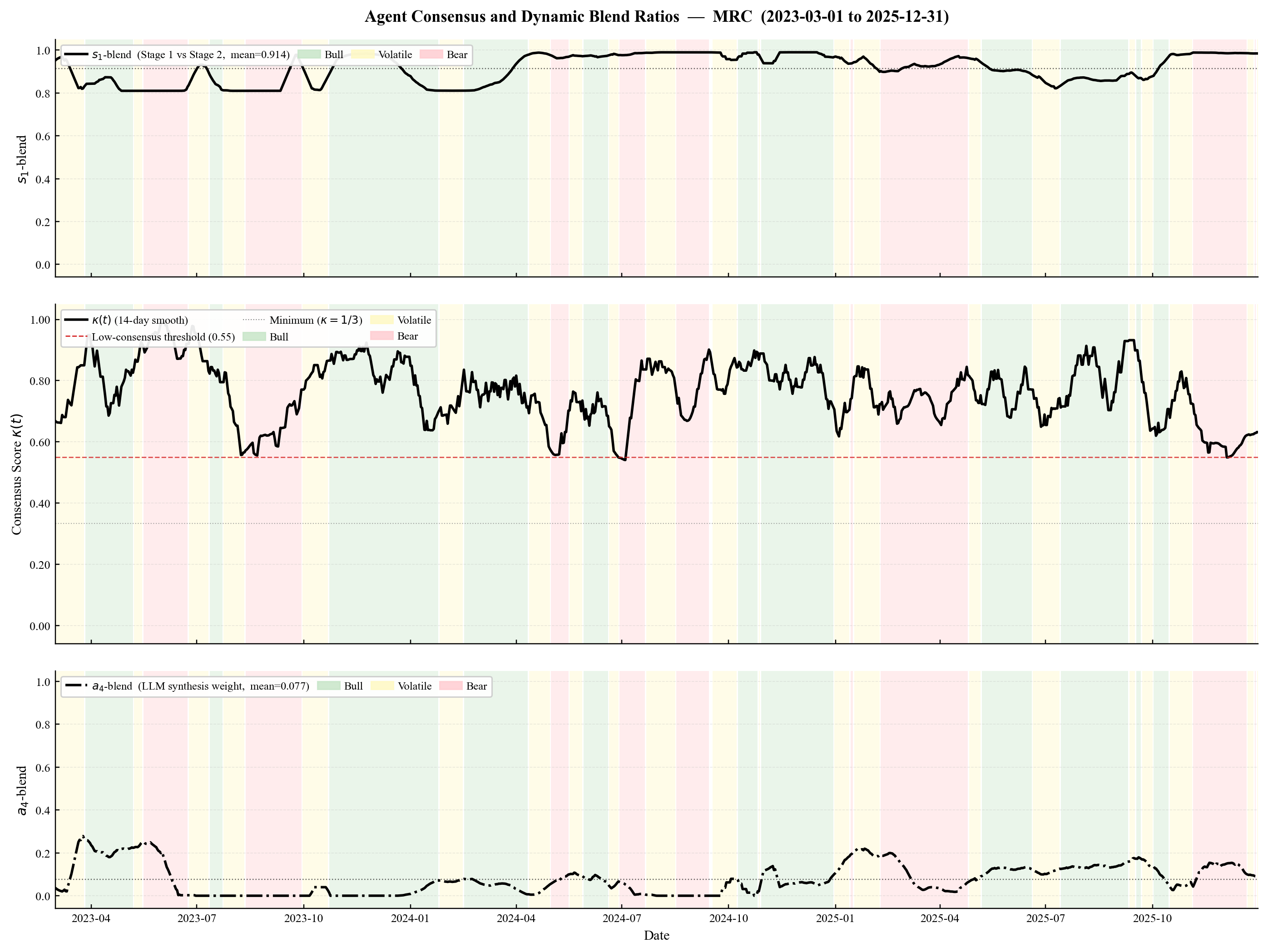}
    \caption{%
        \textbf{Dynamic blend ratios and agent consensus (2023-03-01 to 2025-12-31).} Top: Stage-1 blend $\beta_{S1}^{(t)}$ confirms that the Stage-1 ensemble dominates Stage-2 debate outputs across all regimes. Middle: consensus $\kappa^{(t)}$ (dashed line at 0.55 shown for visual reference), where lower consensus continuously reduces the grand-coalition readout blend via Eq.~\eqref{eq:divergence_discount}.
        Bottom: grand-coalition readout blend $\beta_{\mathrm{gc}}^{\mathrm{final}}$.
    }
    \label{fig:consensus_blend}
\end{figure}

\begin{table}[tbh]
\centering
\caption{Ablation analysis of LLM temperature ($T_{\mathrm{LLM}}$). Primary configuration is $T_{\mathrm{LLM}}{=}0.7$, and other temperature values are shown for sensitivity. \textcolor{red}{Red}: best per column; \textcolor{blue}{Blue}: second best. Single run per temperature setting.}
\label{tab:ablation_temperature}
\vspace{0mm}
\setlength{\tabcolsep}{3.2pt}
\resizebox{\textwidth}{!}{%
\begin{tabular}{l *{5}{cccc}}
\toprule
\multirow{3}{*}{\textbf{Method}} & \multicolumn{4}{c}{$T_{\mathrm{LLM}}=0.1$} & \multicolumn{4}{c}{$T_{\mathrm{LLM}}=0.3$} & \multicolumn{4}{c}{$T_{\mathrm{LLM}}=0.5$} & \multicolumn{4}{c}{$T_{\mathrm{LLM}}=0.7$} & \multicolumn{4}{c}{$T_{\mathrm{LLM}}=0.9$} \\
\cmidrule(lr){2-5} \cmidrule(lr){6-9} \cmidrule(lr){10-13} \cmidrule(lr){14-17} \cmidrule(lr){18-21}
 & CR\,\% & SR & MDD\,\% & IR & CR\,\% & SR & MDD\,\% & IR & CR\,\% & SR & MDD\,\% & IR & CR\,\% & SR & MDD\,\% & IR & CR\,\% & SR & MDD\,\% & IR \\
 & ($\uparrow$) & ($\uparrow$) & ($\downarrow$) & ($\uparrow$) & ($\uparrow$) & ($\uparrow$) & ($\downarrow$) & ($\uparrow$) & ($\uparrow$) & ($\uparrow$) & ($\downarrow$) & ($\uparrow$) & ($\uparrow$) & ($\uparrow$) & ($\downarrow$) & ($\uparrow$) & ($\uparrow$) & ($\uparrow$) & ($\downarrow$) & ($\uparrow$) \\
\midrule
\multicolumn{21}{l}{\small\textit{Market Benchmarks}} \\
~~ EW & 279.93 & 1.14 & 42.78 & 0.00 & 279.93 & 1.14 & 42.78 & 0.00 & 279.93 & 1.14 & 42.78 & 0.00 & 279.93 & 1.14 & 42.78 & 0.00 & 279.93 & 1.14 & 42.78 & 0.00 \\
~~ BTC & 270.94 & 1.23 & \textcolor{blue}{32.02} & -0.07 & 270.94 & 1.23 & \textcolor{blue}{32.02} & -0.07 & 270.94 & 1.23 & \textcolor{blue}{32.02} & -0.07 & 270.94 & 1.23 & \textcolor{blue}{32.02} & -0.07 & 270.94 & 1.23 & \textcolor{blue}{32.02} & -0.07 \\
~~ ETH & 78.45 & 0.63 & 63.75 & -0.59 & 78.45 & 0.63 & 63.75 & -0.59 & 78.45 & 0.63 & 63.75 & -0.59 & 78.45 & 0.63 & 63.75 & -0.59 & 78.45 & 0.63 & 63.75 & -0.59 \\
\midrule
\multicolumn{21}{l}{\small\textit{LLM-based Baselines}} \\
~~ FinCon & 245.50 & 1.11 & 40.40 & -1.38 & 257.90 & 1.14 & 42.00 & -1.03 & 262.70 & 1.13 & 41.80 & -0.99 & 266.50 & 1.12 & 43.02 & -1.13 & 286.60 & 1.17 & 40.90 & \textcolor{blue}{0.06} \\
~~ TradingAgents & 282.90 & 1.20 & 38.20 & -0.20 & 292.20 & 1.22 & 39.50 & -0.09 & 251.30 & 1.14 & 44.40 & -0.47 & 286.52 & 1.21 & 41.55 & -0.17 & 242.00 & 1.12 & 42.30 & -0.59 \\
~~ FinMem & \textcolor{blue}{348.40} & 1.30 & 42.50 & \textcolor{red}{0.28} & 345.40 & 1.31 & 39.10 & \textcolor{blue}{0.23} & 298.40 & 1.21 & 39.80 & 0.00 & \textcolor{blue}{477.66} & \textcolor{blue}{1.53} & 38.06 & \textcolor{red}{0.79} & \textcolor{blue}{319.70} & 1.28 & 39.10 & 0.05 \\
~~ CryptoTrade & 304.80 & 1.24 & 43.90 & 0.02 & \textcolor{blue}{346.70} & 1.30 & 43.60 & \textcolor{red}{0.30} & \textcolor{red}{395.50} & \textcolor{blue}{1.38} & 36.10 & \textcolor{red}{0.55} & 378.35 & 1.36 & 41.74 & 0.45 & \textcolor{red}{334.40} & 1.27 & 43.20 & \textcolor{red}{0.24} \\
~~ FSReason & 240.60 & 1.11 & 40.70 & -0.39 & 284.50 & 1.21 & 38.10 & -0.12 & 309.60 & 1.25 & 36.60 & 0.03 & 295.74 & 1.23 & 38.38 & -0.14 & 244.90 & 1.12 & 38.30 & -0.43 \\
\midrule
\multicolumn{21}{l}{\small\textit{DRL-based Baselines}} \\
~~ A2C & 299.71 & \textcolor{blue}{1.33} & \textcolor{red}{29.88} & -0.13 & 299.71 & \textcolor{blue}{1.33} & \textcolor{red}{29.88} & -0.13 & 299.71 & 1.33 & \textcolor{red}{29.88} & -0.13 & 299.71 & 1.33 & \textcolor{red}{29.88} & -0.13 & 299.71 & \textcolor{red}{1.33} & \textcolor{red}{29.88} & -0.13 \\
~~ PPO & 209.69 & 1.05 & 35.52 & -0.56 & 209.69 & 1.05 & 35.52 & -0.56 & 209.69 & 1.05 & 35.52 & -0.56 & 209.69 & 1.05 & 35.52 & -0.56 & 209.69 & 1.05 & 35.52 & -0.56 \\
~~ DQN & 243.50 & 1.18 & 39.30 & -0.47 & 243.50 & 1.18 & 39.30 & -0.47 & 243.50 & 1.18 & 39.30 & -0.47 & 243.50 & 1.18 & 39.30 & -0.47 & 243.50 & 1.18 & 39.30 & -0.47 \\
~~ MAPPO & 253.42 & 1.14 & 40.14 & -1.17 & 253.42 & 1.14 & 40.14 & -1.17 & 253.42 & 1.14 & 40.14 & -1.17 & 253.42 & 1.14 & 40.14 & -1.17 & 253.42 & 1.14 & 40.14 & -1.17 \\
~~ MADDPG & 263.56 & 1.16 & 39.22 & -0.85 & 263.56 & 1.16 & 39.22 & -0.85 & 263.56 & 1.16 & 39.22 & -0.85 & 263.56 & 1.16 & 39.22 & -0.85 & 263.56 & 1.16 & 39.22 & -0.85 \\
\midrule
\rowcolor{maroon!5} \textbf{MRC (Ours)} & \textcolor{red}{381.20} & \textcolor{red}{1.44} & 34.00 & \textcolor{blue}{0.26} & \textcolor{red}{375.40} & \textcolor{red}{1.42} & 32.40 & \textcolor{blue}{0.23} & \textcolor{blue}{375.50} & \textcolor{red}{1.42} & \textcolor{blue}{31.30} & \textcolor{blue}{0.24} & \textcolor{red}{510.18} & \textcolor{red}{1.60} & 37.71 & \textcolor{blue}{0.74} & 307.10 & \textcolor{blue}{1.31} & 34.20 & -0.10 \\
\bottomrule
\end{tabular}%
}
\end{table}

\begin{table}[h]
\centering
\caption{Burn-in sensitivity analysis. The burn-in period is determined by inverting Eq.~\eqref{eq:alpha}: $t = \lceil{-\lambda\ln(1-\alpha^*)}\rceil$, where $\alpha^* \in [0,1)$ is the target mixing level and $\lambda=30$. Evaluation windows share the same end date; start dates and sizes vary. \textcolor{red}{Red}: best per column; \textcolor{blue}{Blue}: second best.}
\label{tab:burnin_sensitivity}
\vspace{0mm}
\setlength{\tabcolsep}{3.2pt}
\resizebox{\textwidth}{!}{%
\begin{tabular}{l *{5}{cccc}}
\toprule
\multirow{3}{*}{\textbf{Method}} & \multicolumn{4}{c}{$\alpha=0$ (Full Window)} & \multicolumn{4}{c}{$\alpha=0.25$ (Burn-in 29d)} & \multicolumn{4}{c}{$\alpha=0.5$ (Burn-in 69d)} & \multicolumn{4}{c}{$\alpha=0.75$ (139d)} & \multicolumn{4}{c}{$\alpha=0.9$ (230d)} \\
\cmidrule(lr){2-5} \cmidrule(lr){6-9} \cmidrule(lr){10-13} \cmidrule(lr){14-17} \cmidrule(lr){18-21}
 & CR\,\% & SR & MDD\,\% & IR & CR\,\% & SR & MDD\,\% & IR & CR\,\% & SR & MDD\,\% & IR & CR\,\% & SR & MDD\,\% & IR & CR\,\% & SR & MDD\,\% & IR \\

 & ($\uparrow$) & ($\uparrow$) & ($\downarrow$) & ($\uparrow$) & ($\uparrow$) & ($\uparrow$) & ($\downarrow$) & ($\uparrow$) & ($\uparrow$) & ($\uparrow$) & ($\downarrow$) & ($\uparrow$) & ($\uparrow$) & ($\uparrow$) & ($\downarrow$) & ($\uparrow$) & ($\uparrow$) & ($\uparrow$) & ($\downarrow$) & ($\uparrow$) \\
\midrule
\multicolumn{21}{l}{\small\textit{Market Benchmarks}} \\
~~ EW & 279.93 & 1.14 & 42.78 & 0.00 & 265.18 & 1.14 & 42.78 & 0.00 & 287.60 & 1.21 & 42.78 & 0.00 & 227.96 & 1.17 & 42.78 & 0.00 & 272.26 & 1.35 & 42.78 & 0.00 \\
~~ BTC & 270.94 & 1.23 & \textcolor{blue}{32.02} & -0.07 & 209.18 & 1.12 & \textcolor{blue}{32.02} & -0.31 & 216.78 & 1.18 & \textcolor{blue}{32.02} & -0.36 & 190.82 & 1.17 & \textcolor{blue}{32.02} & -0.27 & 207.53 & 1.30 & \textcolor{blue}{32.02} & -0.39 \\
~~ ETH & 78.45 & 0.63 & 63.75 & -0.59 & 65.73 & 0.60 & 63.75 & -0.64 & 60.84 & 0.60 & 63.75 & -0.76 & 55.49 & 0.60 & 63.75 & -0.64 & 85.79 & 0.75 & 63.75 & -0.62 \\
\midrule
\multicolumn{21}{l}{\small\textit{LLM-based Baselines}} \\
~~ FinCon & 266.50 & 1.12 & 43.02 & -1.13 & 251.82 & 1.12 & 43.02 & -1.19 & 270.95 & 1.19 & 43.02 & -1.45 & 214.21 & 1.14 & 43.02 & -1.48 & 257.58 & 1.32 & 43.02 & -1.48 \\
~~ TradingAgents & 286.52 & 1.21 & 41.55 & -0.17 & 272.03 & 1.22 & 41.55 & -0.16 & 293.19 & 1.29 & 41.55 & -0.18 & 249.65 & 1.27 & 41.55 & 0.05 & 291.35 & 1.46 & 41.55 & 0.00 \\
~~ FinMem & \textcolor{blue}{477.66} & \textcolor{blue}{1.53} & 38.06 & \textcolor{red}{0.79} & \textcolor{blue}{424.23} & \textcolor{blue}{1.50} & 38.06 & \textcolor{blue}{0.68} & \textcolor{blue}{452.44} & \textcolor{blue}{1.58} & 38.06 & \textcolor{blue}{0.69} & \textcolor{blue}{375.77} & \textcolor{blue}{1.56} & 38.06 & 0.80 & \textcolor{blue}{428.99} & \textcolor{blue}{1.77} & 38.06 & \textcolor{blue}{0.81} \\
~~ CryptoTrade & 378.35 & 1.36 & 41.74 & 0.45 & 357.02 & 1.36 & 41.74 & 0.45 & 387.77 & 1.45 & 41.74 & 0.48 & 348.17 & 1.47 & 41.74 & \textcolor{blue}{0.82} & 389.83 & 1.65 & 41.74 & 0.75 \\
~~ FSReason & 295.74 & 1.23 & 38.38 & -0.14 & 286.10 & 1.25 & 38.38 & -0.03 & 301.79 & 1.31 & 38.38 & -0.08 & 281.10 & 1.34 & 38.38 & 0.29 & 319.34 & 1.52 & 38.38 & 0.22 \\
\midrule
\multicolumn{21}{l}{\small\textit{DRL-based Baselines}} \\
~~ A2C & 299.71 & 1.33 & \textcolor{red}{29.88} & -0.13 & 288.25 & 1.35 & \textcolor{red}{29.88} & -0.12 & 290.01 & 1.39 & \textcolor{red}{29.88} & -0.20 & 255.66 & 1.39 & \textcolor{red}{29.88} & -0.07 & 273.16 & 1.53 & \textcolor{red}{29.88} & -0.21 \\
~~ PPO & 209.69 & 1.05 & 35.52 & -0.56 & 197.41 & 1.05 & 35.52 & -0.58 & 201.69 & 1.09 & 35.52 & -0.69 & 185.74 & 1.11 & 35.52 & -0.45 & 229.62 & 1.31 & 35.52 & -0.44 \\
~~ DQN & 243.54 & 1.18 & 39.27 & -0.47 & 215.55 & 1.14 & 39.27 & -0.59 & 224.09 & 1.20 & 39.27 & -0.68 & 171.15 & 1.12 & 39.27 & -0.73 & 228.24 & 1.38 & 39.27 & -0.60 \\
~~ MAPPO & 253.42 & 1.14 & 40.14 & -1.17 & 240.81 & 1.14 & 40.14 & -1.16 & 259.84 & 1.21 & 40.14 & -1.24 & 206.65 & 1.16 & 40.14 & -1.23 & 245.17 & 1.34 & 40.14 & -1.40 \\
~~ MADDPG & 263.56 & 1.16 & 39.22 & -0.85 & 250.54 & 1.16 & 39.22 & -0.83 & 270.36 & 1.23 & 39.22 & -0.89 & 215.28 & 1.18 & 39.22 & -0.86 & 253.08 & 1.36 & 39.22 & -1.05 \\
\midrule
\rowcolor{maroon!5} \textbf{MRC (Ours)} & \textcolor{red}{510.18} & \textcolor{red}{1.60} & 37.71 & \textcolor{blue}{0.74} & \textcolor{red}{464.95} & \textcolor{red}{1.58} & 37.71 & \textcolor{red}{0.69} & \textcolor{red}{509.37} & \textcolor{red}{1.69} & 37.71 & \textcolor{red}{0.74} & \textcolor{red}{477.10} & \textcolor{red}{1.75} & 37.71 & \textcolor{red}{1.10} & \textcolor{red}{507.77} & \textcolor{red}{1.91} & 37.71 & \textcolor{red}{1.02} \\
\bottomrule
\end{tabular}%
}
\end{table}

\begin{table}[t]
\centering
\caption{Slippage sensitivity analysis. \textcolor{red}{Red}: best per column; \textcolor{blue}{Blue}: second best. Slippage applied one-way per unit of portfolio turnover.}
\label{tab:slippage_sensitivity}
\vspace{0mm}
\setlength{\tabcolsep}{3.2pt}
\resizebox{\textwidth}{!}{%
\begin{tabular}{l *{1}{cccc} *{1}{cccc} *{1}{cccc} *{1}{cccc}}
\toprule
\multirow{3}{*}{\textbf{Method}} & \multicolumn{4}{c}{0\,bps (no cost)} & \multicolumn{4}{c}{5\,bps} & \multicolumn{4}{c}{10\,bps} & \multicolumn{4}{c}{30\,bps} \\
\cmidrule(lr){2-5} \cmidrule(lr){6-9} \cmidrule(lr){10-13} \cmidrule(lr){14-17}
  & CR\,\% & SR & MDD\,\% & IR & CR\,\% & SR & MDD\,\% & IR & CR\,\% & SR & MDD\,\% & IR & CR\,\% & SR & MDD\,\% & IR \\
  & ($\uparrow$) & ($\uparrow$) & ($\downarrow$) & ($\uparrow$) & ($\uparrow$) & ($\uparrow$) & ($\downarrow$) & ($\uparrow$) & ($\uparrow$) & ($\uparrow$) & ($\downarrow$) & ($\uparrow$) & ($\uparrow$) & ($\uparrow$) & ($\downarrow$) & ($\uparrow$) \\
\midrule
\multicolumn{17}{l}{\small\textit{Market Benchmarks}} \\
~~ EW & 279.93 & 1.14 & 42.78 & 0.00 & 279.93 & 1.14 & 42.78 & 0.00 & \textcolor{blue}{279.93} & 1.14 & 42.78 & \textcolor{red}{0.00} & \textcolor{red}{279.93} & \textcolor{blue}{1.14} & 42.78 & \textcolor{red}{0.00} \\
~~ BTC & 270.94 & 1.23 & \textcolor{blue}{32.02} & -0.07 & 270.94 & 1.23 & \textcolor{red}{32.02} & -0.07 & 270.94 & \textcolor{blue}{1.23} & \textcolor{red}{32.02} & -0.07 & \textcolor{blue}{270.94} & \textcolor{red}{1.23} & \textcolor{red}{32.02} & \textcolor{blue}{-0.07} \\
~~ ETH & 78.45 & 0.63 & 63.75 & -0.59 & 78.45 & 0.63 & 63.75 & -0.59 & 78.45 & 0.63 & 63.75 & -0.59 & 78.45 & 0.63 & 63.75 & -0.59 \\
\midrule
\multicolumn{17}{l}{\small\textit{LLM-based Baselines}} \\
~~ FinCon & 266.50 & 1.12 & 43.00 & -1.13 & 263.90 & 1.12 & 43.10 & -1.20 & 261.30 & 1.11 & 43.10 & -1.27 & 251.20 & 1.10 & 43.30 & -1.52 \\
~~ TradingAgents & 286.50 & 1.21 & 41.50 & -0.17 & 211.60 & 1.06 & 43.10 & -0.91 & 151.20 & 0.90 & 44.50 & -1.65 & 6.00 & 0.29 & 50.10 & -4.60 \\
~~ FinMem & \textcolor{blue}{477.70} & \textcolor{blue}{1.53} & 38.10 & \textcolor{red}{0.79} & \textcolor{blue}{334.50} & \textcolor{blue}{1.32} & 40.30 & \textcolor{blue}{0.13} & 226.70 & 1.11 & 42.40 & -0.54 & 4.40 & 0.27 & 50.30 & -3.17 \\
~~ CryptoTrade & 378.40 & 1.36 & 41.70 & 0.45 & 265.10 & 1.17 & 44.00 & -0.26 & 178.70 & 0.98 & 46.20 & -0.96 & -5.50 & 0.21 & 54.10 & -3.76 \\
~~ FSReason & 295.70 & 1.23 & 38.40 & -0.07 & 198.80 & 1.03 & 40.60 & -0.78 & 125.60 & 0.83 & 42.70 & -1.49 & -26.80 & 0.02 & 50.50 & -4.31 \\
\midrule
\multicolumn{17}{l}{\small\textit{DRL-based Baselines}} \\
~~ A2C & 299.70 & 1.33 & \textcolor{red}{29.90} & -0.13 & 136.20 & 0.91 & \textcolor{blue}{32.80} & -0.92 & 39.50 & 0.49 & \textcolor{blue}{36.50} & -1.71 & -83.10 & -1.20 & 83.20 & -4.86 \\
~~ PPO & 209.70 & 1.05 & 35.50 & -0.56 & 91.50 & 0.71 & 39.00 & -1.54 & 18.40 & 0.37 & 42.20 & -2.52 & -82.80 & -1.01 & 83.40 & -6.44 \\
~~ DQN & 243.50 & 1.18 & 39.30 & -0.47 & 163.40 & 0.97 & 40.90 & -1.05 & 102.00 & 0.77 & 42.50 & -1.63 & -30.30 & -0.05 & 48.50 & -3.90 \\
~~ MAPPO & 253.40 & 1.14 & 40.10 & -1.17 & 245.30 & 1.12 & 40.30 & -1.34 & 237.30 & 1.11 & 40.40 & -1.52 & 207.30 & 1.04 & \textcolor{blue}{41.10} & -2.21 \\
~~ MADDPG & 263.60 & 1.16 & 39.20 & -0.85 & 248.40 & 1.13 & 39.60 & -1.18 & 233.80 & 1.10 & 39.90 & -1.51 & 181.40 & 0.98 & 41.30 & -2.83 \\
\midrule
\rowcolor{maroon!5} \textbf{MRC (Ours)} & \textcolor{red}{510.20} & \textcolor{red}{1.60} & 37.70 & \textcolor{blue}{0.74} & \textcolor{red}{402.10} & \textcolor{red}{1.45} & 39.30 & \textcolor{red}{0.36} & \textcolor{red}{313.10} & \textcolor{red}{1.30} & 40.80 & \textcolor{blue}{-0.02} & 89.20 & 0.71 & 46.50 & -1.53 \\
\bottomrule
\end{tabular}%
}
\end{table}

\begin{table}[h]
\centering
\caption{LLM backbone ablation --- Stage~1: Individual Agents and Full MRC
  (2023-03-01--2025-12-31, 1\,037 days).
  Each specialist agent operates independently without cross-agent communication.
  $A_1$: Market Observer;
  $A_2$: On-chain Observer;
  $A_3$: Macro Observer.
  See Table~\ref{tab:llm_backbone_s2} for Stage~2 and
  Table~\ref{tab:llm_backbone_s3} for Stage~3 results.
  \textcolor{red}{Red}: best per column; \textcolor{blue}{Blue}: second best.}
\label{tab:llm_backbone_s1}
\vspace{0mm}
\setlength{\tabcolsep}{3.5pt}
\resizebox{\textwidth}{!}{%
\begin{tabular}{l *{4}{cccc}}
\toprule
\multirow{3}{*}{\textbf{Backbone}}
  & \multicolumn{12}{c}{\textbf{Stage 1 --- Independent Agent Performance}}
  & \multicolumn{4}{c}{\textbf{Full MRC System}} \\
\cmidrule(lr){2-13} \cmidrule(lr){14-17}
  & \multicolumn{4}{c}{$A_1$ (Market Observer)}
  & \multicolumn{4}{c}{$A_2$ (On-chain Observer)}
  & \multicolumn{4}{c}{$A_3$ (Macro Observer)}
  & \multicolumn{4}{c}{MRC (Full)} \\
\cmidrule(lr){2-5} \cmidrule(lr){6-9} \cmidrule(lr){10-13} \cmidrule(lr){14-17}
  & CR\,\% & SR & MDD\,\% & IR
  & CR\,\% & SR & MDD\,\% & IR
  & CR\,\% & SR & MDD\,\% & IR
  & CR\,\% & SR & MDD\,\% & IR \\
  & ($\uparrow$) & ($\uparrow$) & ($\downarrow$) & ($\uparrow$)
  & ($\uparrow$) & ($\uparrow$) & ($\downarrow$) & ($\uparrow$)
  & ($\uparrow$) & ($\uparrow$) & ($\downarrow$) & ($\uparrow$)
  & ($\uparrow$) & ($\uparrow$) & ($\downarrow$) & ($\uparrow$) \\
\midrule
\multicolumn{17}{l}{\small\textit{Market Benchmarks}} \\
~~ EW & \multicolumn{12}{c}{\textit{n/a}} & 279.9 & 1.14 & 42.8 & 0.00 \\
~~ BTC   & \multicolumn{12}{c}{\textit{n/a}} & 270.9 & 1.23 & \textcolor{blue}{32.0} & $-$0.07 \\
~~ ETH & \multicolumn{12}{c}{\textit{n/a}} & 78.5 & 0.63 & 63.8 & -0.59 \\
\midrule
\multicolumn{17}{l}{\small\textit{LLM Backbone (Qwen3-VL variants)}} \\
~~ \texttt{qwen3-vl-8b-instruct}        & 271.4 & 1.42 & \textcolor{red}{28.8} & $-$0.29 & 249.0 & 1.21 & \textcolor{red}{31.1} & $-$0.55 & \textcolor{red}{195.7} & 1.02 & 40.2 & $-$0.97 & 336.8 & 1.42 & \textcolor{red}{30.3} & $-$0.01 \\
~~ \texttt{qwen3-vl-8b-thinking}        & 256.9 & 1.27 & \textcolor{blue}{41.3} & $-$0.35 & 215.0 & 1.15 & 36.5 & $-$0.66 & 188.2 & 1.02 & 39.9 & $-$0.98 & 299.1 & 1.30 & 36.4 & $-$0.17 \\
~~ \texttt{qwen3-vl-32b-instruct}       & 343.3 & 1.44 & 45.9 & 0.09 & 286.3 & 1.43 & 35.2 & $-$0.15 & 189.6 & \textcolor{blue}{1.03} & 41.6 & \textcolor{blue}{$-$0.77} & 387.4 & 1.45 & 36.9 & 0.25 \\
~~ \texttt{qwen3-vl-30b-a3b-instruct}   & \textcolor{blue}{462.3} & 1.48 & 47.5 & \textcolor{blue}{0.25} & 269.8 & 1.22 & 39.7 & $-$0.73 & \textcolor{blue}{195.2} & 1.01 & 42.1 & $-$1.51 & 418.2 & 1.48 & 39.1 & 0.46 \\
~~ \texttt{qwen3-vl-30b-a3b-thinking}   & \textcolor{red}{589.0} & \textcolor{red}{1.60} & 48.5 & \textcolor{red}{0.66} & 332.5 & 1.36 & 33.9 & \textcolor{blue}{0.00} & 156.9 & 0.96 & \textcolor{red}{33.7} & $-$1.00 & \textcolor{red}{521.6} & \textcolor{red}{1.57} & 38.0 & \textcolor{red}{0.69} \\
~~ \texttt{qwen3-vl-235b-a22b-instruct} & 341.1 & 1.36 & 48.4 & 0.23 & \textcolor{red}{368.5} & \textcolor{blue}{1.49} & \textcolor{blue}{31.1} & \textcolor{red}{0.43} & 189.3 & \textcolor{red}{1.04} & 41.4 & \textcolor{red}{$-$0.53} & \textcolor{blue}{444.3} & \textcolor{blue}{1.52} & 39.6 & \textcolor{blue}{0.51} \\
~~ \texttt{qwen3-vl-235b-a22b-thinking} & 451.0 & \textcolor{blue}{1.52} & 43.9 & 0.22 & \textcolor{blue}{341.7} & \textcolor{red}{1.52} & 33.3 & $-$0.21 & 173.0 & 1.01 & \textcolor{blue}{38.9} & $-$1.18 & 431.9 & 1.49 & 37.5 & 0.44 \\
\bottomrule
\end{tabular}%
}
\end{table}

\begin{table}[h]
\centering
\caption{LLM backbone ablation --- Stage~2: Pairwise Coalitions and Full MRC
  (2023-03-01--2025-12-31, 1\,037 days).
  $P_{12}$: Market$\oplus$On-chain Socratic debate;
  $P_{13}$: Market$\oplus$Macro debate;
  $P_{23}$: On-chain$\oplus$Macro debate.
  Market benchmarks have no pairwise-coalition counterpart (\textit{n/a}).
  See Table~\ref{tab:llm_backbone_s1} for Stage~1 and
  Table~\ref{tab:llm_backbone_s3} for Stage~3 results.
  \textcolor{red}{Red}: best per column; \textcolor{blue}{Blue}: second best.}
\label{tab:llm_backbone_s2}
\vspace{0mm}
\setlength{\tabcolsep}{3.5pt}
\resizebox{\textwidth}{!}{%
\begin{tabular}{l *{4}{cccc}}
\toprule
\multirow{3}{*}{\textbf{Backbone}}
  & \multicolumn{12}{c}{\textbf{Stage 2 --- Pairwise Coalitions}}
  & \multicolumn{4}{c}{\textbf{Full MRC System}} \\
\cmidrule(lr){2-13} \cmidrule(lr){14-17}
  & \multicolumn{4}{c}{$P_{12}$ (Mkt.$\oplus$Chain)}
  & \multicolumn{4}{c}{$P_{13}$ (Mkt.$\oplus$Macro)}
  & \multicolumn{4}{c}{$P_{23}$ (Chain$\oplus$Macro)}
  & \multicolumn{4}{c}{MRC (Full)} \\
\cmidrule(lr){2-5} \cmidrule(lr){6-9} \cmidrule(lr){10-13} \cmidrule(lr){14-17}
  & CR\,\% & SR & MDD\,\% & IR
  & CR\,\% & SR & MDD\,\% & IR
  & CR\,\% & SR & MDD\,\% & IR
  & CR\,\% & SR & MDD\,\% & IR \\
  & ($\uparrow$) & ($\uparrow$) & ($\downarrow$) & ($\uparrow$)
  & ($\uparrow$) & ($\uparrow$) & ($\downarrow$) & ($\uparrow$)
  & ($\uparrow$) & ($\uparrow$) & ($\downarrow$) & ($\uparrow$)
  & ($\uparrow$) & ($\uparrow$) & ($\downarrow$) & ($\uparrow$) \\
\midrule
\multicolumn{17}{l}{\small\textit{Market Benchmarks}} \\
~~ EW & \multicolumn{12}{c}{\textit{n/a}} & 279.9 & 1.14 & 42.8 & 0.00 \\
~~ BTC   & \multicolumn{12}{c}{\textit{n/a}} & 270.9 & 1.23 & \textcolor{blue}{32.0} & $-$0.07 \\
~~ ETH & \multicolumn{12}{c}{\textit{n/a}} & 78.5 & 0.63 & 63.8 & -0.59 \\
\midrule
\multicolumn{17}{l}{\small\textit{LLM Backbone (Qwen3-VL variants)}} \\
~~ \texttt{qwen3-vl-8b-instruct}        & 300.2 & \textcolor{red}{1.62} & \textcolor{red}{20.0} & \textcolor{blue}{$-$0.28} & 218.7 & 1.37 & \textcolor{red}{28.1} & $-$0.48 & 210.8 & 1.16 & \textcolor{blue}{29.5} & $-$0.79 & 336.8 & 1.42 & \textcolor{red}{30.3} & $-$0.01 \\
~~ \texttt{qwen3-vl-8b-thinking}        & 174.5 & 1.16 & 38.9 & $-$0.79 & 172.8 & 1.12 & 38.9 & $-$0.83 & 165.8 & 1.08 & 34.1 & $-$0.93 & 299.1 & 1.30 & 36.4 & $-$0.17 \\
~~ \texttt{qwen3-vl-32b-instruct}       & 231.0 & 1.33 & 42.0 & $-$0.31 & 276.4 & 1.32 & 47.2 & \textcolor{blue}{$-$0.08} & \textcolor{red}{295.8} & \textcolor{blue}{1.44} & 31.4 & \textcolor{blue}{$-$0.11} & 387.4 & 1.45 & 36.9 & 0.25 \\
~~ \texttt{qwen3-vl-30b-a3b-instruct}   & \textcolor{blue}{335.5} & 1.42 & 39.5 & $-$0.33 & \textcolor{red}{363.2} & \textcolor{blue}{1.42} & 39.1 & $-$0.20 & 197.0 & 1.07 & 38.4 & $-$1.29 & 418.2 & 1.48 & 39.1 & 0.46 \\
~~ \texttt{qwen3-vl-30b-a3b-thinking}   & 189.7 & 1.13 & 44.6 & $-$0.48 & 227.5 & 1.25 & \textcolor{blue}{37.4} & $-$0.44 & 180.2 & 1.15 & 31.5 & $-$0.72 & \textcolor{red}{521.6} & \textcolor{red}{1.57} & 38.0 & \textcolor{red}{0.69} \\
~~ \texttt{qwen3-vl-235b-a22b-instruct} & \textcolor{red}{347.2} & \textcolor{blue}{1.57} & 39.3 & \textcolor{red}{0.13} & 272.6 & 1.38 & 42.3 & \textcolor{red}{$-$0.06} & \textcolor{blue}{291.3} & 1.38 & 33.6 & \textcolor{red}{0.03} & \textcolor{blue}{444.3} & \textcolor{blue}{1.52} & 39.6 & \textcolor{blue}{0.51} \\
~~ \texttt{qwen3-vl-235b-a22b-thinking} & 193.2 & 1.30 & \textcolor{blue}{38.0} & $-$0.67 & \textcolor{blue}{285.1} & \textcolor{red}{1.51} & 38.5 & $-$0.34 & 269.3 & \textcolor{red}{1.50} & \textcolor{red}{28.1} & $-$0.47 & 431.9 & 1.49 & 37.5 & 0.44 \\
\bottomrule
\end{tabular}%
}
\end{table}

\begin{table}[h]
\centering
\caption{LLM backbone ablation --- Stage~3: Grand Coalition and Full MRC
  (2023-03-01--2025-12-31, 1\,037 days).
  The Stage~3 grand-coalition readout $\mathbf{w}_{123}$ is produced by $A_4$ from the full coalition context.
  Market benchmarks have no grand-coalition counterpart (\textit{n/a}).
  See Table~\ref{tab:llm_backbone_s1} for Stage~1 and
  Table~\ref{tab:llm_backbone_s2} for Stage~2 results.
  \textcolor{red}{Red}: best per column; \textcolor{blue}{Blue}: second best.}
\label{tab:llm_backbone_s3}
\vspace{0mm}
\setlength{\tabcolsep}{5pt}
\resizebox{0.9\textwidth}{!}{%
\begin{tabular}{l *{2}{cccc}}
\toprule
\multirow{3}{*}{\textbf{Backbone}}
  & \multicolumn{4}{c}{\textbf{Stage 3}}
  & \multicolumn{4}{c}{\textbf{Full MRC System}} \\
\cmidrule(lr){2-5} \cmidrule(lr){6-9}
  & \multicolumn{4}{c}{Grand Coalition Readout ($\mathbf{w}_{123}$)}
  & \multicolumn{4}{c}{MRC (Full)} \\
\cmidrule(lr){2-5} \cmidrule(lr){6-9}
  & CR\,\% & SR & MDD\,\% & IR
  & CR\,\% & SR & MDD\,\% & IR \\
  & ($\uparrow$) & ($\uparrow$) & ($\downarrow$) & ($\uparrow$)
  & ($\uparrow$) & ($\uparrow$) & ($\downarrow$) & ($\uparrow$) \\
\midrule
\multicolumn{9}{l}{\small\textit{Market Benchmarks}} \\
~~ EW & \multicolumn{4}{c}{\textit{n/a}} & 279.9 & 1.14 & 42.8 & 0.00 \\
~~ BTC   & \multicolumn{4}{c}{\textit{n/a}} & 270.9 & 1.23 & \textcolor{blue}{32.0} & $-$0.07 \\
~~ ETH & \multicolumn{4}{c}{\textit{n/a}} & 78.5 & 0.63 & 63.8 & -0.59 \\
\midrule
\multicolumn{9}{l}{\small\textit{LLM Backbone (Qwen3-VL variants)}} \\
~~ \texttt{qwen3-vl-8b-instruct}        & 201.1 & 1.18 & \textcolor{red}{31.3} & $-$0.71 & 336.8 & 1.42 & \textcolor{red}{30.3} & $-$0.01 \\
~~ \texttt{qwen3-vl-8b-thinking}        & 201.2 & 1.12 & 40.2 & $-$0.85 & 299.1 & 1.30 & 36.4 & $-$0.17 \\
~~ \texttt{qwen3-vl-32b-instruct}       & 248.9 & 1.30 & 39.4 & $-$0.28 & 387.4 & 1.45 & 36.9 & 0.25 \\
~~ \texttt{qwen3-vl-30b-a3b-instruct}   & \textcolor{blue}{349.6} & 1.34 & 44.1 & $-$0.23 & 418.2 & 1.48 & 39.1 & 0.46 \\
~~ \texttt{qwen3-vl-30b-a3b-thinking}   & \textcolor{red}{384.8} & \textcolor{red}{1.43} & 43.0 & \textcolor{red}{0.22} & \textcolor{red}{521.6} & \textcolor{red}{1.57} & 38.0 & \textcolor{red}{0.69} \\
~~ \texttt{qwen3-vl-235b-a22b-instruct} & 299.1 & 1.34 & 40.2 & \textcolor{blue}{0.11} & \textcolor{blue}{444.3} & \textcolor{blue}{1.52} & 39.6 & \textcolor{blue}{0.51} \\
~~ \texttt{qwen3-vl-235b-a22b-thinking} & 303.8 & \textcolor{blue}{1.38} & \textcolor{blue}{38.4} & $-$0.33 & 431.9 & 1.49 & 37.5 & 0.44 \\
\bottomrule
\end{tabular}%
}
\end{table}

\begin{table}[h]
\centering
\caption{LLM backbone ablation --- Inference cost and parse failure rate
  (2023-03-01--2025-12-31, 1\,037 days).
  \textbf{Parse failure} occurs when an LLM response cannot be decoded into
  a valid portfolio weight vector; the system falls back to the previous period's
  weights for that step.
  Rates are computed over all 1\,037 rebalancing steps.
  The end-to-end Full MRC failure rate (all three stages fail simultaneously)
  is 0.10\% uniformly across all backbones.
  \textbf{Inference time} reports mean\,$\pm$\,std.\ per individual LLM call:
  Stage~1 = one specialist-agent call; Stage~2 = one pairwise-debate turn;
  Stage~3 = grand-coalition synthesis call;
  Full MRC = total wall-clock time per rebalancing period.
  \textcolor{red}{Red}: best (lowest) per column;
  \textcolor{blue}{Blue}: second best.
  $\downarrow$ lower is better.
  See Tables~\ref{tab:llm_backbone_s1}--\ref{tab:llm_backbone_s3}
  for portfolio performance metrics.}
\label{tab:llm_backbone_cost}
\vspace{0mm}
\setlength{\tabcolsep}{6pt}
\resizebox{\textwidth}{!}{%
\begin{tabular}{l c c c c c}
\toprule
\multirow{2}{*}{\textbf{Backbone}}
  & \multirow{2}{*}{\shortstack[c]{\textbf{Parse Fail.}\\\textbf{Rate (\%)} $(\downarrow)$}}
  & \multicolumn{4}{c}{\textbf{Mean Inference Time per LLM Call (s)} $(\downarrow)$} \\
\cmidrule(lr){3-6}
  & & Stage~1 (Agent) & Stage~2 (Debate) & Stage~3 (GC) & Full MRC \\
\midrule
~~ \texttt{qwen3-vl-8b-instruct}        & \textcolor{blue}{0.39} & $24.5_{\pm 8.2}$ & $12.0_{\pm 4.0}$ & \textcolor{blue}{$5.8_{\pm 2.4}$} & $42.5_{\pm 10.4}$ \\
~~ \texttt{qwen3-vl-8b-thinking}        & \textcolor{red}{0.19} & \textcolor{red}{$14.6_{\pm 2.3}$} & \textcolor{red}{$7.1_{\pm 1.8}$} & \textcolor{red}{$3.5_{\pm 1.8}$} & \textcolor{red}{$25.5_{\pm 3.6}$} \\
~~ \texttt{qwen3-vl-32b-instruct}       & 2.80 & \textcolor{blue}{$22.0_{\pm 2.1}$} & \textcolor{blue}{$11.7_{\pm 6.2}$} & $6.1_{\pm 2.8}$ & \textcolor{blue}{$40.1_{\pm 7.2}$} \\
~~ \texttt{qwen3-vl-30b-a3b-instruct}   & 0.96 & $26.9_{\pm 8.7}$ & $12.9_{\pm 4.5}$ & $6.2_{\pm 3.9}$ & $46.3_{\pm 11.3}$ \\
~~ \texttt{qwen3-vl-30b-a3b-thinking}   & 0.68 & $135.5_{\pm 38.1}$ & $258.6_{\pm 105.4}$ & $63.8_{\pm 42.4}$ & $458.1_{\pm 130.4}$ \\
~~ \texttt{qwen3-vl-235b-a22b-instruct} & 4.53 & $42.2_{\pm 18.9}$ & $22.2_{\pm 11.2}$ & $10.5_{\pm 7.2}$ & $75.1_{\pm 26.9}$ \\
~~ \texttt{qwen3-vl-235b-a22b-thinking} & 3.47 & $42.2_{\pm 15.6}$ & $21.1_{\pm 8.1}$ & $9.6_{\pm 6.2}$ & $73.1_{\pm 19.8}$ \\
\bottomrule
\end{tabular}%
}
\end{table}

\clearpage
\section{Source of the Gain and Robustness}
\label{app:gain}

Tables~\ref{tab:source_of_gain}, \ref{tab:constraints} and~\ref{tab:sensitivity} come from an offline replay of the five logged runs. Each arm reuses the logged portfolios of the seven coalitions and changes only the parts named in its row. The grand-coalition portfolio enters every arm as logged, after its blend with the Shapley-weighted ensemble (Eq.~\eqref{eq:a4_blend}), at a fixed weight of 0.10, which keeps the LLM readout identical across arms. MRC scores SR 1.512 in the replay against 1.51 in Table~\ref{tab:main_results}. Differences are paired by seed and computed before rounding. The Seeds columns give the number of the five seeds that move in the stated direction, and $\pm$ marks the standard deviation over seeds.

\paragraph{Aggregation Rules and Leave-One-Out.}
Table~\ref{tab:source_of_gain} separates the credit mechanism (Shapley credit, the Bayesian mixture, the Regime-Aware Multiplier and Selective WTA) from the risk-control overlays for RQ2 in Section~\ref{sec:experiments}. The left side changes only the aggregation rule over the same coalition portfolios, with the overlays on or off. Its worst window is the lowest SR over 190 rolling 252-day windows (38 per seed), and its 10\,bps cost also counts changes in the cash weight. The right side removes one part of the credit mechanism or of the risk-control overlays at a time and gives the change in SR at 0\,bps. With the overlays on, the credit mechanism reaches SR 1.512 against 1.374 for uniform weighting and beats it in all five seeds ($+0.137$, first row of Table~\ref{tab:constraints}). It is ahead of the three alternative rules with and without the overlays, at a 10\,bps trading cost and on the worst rolling 252-day window. Removing any part listed on the right lowers SR, the cash target most of all.

\begin{table}[tbh]
\centering
\caption{Source of the gain (offline replay, 5 seeds). Left: aggregation rules on the same coalition portfolios. \textcolor{red}{Red}: best per column; \textcolor{blue}{Blue}: second best. Right: SR change when one part is removed from MRC.}
\label{tab:source_of_gain}
\footnotesize
\setlength{\tabcolsep}{4pt}
\resizebox{0.8\textwidth}{!}{%
\begin{tabular}[t]{lcccc}
\toprule
& \multicolumn{2}{c}{SR, overlays} & SR at & Worst \\
\cmidrule(lr){2-3}
\textbf{Aggregation rule} & on & off & 10\,bps & window \\
\midrule
Uniform $1/N$ & 1.374 & 1.258 & 1.156 & \textcolor{blue}{0.574} \\
Inverse volatility & 1.393 & 1.276 & \textcolor{blue}{1.170} & 0.572 \\
Hedge~\cite{freund1997decision} & \textcolor{blue}{1.455} & \textcolor{blue}{1.306} & 1.168 & 0.412 \\
\midrule
\rowcolor{maroon!5} \textbf{MRC (Ours)} & \textcolor{red}{1.512} & \textcolor{red}{1.321} & \textcolor{red}{1.235} & \textcolor{red}{0.608} \\
\bottomrule
\end{tabular}\hspace{1.5em}%
\begin{tabular}[t]{lcc}
\toprule
& \multicolumn{2}{c}{Change at 0\,bps} \\
\cmidrule(lr){2-3}
\textbf{Removed from MRC} & $\Delta$SR & Seeds \\
\midrule
\multicolumn{3}{l}{\textit{Credit mechanism}} \\
~~Shapley credit in the mixture & $-0.057$ & 4/5 \\
~~Selective WTA & $-0.041$ & 5/5 \\
~~Regime-Aware Multiplier & $-0.022$ & 4/5 \\
\midrule
\multicolumn{3}{l}{\textit{Risk-control overlays}} \\
~~Cash target & $-0.157$ & 5/5 \\
~~Momentum Overlay & $-0.063$ & 5/5 \\
~~Transition buffer & $-0.005$ & 5/5 \\
\bottomrule
\end{tabular}%
}
\end{table}

\paragraph{Constraint Sensitivity.}
Table~\ref{tab:constraints} adds caps after the overlays. A final projection caps each asset weight at $w_{\max}$ and the cash weight at $c_{\max}$, and weight above a cap moves to the other positions in proportion to their weights. At zero cost the credit mechanism beats uniform weighting under every cap in all five seeds, by $+0.101$ to $+0.137$. At 10\,bps its lead is $+0.047$ to $+0.088$ and holds in at least four of five seeds.

\begin{table}[tbh]
\centering
\caption{Constraint sensitivity analysis (offline replay, 5 seeds, mean\,$\pm$\,std). $\Delta$SR: MRC minus uniform $1/N$ under the same caps.}
\label{tab:constraints}
\small
\setlength{\tabcolsep}{2pt}
\begin{tabular}{lcccccc}
\toprule
& \multicolumn{2}{c}{SR ($\uparrow$)} & \multicolumn{2}{c}{$\Delta$SR at 0\,bps} & \multicolumn{2}{c}{$\Delta$SR at 10\,bps} \\
\cmidrule(lr){2-3}\cmidrule(lr){4-5}\cmidrule(lr){6-7}
\textbf{Feasible set} & MRC & Uniform & Mean & Seeds & Mean & Seeds \\
\midrule
As run & 1.512 & 1.374 & $+0.137_{\pm 0.029}$ & 5/5 & $+0.079$ & 5/5 \\
\addlinespace
$w_{\max}=0.25$ & 1.510 & 1.377 & $+0.133_{\pm 0.028}$ & 5/5 & $+0.081$ & 5/5 \\
$w_{\max}=0.20$ & 1.514 & 1.377 & $+0.137_{\pm 0.027}$ & 5/5 & $+0.088$ & 5/5 \\
$w_{\max}=0.15$ & 1.479 & 1.374 & $+0.106_{\pm 0.048}$ & 5/5 & $+0.064$ & 4/5 \\
\addlinespace
$c_{\max}=0.20$ & 1.468 & 1.362 & $+0.106_{\pm 0.027}$ & 5/5 & $+0.049$ & 4/5 \\
$c_{\max}=0.10$ & 1.436 & 1.330 & $+0.106_{\pm 0.026}$ & 5/5 & $+0.047$ & 4/5 \\
\addlinespace
$w_{\max}=0.20$, $c_{\max}=0.10$ & 1.433 & 1.332 & $+0.101_{\pm 0.016}$ & 5/5 & $+0.058$ & 5/5 \\
\bottomrule
\end{tabular}
\end{table}

\paragraph{Sensitivity to the Constants.}
Table~\ref{tab:sensitivity} perturbs 25 scalar constants of the credit mechanism and blending, the smoothing and overlays, and the regime definition. One at a time, each constant is multiplied by 0.7, 0.8, 1.2 and 1.3. The table gives the lowest and highest SR over these four settings (mean over seeds) and the worst change against the reported configuration (SR 1.512), paired by seed. No single constant costs more than 0.031 in SR. Jointly, all 25 constants are multiplied by independent draws from $U(0.8, 1.2)$ (200 draws, last row of the table). SR then ranges from 1.465 to 1.589, above Equal-Weight, BTC and CryptoTrade (Table~\ref{tab:main_results}) in every draw. The response is broad rather than peaked at the reported values, which is consistent with constants that were not tuned on the evaluation window.

\begin{table}[tbh]
\centering
\caption{Sensitivity analysis of 25 scalar constants (offline replay, 5 seeds).}
\label{tab:sensitivity}
\small
\setlength{\tabcolsep}{4pt}
\begin{tabular}{lcccc}
\toprule
& \multicolumn{2}{c}{SR range} & \multicolumn{2}{c}{Worst $\Delta$SR} \\
\cmidrule(lr){2-3}\cmidrule(lr){4-5}
\textbf{Constant} & Min & Max & Mean & Seeds \\
\midrule
\multicolumn{5}{l}{\textit{Credit mechanism and blending}} \\
~~WTA rolling window $n_{\mathrm{win}}$ & 1.481 & 1.552 & $-0.031$ & 5/5 \\
~~WTA dominance threshold $\theta_{\mathrm{WTA}}$ & 1.488 & 1.527 & $-0.024$ & 5/5 \\
~~WTA weight $\omega_{\mathrm{wta}}$ & 1.491 & 1.541 & $-0.021$ & 5/5 \\
~~Characteristic-function weight $\gamma_\rho$ & 1.507 & 1.515 & $-0.005$ & 5/5 \\
~~Regime multiplier spread (scale on $\psi_i-1$) & 1.508 & 1.518 & $-0.004$ & 4/5 \\
~~EWP e-fold decay $h$ & 1.510 & 1.518 & $-0.002$ & 3/5 \\
~~Bayesian concentration $\lambda$ & 1.511 & 1.512 & $-0.001$ & 3/5 \\
~~Stage 1/2 blend bandwidth $\tau_{S1}$ & 1.512 & 1.512 & $0.000$ & 4/5 \\
~~Grand-coalition weight in the replay & 1.512 & 1.512 & $0.000$ & 3/5 \\
\midrule
\multicolumn{5}{l}{\textit{Smoothing and overlays}} \\
~~EMA smoothing rate & 1.486 & 1.524 & $-0.026$ & 5/5 \\
~~Momentum Overlay slope & 1.502 & 1.522 & $-0.010$ & 5/5 \\
~~BTC floor & 1.504 & 1.528 & $-0.008$ & 3/5 \\
~~Cash target floor & 1.505 & 1.520 & $-0.007$ & 5/5 \\
~~Momentum Overlay base & 1.505 & 1.519 & $-0.007$ & 5/5 \\
~~Per-asset cap in the Dominance Overlay & 1.507 & 1.512 & $-0.005$ & 4/5 \\
~~Transition buffer depth & 1.510 & 1.514 & $-0.002$ & 5/5 \\
~~Drawdown bandwidth $\tau_{\mathrm{dd}}$ & 1.510 & 1.514 & $-0.002$ & 5/5 \\
~~Drawdown reduction $\nu_{\mathrm{dd}}$ & 1.511 & 1.513 & $-0.001$ & 4/5 \\
~~Cash target bandwidth & 1.511 & 1.513 & $-0.001$ & 5/5 \\
~~Transition buffer bandwidth & 1.511 & 1.514 & $-0.001$ & 5/5 \\
~~Cash target amplitude & 1.511 & 1.513 & $0.000$ & 4/5 \\
~~Dominance coefficient in volatile regimes & 1.512 & 1.512 & $0.000$ & 3/5 \\
~~Dominance bandwidth $\tau_d$ & 1.512 & 1.513 & $0.000$ & 3/5 \\
\midrule
\multicolumn{5}{l}{\textit{Regime definition}} \\
~~Bear threshold $|\xi_-|$ & 1.507 & 1.509 & $-0.005$ & 4/5 \\
~~Bull threshold $\xi_+$ & 1.519 & 1.540 & $+0.007$ & 4/5 \\
\midrule
\multicolumn{5}{l}{\textit{All 25 constants jointly}} \\
~~200 draws from $U(0.8, 1.2)$ & 1.465 & 1.589 & & \\
\bottomrule
\end{tabular}
\end{table}

\paragraph{Rule Council.}
To test whether the credit mechanism needs LLM specialists, we replace the three LLM agents by three rule specialists that read one representative signal per data family: cross-sectional momentum for $A_1$, on-chain activity growth for $A_2$, and a risk-appetite score from VIX and the Fear \& Greed Index for $A_3$. 

They emit weight vectors of the same form, fill the same seven coalitions and pass through the same credit mechanism and overlays. Every rule constant is fitted on the 789 days that end on 2023-02-28, before the evaluation window, and then frozen, and the softmax temperatures are set so that each rule specialist's median portfolio breadth matches its LLM counterpart. Switching from uniform weighting to the full credit mechanism is worth $+0.065$ in SR on the rule specialists and $+0.137$ on the LLM specialists, so the mechanism pays on both sides and more on the LLM side.

\paragraph{Regimes and Agent Weights over Five Seeds.}
Table~\ref{tab:regime_five} repeats the regime decomposition of Table~\ref{tab:regime_conditional} for all 14 methods of Table~\ref{tab:main_results}, each over its five runs. The regime labels are the same in every run. Mean rank is the mean SR rank over the three regimes among the 14 methods. Over the same five seeds, the mean regime-adjusted weights $\tilde{\omega}_i^{(t)}$ of $A_1$, $A_2$ and $A_3$ are $0.34{\pm}0.10$, $0.24{\pm}0.09$ and $0.42{\pm}0.16$ in the first 100 periods and $0.50{\pm}0.07$, $0.35{\pm}0.07$ and $0.15{\pm}0.05$ after day 100.

\begin{table}[tbh]
\centering
\caption{Regime-conditional SR over five seeds. Bull / Bear / Volatile identified by $\xi^{(t)}$. Non-benchmark methods: mean\,$\pm$\,std. Benchmarks: deterministic. \textcolor{red}{Red}: best mean per column among non-benchmark methods; \textcolor{blue}{Blue}: second best.}
\label{tab:regime_five}
\small
\setlength{\tabcolsep}{4pt}
\begin{tabular}{lcccc}
\toprule
& \multicolumn{3}{c}{SR ($\uparrow$)} & Mean rank \\
\cmidrule(lr){2-4}
\textbf{Method} & Bull & Bear & Volatile & ($\downarrow$) \\
\midrule
\multicolumn{5}{l}{\textit{Market Benchmarks}} \\
~~EW & 2.26 & 0.50 & 0.04 & 7.33 \\
~~BTC & 1.77 & 0.57 & 1.11 & 5.67 \\
~~ETH & 2.08 & $-$0.79 & $-$0.19 & 13.33 \\
\midrule
\multicolumn{5}{l}{\textit{DRL-based Baselines}} \\
~~A2C & 2.13$_{\pm 0.30}$ & 0.23$_{\pm 0.49}$ & 0.06$_{\pm 0.11}$ & 11.00 \\
~~PPO & 2.08$_{\pm 0.21}$ & 0.29$_{\pm 0.14}$ & 0.17$_{\pm 0.24}$ & 10.67 \\
~~DQN & 2.31$_{\pm 0.28}$ & 0.30$_{\pm 0.43}$ & 0.32$_{\pm 0.21}$ & 7.33 \\
~~MAPPO & 2.26$_{\pm 0.01}$ & \textcolor{blue}{0.51}$_{\pm 0.01}$ & 0.03$_{\pm 0.01}$ & 8.00 \\
~~MADDPG & 2.25$_{\pm 0.04}$ & 0.50$_{\pm 0.05}$ & 0.03$_{\pm 0.03}$ & 9.33 \\
\midrule
\multicolumn{5}{l}{\textit{LLM-based Baselines}} \\
~~FinCon & 2.27$_{\pm 0.04}$ & 0.46$_{\pm 0.03}$ & 0.03$_{\pm 0.02}$ & 8.00 \\
~~TradingAgents & 2.34$_{\pm 0.09}$ & 0.40$_{\pm 0.18}$ & 0.24$_{\pm 0.23}$ & 6.00 \\
~~FinMem & 2.38$_{\pm 0.26}$ & 0.36$_{\pm 0.26}$ & 0.40$_{\pm 0.64}$ & 5.67 \\
~~CryptoTrade & \textcolor{blue}{2.46}$_{\pm 0.09}$ & 0.37$_{\pm 0.25}$ & \textcolor{blue}{0.44}$_{\pm 0.16}$ & \textcolor{blue}{4.33} \\
~~FSReason & 2.42$_{\pm 0.11}$ & 0.32$_{\pm 0.13}$ & 0.12$_{\pm 0.12}$ & 7.00 \\
\midrule
\rowcolor{maroon!5} \textbf{MRC (Ours)} & \textcolor{red}{2.55}$_{\pm 0.05}$ & \textcolor{red}{0.63}$_{\pm 0.27}$ & \textcolor{red}{0.45}$_{\pm 0.22}$ & \textcolor{red}{1.33} \\
\bottomrule
\end{tabular}
\end{table}

\paragraph{Turnover and Cost.}
Table~\ref{tab:turnover} reports turnover and costed SR for the same 14 methods. Turnover is the yearly sum of $|\Delta w|$ over the 13 risky weights. Cost is charged one way per unit of turnover, and 10\,bps is the reference rate. Unlike Table~\ref{tab:slippage_sensitivity}, which shows a single run, this table averages the five logged runs and charges Equal-Weight for its daily rebalancing. Break-even is the cost at which SR falls below Equal-Weight charged the same rate. A break-even of 0 means that SR is not above Equal-Weight even at zero cost.

\begin{table}[tbh]
\centering
\caption{Turnover and transaction cost. Non-benchmark methods: mean\,$\pm$\,std (5 seeds). Benchmarks: deterministic. \textcolor{red}{Red}: best SR and break-even among non-benchmark methods; \textcolor{blue}{Blue}: second best.}
\label{tab:turnover}
\small
\setlength{\tabcolsep}{4pt}
\begin{tabular}{lcccc}
\toprule
& Turnover & \multicolumn{2}{c}{SR ($\uparrow$)} & Break-even \\
\cmidrule(lr){3-4}
\textbf{Method} & per year & 0\,bps & 10\,bps & bps ($\uparrow$) \\
\midrule
\multicolumn{5}{l}{\textit{Market Benchmarks}} \\
~~EW (charged) & 6.2 & 1.14 & 1.13 & \textit{n/a} \\
~~BTC & 0.0 & 1.23 & 1.23 & \textit{n/a} \\
~~ETH & 0.0 & 0.63 & 0.63 & \textit{n/a} \\
\midrule
\multicolumn{5}{l}{\textit{DRL-based Baselines}} \\
~~A2C & 386.3$_{\pm 10.4}$ & 1.03$_{\pm 0.18}$ & 0.22$_{\pm 0.17}$ & 0.0 \\
~~PPO & 344.4$_{\pm 3.9}$ & 1.04$_{\pm 0.13}$ & 0.36$_{\pm 0.13}$ & 0.0 \\
~~DQN & 210.5$_{\pm 28.4}$ & 1.16$_{\pm 0.18}$ & 0.71$_{\pm 0.19}$ & 0.3 \\
~~MAPPO & 17.0$_{\pm 0.5}$ & 1.14$_{\pm 0.01}$ & \textcolor{blue}{1.11}$_{\pm 0.01}$ & 0.0 \\
~~MADDPG & 28.5$_{\pm 2.1}$ & 1.14$_{\pm 0.03}$ & 1.08$_{\pm 0.03}$ & 0.0 \\
\midrule
\multicolumn{5}{l}{\textit{LLM-based Baselines}} \\
~~FinCon & 14.3$_{\pm 2.4}$ & 1.14$_{\pm 0.01}$ & \textcolor{blue}{1.11}$_{\pm 0.01}$ & 0.0 \\
~~TradingAgents & 147.2$_{\pm 4.7}$ & 1.19$_{\pm 0.03}$ & 0.89$_{\pm 0.02}$ & 1.6 \\
~~FinMem & 192.8$_{\pm 7.5}$ & 1.25$_{\pm 0.20}$ & 0.87$_{\pm 0.19}$ & 3.0 \\
~~CryptoTrade & 181.4$_{\pm 12.2}$ & \textcolor{blue}{1.31}$_{\pm 0.06}$ & 0.95$_{\pm 0.05}$ & \textcolor{blue}{4.8} \\
~~FSReason & 188.6$_{\pm 12.7}$ & 1.18$_{\pm 0.06}$ & 0.80$_{\pm 0.06}$ & 1.1 \\
\midrule
\rowcolor{maroon!5} \textbf{MRC (Ours)} & 128.7$_{\pm 5.5}$ & \textcolor{red}{1.51}$_{\pm 0.07}$ & \textcolor{red}{1.23}$_{\pm 0.06}$ & \textcolor{red}{13.8} \\
\bottomrule
\end{tabular}
\end{table}

\paragraph{Scalability.}
Table~\ref{tab:scaling} separates the two costs of exact credit. Exact Shapley values need all $2^N{-}1$ coalition portfolios per period, and the LLM cost per period grows in proportion to this count. The Shapley arithmetic itself, measured as wall-clock time per period, stays below 0.05\,s up to $N{=}12$. Sampling error is the largest difference between sampled and exact $\phi_i$, relative to $\max_i|\phi_i|$, when permutation sampling~\cite{castro2009polynomial} with 500 permutations replaces exact values. It is measured on 20 games per $N$ whose coalition values follow our logged runs and stays below 0.10 up to $N{=}12$. Exact values are therefore practical up to about $N{=}5$ (31 portfolios per period).

\begin{table}[tbh]
\centering
\caption{Cost and error of credit assignment as the number of agents $N$ grows.}
\label{tab:scaling}
\small
\setlength{\tabcolsep}{4pt}
\begin{tabular}{lccc}
\toprule
& \multicolumn{2}{c}{Exact Shapley values} & Sampling \\
\cmidrule(lr){2-3}\cmidrule(lr){4-4}
$N$ & Coalitions ($2^N{-}1$) & Time per period (s) & error ($\downarrow$) \\
\midrule
\rowcolor{maroon!5} 3 (MRC) & 7 & $<$0.001 & 0.04 \\
4 & 15 & $<$0.001 & 0.06 \\
5 & 31 & $<$0.001 & 0.07 \\
6 & 63 & $<$0.001 & 0.07 \\
8 & 255 & 0.001 & 0.08 \\
10 & 1,023 & 0.008 & 0.09 \\
12 & 4,095 & 0.047 & 0.09 \\
\bottomrule
\end{tabular}
\end{table}

\clearpage  
\section{Discussion}
\label{sec:discussion}

\paragraph{Why Debate Underperforms.}
The most striking finding from the coalition ablation is that Stage~2 Socratic debate consistently reduces Sharpe ratio relative to the best individual specialist, a result that stands in apparent tension with evidence that structured LLM debate improves reasoning quality~\cite{du2024improving}.
The resolution lies in a distinction between reasoning tasks and portfolio construction tasks.
In open-ended reasoning, debate surfaces inconsistencies between agents' reasoning over a shared evidence base.
In portfolio construction, the three specialist agents process fundamentally different data modalities across separate information channels, and when their views diverge, the divergence is often epistemically valid rather than a correctable error.
Forcing pairwise reconciliation attenuates signal rather than eliminating noise.
Within MRC, Stage~2 serves a different purpose: it generates structured disagreement signals that modulate $\beta_{\mathrm{gc}}$ via the divergence discount (Eq.~\ref{eq:divergence_discount}), reducing reliance on the grand-coalition readout precisely when inter-agent consensus is low.
The full system recovers the efficient frontier not by improving debate quality, but by Shapley-weighting all coalition outputs so that empirically inferior coalitions contribute proportionally less to the final allocation.

\paragraph{Mathematical Integration versus LLM synthesis.}
A consistent pattern emerges across all ablation configurations: Shapley-weighted mathematical integration outperforms LLM synthesis at every stage.
The dynamic blend ratios $\beta_{S1}^{(t)}$ and $\beta_{\mathrm{gc}}^{(t)}$ reflect this empirically, with $\beta_{\mathrm{gc}}^{(t)}$ remaining near its lower bound throughout most of the evaluation window.
This finding has a broader implication for multi-agent LLM systems in financial applications: the primary value of LLM agents may lie not in their capacity to aggregate information better than quantitative rules, but in their ability to translate heterogeneous raw inputs (candlestick images, on-chain tables, macroeconomic time series) into a common representational format amenable to principled mathematical integration.
The Shapley system provides the bridge between these two layers: it quantifies each modality's marginal contribution and weights the final allocation accordingly, separating the question of which signals matter from the question of how to combine them.

\paragraph{Regime-conditional Performance Pattern.}
Over five seeds (Table~\ref{tab:regime_five}), MRC keeps a positive annualised Sharpe in all three market conditions (bull, volatile, and bear) and ranks first in bull and bear markets among the 14 methods, with the advantage over passive benchmarks most pronounced in bull markets.

\paragraph{Convergence and Regret Analysis.}
The Exponentially Weighted Characteristic Value (EW-CV) characteristic function $v(S;\mathcal{H}_t)$ with e-fold decay period $h{=}252$ (discrete decay ratio $e^{-1/h}$, distinct from the Bayesian concentration parameter $\lambda=30$ in Eq.~\eqref{eq:alpha}) connects MRC to discounted online learning~\cite{zinkevich2003online}, bounding the effective memory to approximately $h$ recent trading periods regardless of the evaluation horizon $T$.

The Bayesian adaptive mixture $\alpha(t) = 1-\exp(-t/\lambda)$ with $\lambda{=}30$ converges from the uniform prior to the Shapley-dominated distribution at an exponential rate, yielding a warm-up period of $\lceil\lambda\ln(2/\varepsilon)\rceil$ periods to reach $\varepsilon$-proximity.
For $\lambda{=}30$ and $\varepsilon{=}0.05$, this is approximately 110 trading days, close to the 100-day cold-start window used in RQ4.
This connection motivates the use of exponential discounting and the Bayesian warm-up schedule in MRC.
We do not claim a formal regret bound for the full deployed system, since regime multipliers, blend rules, and overlay layers introduce additional nonlinearities beyond the standard expert-aggregation setting.

\paragraph{Cold-start Correctness and EWP behavior.}
Over five seeds (Appendix~\ref{app:gain}) no agent dominates the first 100 periods (mean regime-adjusted weights $\tilde{\omega}_i^{(t)}$ of 0.34, 0.24 and 0.42 for $A_1$, $A_2$ and $A_3$), and the weights then diverge as Shapley evidence accumulates, averaging 0.50, 0.35 and 0.15 after day 100. This validates the joint design of the EWP decay and the Bayesian adaptive mixture.
The decline of $A_3$ (Macro Agent) to a mean weight of 0.15 is consistent with theoretical prediction: in a risk-on, momentum-driven environment, macroeconomic indicators have genuinely low marginal contribution to portfolio returns, and the EWP correctly discounts early periods in which macro signals appeared more relevant.
This weight is expected to recover in a sustained macro-driven bear regime, and the geometric discounting ensures the recovery is continuous rather than abrupt.
The three-phase learning schedule operates without any external reset mechanism, a property practically important in live deployment, where retraining is infeasible.

\section{Individual Token, Strategy Return Profiles, and Additional Details of Agent Settings}
\label{app:return_profiles}

Figures~\ref{fig:token_returns} and~\ref{fig:strategy_returns} provide full per-asset and per-strategy visualisations over the entire backtest window (2023-03-01 to 2025-12-31, $T=1{,}037$ trading days).

\begin{figure}[htbp]
    \centering
    \includegraphics[width=0.7\textwidth]{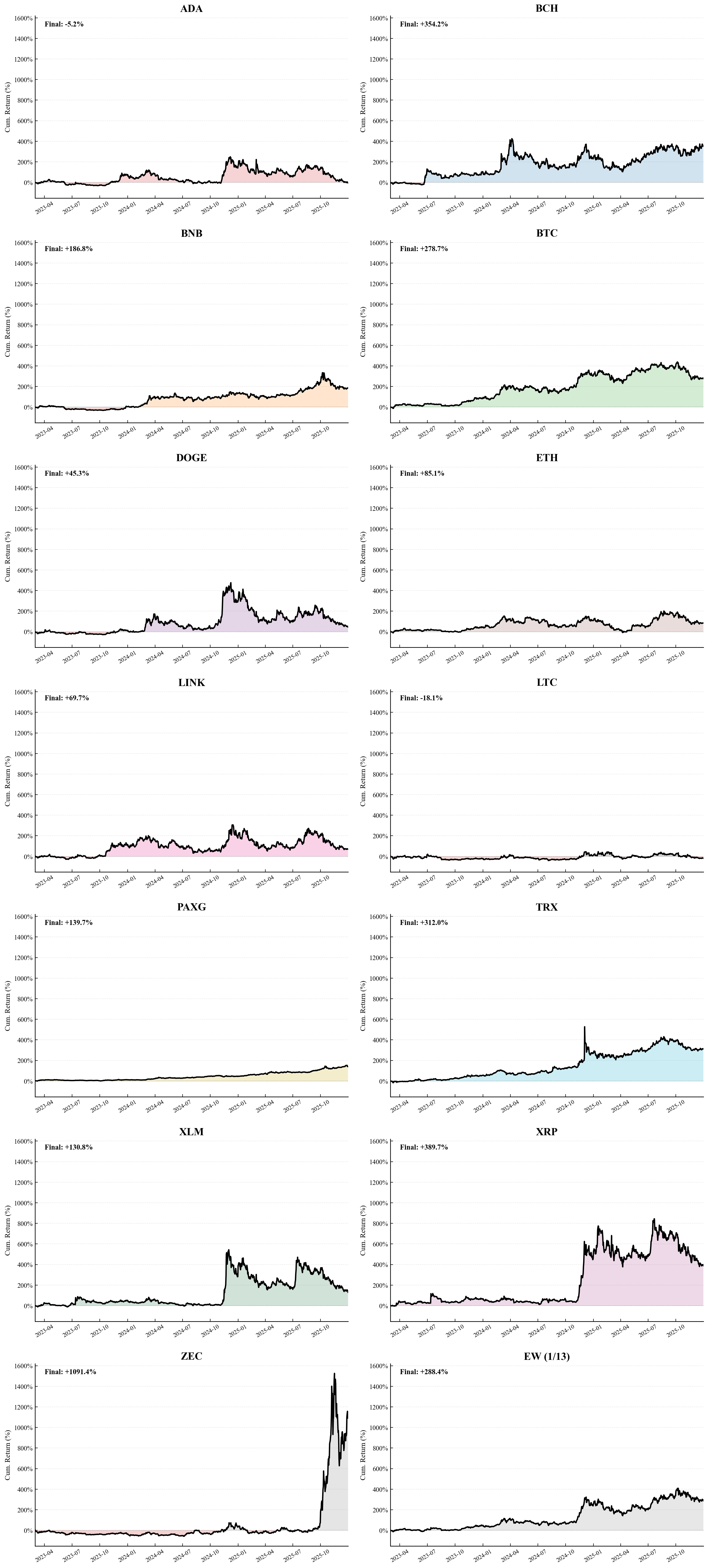}
    \caption{%
        \textbf{Cumulative Hold returns of the 13 tokens and EW benchmark
        (2023-03-01 to 2025-12-31).}
        Each panel: single-asset cumulative return; final value annotated top-left.
    }
    \label{fig:token_returns}
\end{figure}

\begin{figure}[htbp]
    \centering
    \includegraphics[width=0.8\textwidth]{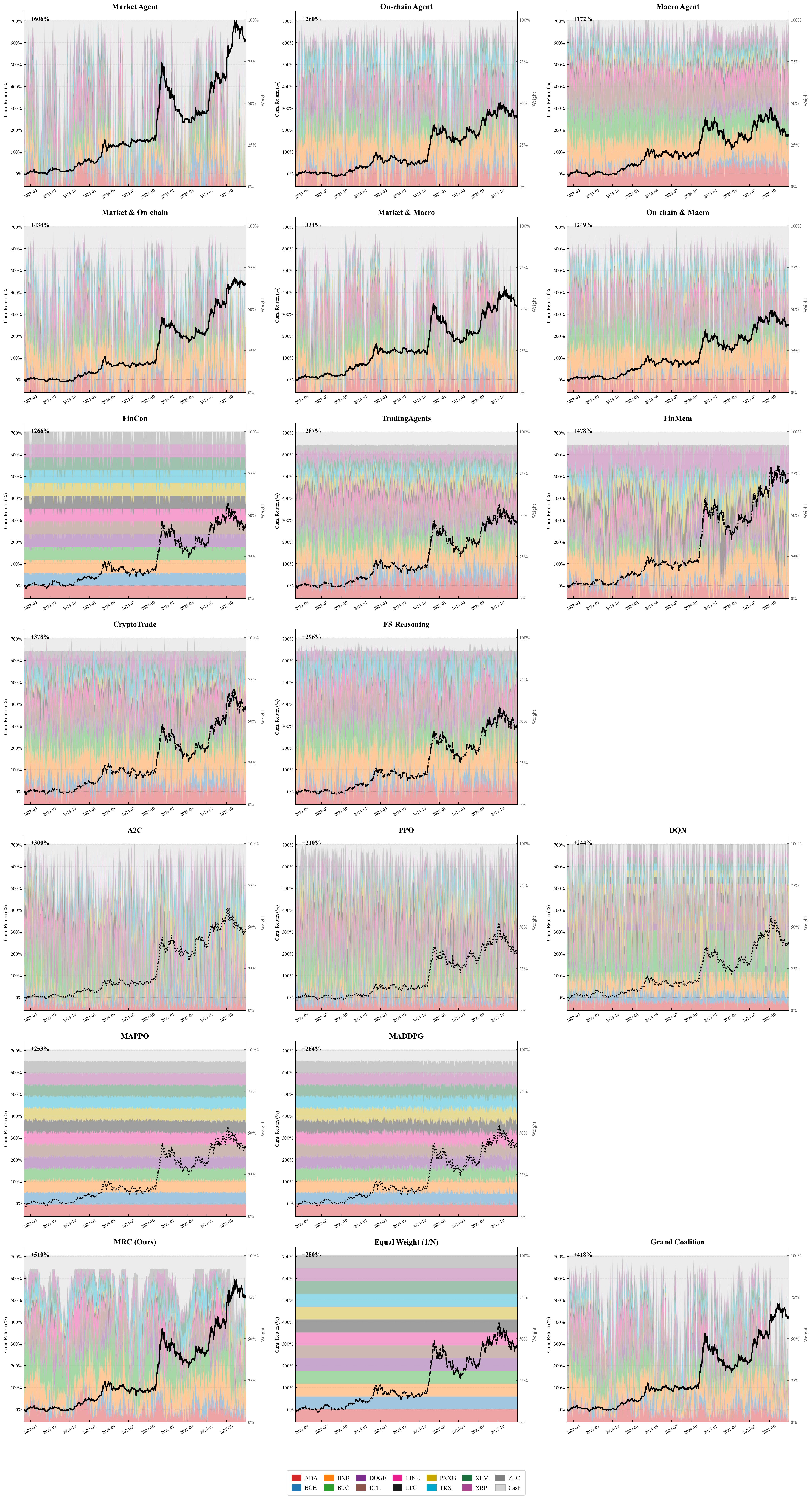}
    \caption{%
        Cumulative returns and portfolio weights for all 17 strategies.
    }
    \label{fig:strategy_returns}
\end{figure}

\clearpage

\begin{figure}[htbp]
    \centering
    \includegraphics[width=0.9\textwidth]{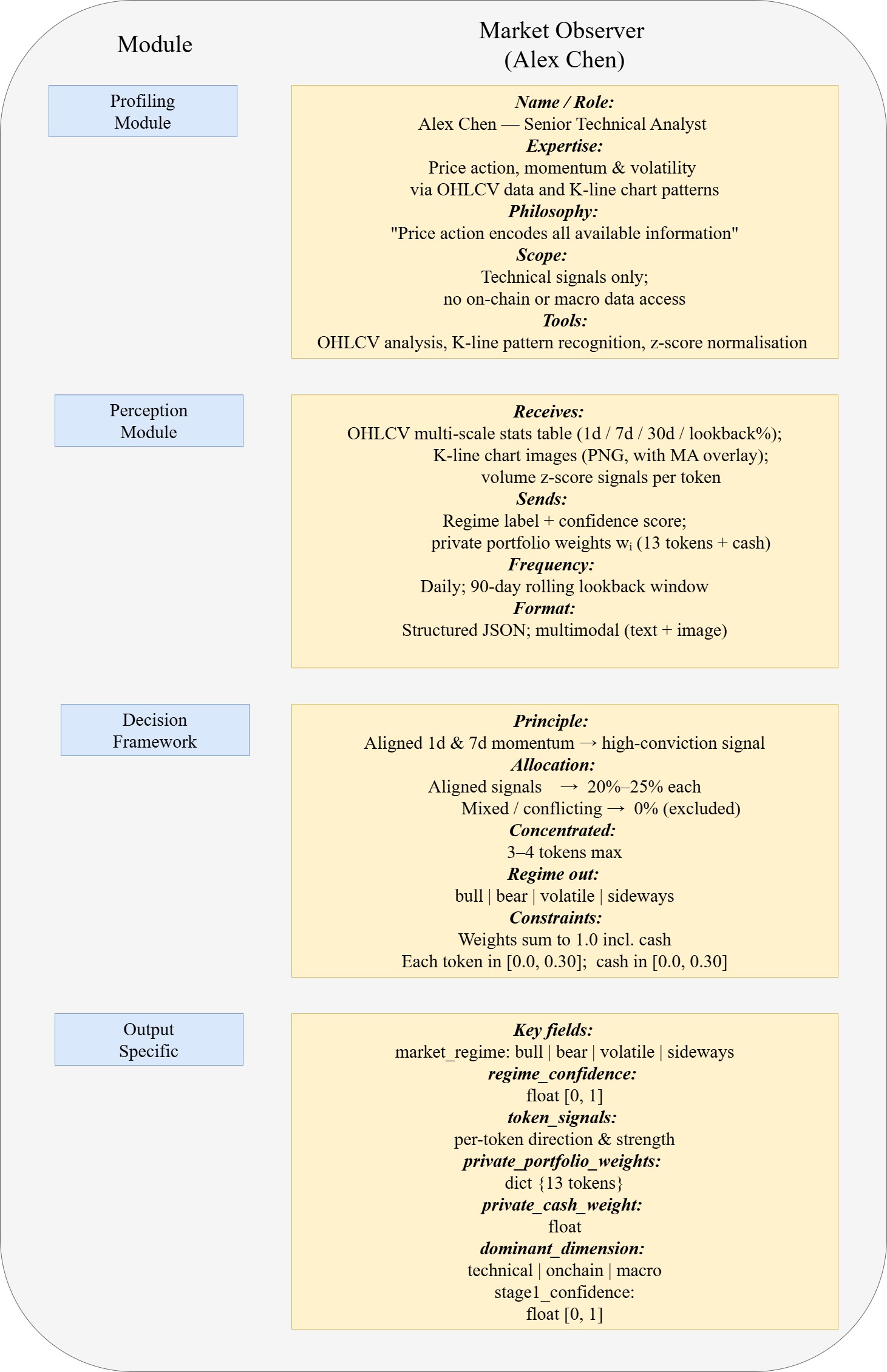}
    \caption{%
        Full specification of $A_1$ Market Observer (Alex Chen).
    }
    \label{fig:market_observer_full}
\end{figure}

\begin{figure}[htbp]
    \centering
    \includegraphics[width=0.9\textwidth]{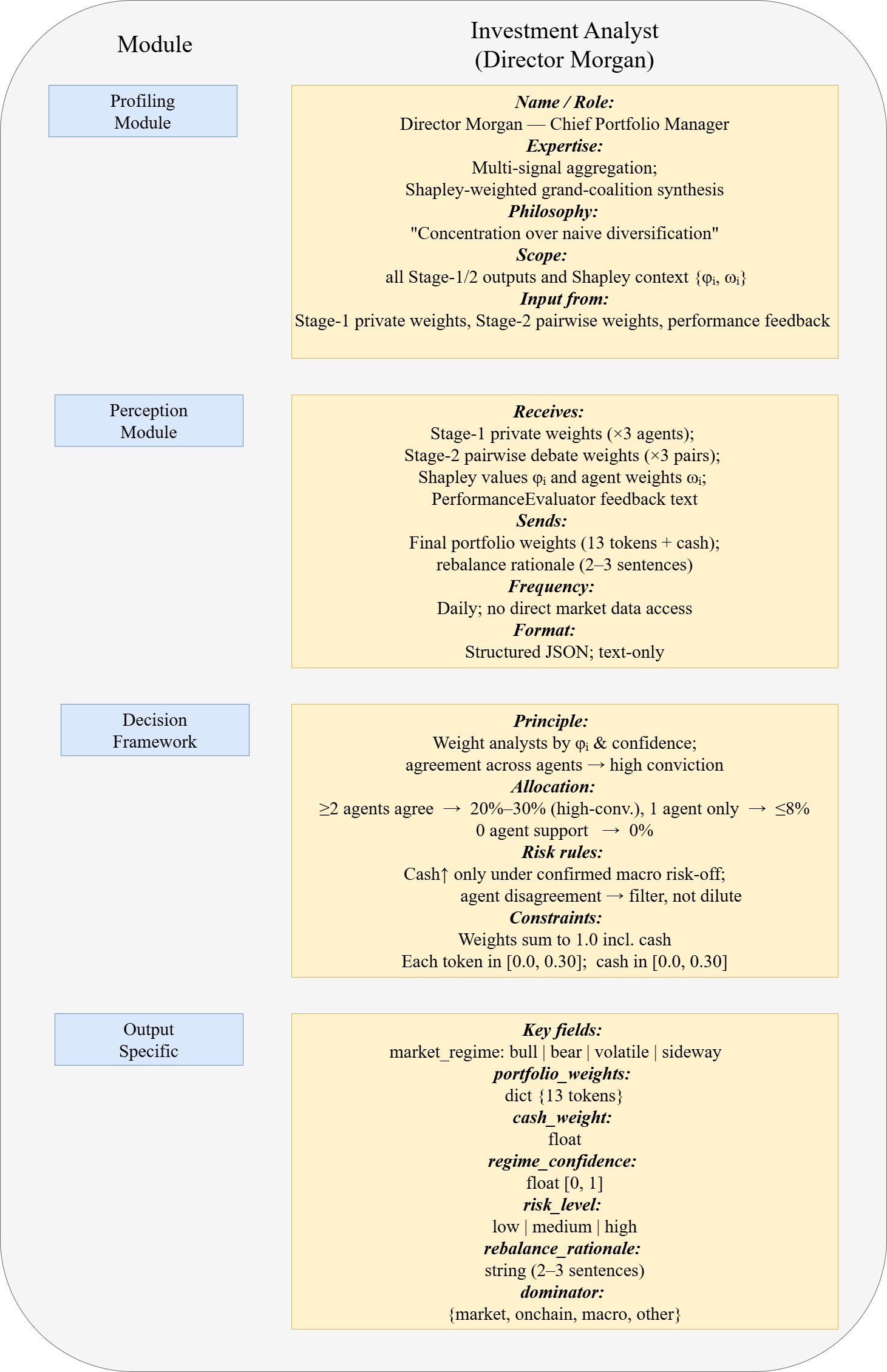}
    \caption{%
        Full specification of $A_4$ Investment Analyst (Director Morgan).
    }
    \label{fig:investment_analyst_full}
\end{figure}


\clearpage

\ifnum\ARXIV=1
\else
    \newpage
    \section*{NeurIPS Paper Checklist}


\begin{enumerate}

\item {\bf Claims}
    \item[] Question: Do the main claims made in the abstract and introduction accurately reflect the paper's contributions and scope?
    \item[] Answer: \answerYes{} 
    \item[] Justification: All claims in the abstract and introduction are supported by experimental results in the paper.

\item {\bf Limitations}
    \item[] Question: Does the paper discuss the limitations of the work performed by the authors?
    \item[] Answer: \answerYes{} 
    \item[] Justification: Limitations are discussed in the Conclusion and the Discussion appendix, including the restriction to cryptocurrency assets, the fixed three-agent design, the zero-cost assumption, and the bear-regime performance gap.

\item {\bf Theory assumptions and proofs}
    \item[] Question: For each theoretical result, does the paper provide the full set of assumptions and a complete (and correct) proof?
    \item[] Answer: \answerYes{} 
    \item[] Justification: All assumptions are stated in the Preliminaries. Formal proofs are provided in the appendix, including a proposition on EWP cold-start attenuation and a full derivation and axiom verification for the three-player Shapley values.

    \item {\bf Experimental result reproducibility}
    \item[] Question: Does the paper fully disclose all the information needed to reproduce the main experimental results of the paper to the extent that it affects the main claims and/or conclusions of the paper (regardless of whether the code and data are provided or not)?
    \item[] Answer: \answerYes{} 
    \item[] Justification: All hyperparameters, dataset details, backbone configuration, and evaluation protocol are fully specified in the Implementation Details appendix. Code and demo data are released at \url{https://github.com/PEIYUNHUA/Market_Regime_Council}.

\item {\bf Open access to data and code}
    \item[] Question: Does the paper provide open access to the data and code, with sufficient instructions to faithfully reproduce the main experimental results, as described in supplemental material?
    \item[] Answer: \answerYes{} 
    \item[] Justification: Code is released at \url{https://github.com/PEIYUNHUA/Market_Regime_Council} under the MIT license, with data samples and a description of each data source and its format. All data sources are described in the dataset inventory table in the appendix.

\item {\bf Experimental setting/details}
    \item[] Question: Does the paper specify all the training and test details (e.g., data splits, hyperparameters, how they were chosen, type of optimizer) necessary to understand the results?
    \item[] Answer: \answerYes{} 
    \item[] Justification: The dataset, baselines, metrics, and hyperparameters are described in the Experiments section and the Implementation Details appendix.

\item {\bf Experiment statistical significance}
    \item[] Question: Does the paper report error bars suitably and correctly defined or other appropriate information about the statistical significance of the experiments?
    \item[] Answer: \answerYes{} 
    \item[] Justification: Results are reported as mean and standard deviation over five independent seeds. Significance of MRC's advantage over the second-best method is assessed with Welch's two-sample $t$-test, with significance levels annotated in the table captions.

\item {\bf Experiments compute resources}
    \item[] Question: For each experiment, does the paper provide sufficient information on the computer resources (type of compute workers, memory, time of execution) needed to reproduce the experiments?
    \item[] Answer: \answerYes{} 
    \item[] Justification: Mean and standard deviation of wall-clock inference time per LLM call and per full rebalancing period are reported for all seven backbone variants in the backbone ablation cost table in the appendix. All inference is performed via the OpenRouter API.

\item {\bf Code of ethics}
    \item[] Question: Does the research conducted in the paper conform, in every respect, with the NeurIPS Code of Ethics \url{https://neurips.cc/public/EthicsGuidelines}?
    \item[] Answer: \answerYes{} 
    \item[] Justification: The paper uses publicly available financial data and does not involve human subjects, private data, or any ethical concerns.

\item {\bf Broader impacts}
    \item[] Question: Does the paper discuss both potential positive societal impacts and negative societal impacts of the work performed?
    \item[] Answer: \answerYes{} 
    \item[] Justification: The five-layer causal trace improves transparency of algorithmic trading decisions. Potential risks include signal manipulation and correlated positioning. The released code is for research reproducibility only.

\item {\bf Safeguards}
    \item[] Question: Does the paper describe safeguards that have been put in place for responsible release of data or models that have a high risk for misuse (e.g., pre-trained language models, image generators, or scraped datasets)?
    \item[] Answer: \answerNA{} 
    \item[] Justification: The paper does not release new models or high-risk datasets. The supplementary code is a backtesting framework using existing public APIs.

\item {\bf Licenses for existing assets}
    \item[] Question: Are the creators or original owners of assets (e.g., code, data, models), used in the paper, properly credited and are the license and terms of use explicitly mentioned and properly respected?
    \item[] Answer: \answerYes{} 
    \item[] Justification: All data sources and models are identified in the dataset inventory table in the appendix. They are accessed through public APIs under standard terms of use.

\item {\bf New assets}
    \item[] Question: Are new assets introduced in the paper well documented and is the documentation provided alongside the assets?
    \item[] Answer: \answerYes{} 
    \item[] Justification: The MRC codebase is released at \url{https://github.com/PEIYUNHUA/Market_Regime_Council} under the MIT license, with a README covering installation, data layout and commands, and with data samples.

\item {\bf Crowdsourcing and research with human subjects}
    \item[] Question: For crowdsourcing experiments and research with human subjects, does the paper include the full text of instructions given to participants and screenshots, if applicable, as well as details about compensation (if any)? 
    \item[] Answer: \answerNA{} 
    \item[] Justification: This paper does not involve crowdsourcing or human subjects.

\item {\bf Institutional review board (IRB) approvals or equivalent for research with human subjects}
    \item[] Question: Does the paper describe potential risks incurred by study participants, whether such risks were disclosed to the subjects, and whether Institutional Review Board (IRB) approvals (or an equivalent approval/review based on the requirements of your country or institution) were obtained?
    \item[] Answer: \answerNA{} 
    \item[] Justification: This paper does not involve human subjects research.

\item {\bf Declaration of LLM usage}
    \item[] Question: Does the paper describe the usage of LLMs if it is an important, original, or non-standard component of the core methods in this research? Note that if the LLM is used only for writing, editing, or formatting purposes and does \emph{not} impact the core methodology, scientific rigor, or originality of the research, declaration is not required.
    \item[] Answer: \answerYes{} 
    \item[] Justification: LLMs are a core component of the method. Three specialist agents and one readout agent use the Qwen3-VL model family at temperature $T_{\mathrm{LLM}}{=}0.7$. Prompt templates are described in the Implementation Details appendix.

\end{enumerate}
\fi

\end{document}